\documentclass[twoside]{article}

\usepackage[utf8]{inputenc} 
\usepackage[T1]{fontenc}    
\usepackage[hyphens]{url}
\usepackage{hyperref}
\usepackage{booktabs}       
\usepackage{amsfonts}       
\usepackage{nicefrac}       
\usepackage{microtype}      

\usepackage{microtype}
\usepackage{graphicx}    
\usepackage{subfigure}
\usepackage{booktabs} 

\newcommand{\cH}{\mathcal{H}}

\newcommand{\cT}{\mathcal{T}}

\usepackage[dvipsnames]{xcolor}

\usepackage{enumerate}
\usepackage{enumitem}
\usepackage{fancyhdr}
\usepackage{mathrsfs}
\usepackage{xcolor}
\usepackage{graphicx}
\usepackage{listings}
\usepackage{hyperref}
\usepackage{caption}
\usepackage{adjustbox}
\usepackage{multirow}
\usepackage{microtype}
\usepackage{tikz}
\usepackage{pgfplots}
\pgfplotsset{compat=newest}
\usepgfplotslibrary{groupplots}
\usepgfplotslibrary{dateplot}
\usepackage{caption}
\usepackage{todonotes}
\tikzset{
  causalvar/.style      = {draw, circle, node distance = 2cm}
}
\newcommand{\indep}{\perp \!\!\! \perp}

\usetikzlibrary{matrix}
\usetikzlibrary{shapes,arrows}
\usepackage{mathtools}
\usepackage{bbm}

\DeclareMathOperator{\E}{\mathbb{E}}

\newcommand{\R}{\mathbb{R}}
\newcommand{\Prb}{\mathbb{P}}

\usepackage{tabularx}

\usepackage[ruled,vlined]{algorithm2e}

\usepackage{microtype}
\usepackage{graphicx}
\usepackage{subfigure}
\usepackage{booktabs} 

\usepackage{amsfonts,amssymb,amscd, amsthm}
\usepackage[dvipsnames]{xcolor}

\usepackage{enumerate}
\usepackage{fancyhdr}
\usepackage{mathrsfs}
\usepackage{xcolor}
\usepackage{graphicx}
\usepackage{listings}
\usepackage{hyperref}
\usepackage{caption}
\usepackage{adjustbox}
\usepackage{multirow}
\usepackage{microtype}
\usepackage{tikz}
\usepackage{pgfplots}
\pgfplotsset{compat=newest}
\usepgfplotslibrary{groupplots}
\usepgfplotslibrary{dateplot}
\usepackage{caption}
\usepackage{todonotes}
\tikzset{
  causalvar/.style      = {draw, circle, node distance = 2cm}
}
\usetikzlibrary{matrix}
\usetikzlibrary{shapes,arrows}
\usepackage{mathtools}

\usepackage{fullpage}

\SetKwComment{Comment}{/* }{ */}

\newtheorem{lemma}{Lemma}[section]
\newtheorem{definition}{Definition}[section]

\newtheorem{thm}{Theorem}

\newtheorem{cor}[thm]{Corollary}
\newtheorem{rem}[thm]{Remark}

\newtheorem{assum}[thm]{Assumption}

\newtheorem{innercustomthm}{Theorem}
\newenvironment{customthm}[1]
  {\renewcommand\theinnercustomthm{#1}\innercustomthm}
  {\endinnercustomthm}

 \newtheorem{innercustomlemma}{Lemma}
\newenvironment{customlemma}[1]
  {\renewcommand\theinnercustomlemma{#1}\innercustomlemma}
  {\endinnercustomlemma}

   \newtheorem{innercustomcor}{Corollary}
\newenvironment{customcor}[1]
  {\renewcommand\theinnercustomcor{#1}\innercustomcor}
  {\endinnercustomcor}

\newcommand{\ice}{\delta}		
\newcommand{\dce}{\theta}		
\newcommand{\mr}{\eta}		
\newcommand{\tce}{\tau}		

\newcommand{\cmo}{\mu_Y} 
\newcommand{\ccmo}{\omega_Y} 

\newcommand{\esf}{\psi} 

\newcommand{\mdensity}{f_{M\mid T,X}} 

\newcommand{\txdensity}{f_{T\mid X}}
\newcommand{\txmdensity}{f_{T\mid X,M}}

\usepackage{amssymb}
\usepackage{pifont}

\makeatletter
\newcommand*\rel@kern[1]{\kern#1\dimexpr\macc@kerna}
\newcommand*\widebar[1]{%
  \begingroup
  \def\mathaccent##1##2{%
    \rel@kern{0.8}%
    \overline{\rel@kern{-0.8}\macc@nucleus\rel@kern{0.2}}%
    \rel@kern{-0.2}%
  }%
  \macc@depth\@ne
  \let\math@bgroup\@empty \let\math@egroup\macc@set@skewchar
  \mathsurround\z@ \frozen@everymath{\mathgroup\macc@group\relax}%
  \macc@set@skewchar\relax
  \let\mathaccentV\macc@nested@a
  \macc@nested@a\relax111{#1}%
  \endgroup
}
\makeatother

\usepackage[accepted]{aistats2025}
\usepackage[round]{natbib}

\begin{document}

%

%

\twocolumn[

\aistatstitle{Double Debiased Machine Learning for Mediation Analysis with Continuous Treatments}

\aistatsauthor{Houssam Zenati \And Judith Ab\'ecassis \And Julie Josse \And Bertrand Thirion}

\aistatsaddress{Inria MIND\footnotemark[1] \And  Inria SODA\footnotemark[2] \And Inria PreMedical\footnotemark[3] \And Inria MIND\footnotemark[1]}
]
\footnotetext[1]{MIND, Inria-Saclay, Palaiseau, Universit\'e Paris Saclay, CEA Saclay, Paris, France}
\footnotetext[2]{SODA, Inria-Saclay, Palaiseau, Universit\'e Paris Saclay, Paris, France}
\footnotetext[3]{PreMedical, Inria-Sophia-Antipolis, Montpellier, France}

\begin{abstract}
    Uncovering causal mediation effects is of significant value to practitioners seeking to isolate the direct treatment effect from the potential mediated effect. We propose a double machine learning (DML) algorithm for mediation analysis that supports continuous treatments. To estimate the target mediated response curve, our method uses a kernel-based doubly robust moment function for which we prove asymptotic Neyman orthogonality. This allows us to obtain asymptotic normality with nonparametric convergence rate while allowing for nonparametric or parametric estimation of the nuisance parameters. We then derive an optimal bandwidth strategy along with a procedure for estimating asymptotic confidence intervals. Finally, to illustrate the benefits of our method, we provide a numerical evaluation of our approach on a simulation along with an application to real-world medical data to analyze the effect of glycemic control on cognitive functions. 
\end{abstract}

\section{INTRODUCTION}

\label{sec:introduction}

In causal inference \citep{imbens_rubin_2015}, traditional treatment evaluations often focus on assessing the total causal effect of a treatment on an outcome variable, such as the average treatment effect (ATE). However, in many evaluation settings, understanding the causal mechanisms driving these total effects is equally important. 
For example, in studying the relationship between socio-demographic factors and cognitive function, one might hypothesize that the effect of these factors is mediated by the brain structure captured in brain imaging. 
This mediation insight comes from studies such as \citep{dadi2021} and \citep{cox2019structural}, that have demonstrated that combining IDPs with socio-demographic variables often provides little to no improvement in predictive performance over socio-demographic data alone. 
However, the causal impact of socio-demographic variables on cognitive outcomes is expected to occur through their influence on brain structure captured in brain imaging.
Understanding these mediation pathways can help researchers and health policymakers design more targeted interventions that focus e.g. on altering socio-demographic determinants before they alter irreversibly the brain tissues.

Mediation analysis aims to assess the specific contributions of the indirect and direct effects of a given treatment. The indirect effect quantifies the proportion of the total effect that is due to a mediator, while the direct effect measures the effect of the treatment without additional mediators \citep{greenland1992, pearl2001}. Several studies have used flexible (often nonparametric) models \citep{petersen2006, flores2009, imai2010, hong2010ratio, imai2013identification}: %
this includes regression-based \citep
{Robins1986, zhengVanderLaan2012} using the Pearl formula \citep{pearl2001} or models that use inverse probability weighting (IPW) \citep{thompson1952, huber2014}, hence requiring the computation of propensity scores, either conditional on both the mediator and covariates or on the covariates alone. 
However, these methods become inconsistent if the models for the conditional mean outcome or the treatment and mediator densities are misspecified. To improve on this practice, more sophisticated causal mediation analysis schemes have been proposed, based on efficient score functions \citep{tchetgen2012}, with double machine learning (DML) as outlined in \citet{Chernozhukov2018}, which relies on Neyman orthogonality and sample splitting. Typically, when combining efficient score-based estimation with sample splitting, $n^{-1/2}$ convergence for treatment effect estimation can be achieved, even with slower plug-in estimate convergence rates of $n^{-1/4}$.


While most nonparametric mediation studies have focused on binary treatments, many real-world problems involve continuous treatments, such as the dose of a medical treatment \citep{hirano2005, imai2004causal, bia2012assessing, kluve2012evaluating, galvao2015, lee2018partial, cox2019structural}. This paper addresses nonparametric estimation of natural direct and indirect effects \citep{pearl2001} for continuous treatments. We propose a multiply robust estimator based on weighting by the inverse of conditional treatment densities \citep{hirano2005, imai2004causal} and the estimation of conditional mean outcomes \citep{pearl2001, singh2023sequential}. Our method is asymptotically normal and converges at the rate of one-dimensional nonparametric regression under specific regularity conditions. 

Our first contribution is to show the asymptotic \citet{neyman1959} orthogonality of a kernelized moment function, similar to the efficient score functions of
\citet{tchetgen2012, farbmacher2022}. This property
makes the estimation of direct and indirect effects rather insensitive to (local) estimation errors in the plug-in estimates. Second, we propose a DML estimator with a Bayes transformation that avoids the conditional mediator density estimation and directly estimates a nested conditional mean outcome \citep{farbmacher2022}. This
appears particularly useful when the mediator is a vector of variables and/or continuous. Third, we analyze this estimator and prove that it is asymptotically normal with a nonparametric convergence rate under mild regularity conditions and mild convergence requirements for nuisance parameters. This allows us to derive an asymptotic mean squared error optimal bandwidth and an asymptotic confidence interval. Moreover, our analysis proves the multiple robustness of our estimator with regards to i) the consistency of the nuisance parameters and ii) achieving a faster non-parametric convergence rate with slow nuisance convergence rates. Eventually, we provide numerical experiments illustrating the practical performance of DML compared to other mediation analysis estimators and an application of continuous treatment mediation for cognitive function on real-world data.


\subsection{Related Work}

In the causal mediation literature, most studies assume sequential conditional independence, where (i) potential outcomes and treatment assignment are independent given covariates, and (ii) potential outcomes and mediator are independent given treatment and covariates \citep{imai2010, pearl2001, flores2009}.  Under such assumptions, a common approach uses linear equations \citep{judd1981process, baron1986, vanderweele_explanation_2015} modeling of the mediator as a function of treatment and covariates, and the outcome as a function of the mediator, treatment, and covariates. However, such methods rely on the correct specification of linear relationships, which may not be appropriate for binary mediators or outcomes. In contrast, G-computation \citep{Robins1986, zhengVanderLaan2012} using Pearl formula \citep{pearl2001} allows for non-linear models. Moreover, inverse probability weighting (IPW) \citep{thompson1952, huber2014} methods (in the binary treatment case) leverage propensity scores (either conditional on the mediator and covariates or on the covariates alone) to estimate mediation mechanisms. However, if either the conditional means or densities are misspecified, those methods become inconsistent. 

Subsequently, multiply robust estimators—akin to doubly robust estimation in standard treatment effect models—have been proposed \citep{tchetgen2012, zhengVanderLaan2012} and shown to be robust to model misspecification \citep{huber2015}. Unlike \citet{tchetgen2012} which requires the mediator density estimation, \citet{farbmacher2022} analyzed a doubly robust estimation method that relies on estimates of the conditional mean outcome, cross conditional mean outcome, and conditional treatment probabilities. Analogously to doubly robust estimation of average treatment effects \citep{robins1994,robins1995}, all multiply robust resulting estimators are semiparametrically efficient if all models underlying the plug-in estimates are correctly specified and remain consistent even if one model is misspecified. Our work builds on the results for semiparametric models in \citet{ichimura2022,chernozhukov2022} and follows the line of \citet{farbmacher2022}: a growing literature employs the double machine learning (DML) approach for non-regular nonparametric infinite-dimensional objects for estimating causal quantities \citep{Chernozhukov2022b,semenova2020,fan2021,zimmert2019,bonvini2022}.

However, the majority of research in mediation analysis has focused on binary 
treatments. Continuous treatments specifically pose a challenge in IPW-based methods, where propensity scores need to be adapted \citep{imbens2000, hirano2005, flores2007, galvao2015, zenati}, even in the mediation analysis task \citet{hsu2020} which uses a kernel smoothing strategy. \citet{sani2024} proposes a multiply robust estimator for the mediation task with continuous treatment, but requires i) the mediator conditional density estimation and ii) consistency of the nuisance parameters. Instead, our work does not require the cumbersome mediator conditional density estimation, as in \citet{farbmacher2022} and we prove the asymptotic Neyman orthogonality \citep{neyman1959} of our kernel moment function. Moreover, our estimator is multiply robust in the sense that our inference theory is valid even when one nuisance function is inconsistent, as in \citet{kennedy2017} and \citet{takatsu2024}. This is a stronger result than the usual multiply robustness on the consistency of the estimator used in \citet{farbmacher2022} and \citet{sani2024} which analysis does not consider inconsistent nuisance parameters. Eventually, unlike \citet{sani2024} our method discussed and posed the problem of bandwidth selection and uncertainty quantification through confidence intervals.


\section{BACKGROUND}

\label{sec:background}

In this section we introduce the necessary definitions and notations for the general mediation analysis task. 

\subsection{Natural direct and indirect effects}
\label{sec:definitions_ce}

We label the random variable for the treatment as $T$, the outcome as $Y$, the mediator(s) as $M$, and the covariate(s) as $X$ for each individual. We will also write $Z=(Y, T, X, M)$. The interrelations among these variables are depicted in Figure \ref{fig:mediation}. 
 
\begin{figure}[h]
    \centering
    \includegraphics[width=0.4\textwidth]{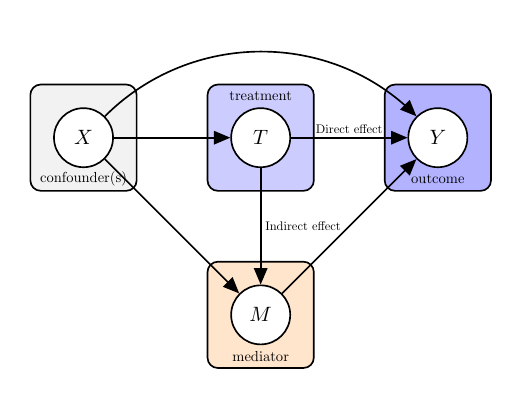}
    \caption{Causal graph for mediation analysis. }
    \label{fig:mediation}
\end{figure}

Given treatment values $t, t^\prime \in \mathcal{T}$ we denote $M(t)$ and $Y(t, M(t^\prime))$ the potential mediator state and potential outcome, expanding upon the potential outcomes framework \citep{imbens_rubin_2015}. We make the assumption that the observed outcome and mediator correspond to the potential outcome and mediator under the actually assigned treatment.

We then define the total average treatment effect for a general treatment space $\mathcal{T}$. For the next definitions, let $t, t^\prime \in \mathcal{T}$.

\begin{definition}[Total average treatment effect]
\begin{equation}
    \tce(t,t') = \E [Y\left(t', M(t')\right)] - \E [Y\left(t,M(t)\right)].
\label{eq:tce}
\end{equation}
\end{definition}

Next, we define the natural direct effect, denoted by $\dce(t, t')$,
as the difference in expected potential outcomes when switching the treatment while keeping the potential mediator fixed, which blocks the causal mechanism via $M$:

\begin{definition}[Direct effect]
\begin{equation}
    \dce(t, t') = \E [Y\left(t',M(t)\right)] - \E [Y\left(t,M(t)\right)].
\label{eq:dce}
\end{equation}
\end{definition}

Eventually, the natural indirect effect, $\ice(t,t')$, equals the difference in expected potential outcomes when switching the potential mediator values while keeping the treatment fixed to block the direct effect.

\begin{definition}[Indirect effect]
\begin{equation}
    \ice(t, t') = \E [Y\left(t',M(t')\right)] - \E [Y\left(t',M(t)\right)].
\label{eq:ice}
\end{equation}
\end{definition}

Let $\mathcal{Z}= \mathcal{Y} \times \mathcal{T} \times \mathcal{X} \times \mathcal{M}$ represent the support of $Z = (Y, T, X, M)$, with an associated cumulative distribution function (CDF) $F_{Z}(Z)$.

\subsection{Mediated response}

When it comes to isolating direct and effects, it is more common to introduce the notion of mediated response for $t, t^\prime \in \mathcal{T}$.
\begin{definition}[Mediated response]
\begin{equation}
    \mr_{t,t^\prime} = \E [Y\left(t,M(t')\right)]
\label{eq:mediated_response}
\end{equation}
\end{definition}


It is then possible to introduce simple and convenient expressions that establish relations between previously defined quantities:
\begin{align}
    \tce(t,t') &= \dce(t,t') + \ice(t,t') \\
    \dce(t,t') &= \mr(t',t) - \mr(t,t) \\
    \ice(t, t') &= \mr(t',t') - \mr(t',t)
\end{align}

Therefore, one  can rely the counterfactual potential outcome $\E [Y\left(t,M(t')\right)]$, that is the mediated response, to derive expressions of the mediated causal quantities of interest. To this end, we present in the next subsection classical assumptions necessary for its identification.

\subsection{Identification assumptions}
\label{sec:assumptions}

Since the mediated response curve is defined in terms of potential outcomes that are not directly observed, we must consider assumptions under which it can be expressed in terms of observed data. A complete treatment of identification in the presence of continuous random variables has been given by \citet{gill2001}; we refer the reader there for details. The assumptions most commonly employed \citep{imai2010, pearl2001, flores2009} for identification are as follows.

\begin{assum}[Conditional independance of the treatment]
For all $t',t \in \cT, m \in \mathcal{M}$
\begin{equation*}
\lbrace Y(t',m), M(t) \rbrace \indep T|X. 
\end{equation*} 
\label{assum:indep_treatment}
\end{assum}
\vspace{-7mm}
Assumption \ref{assum:indep_treatment} requires the absence of unobserved confounders influencing the treatment or the outcome, or affecting their relationship.

\begin{assum}[Conditional independance of the mediator]
For all $t',t \in \cT, m \in \mathcal{M}, x \in \mathcal{X}$
$$Y(t',m) \indep M|T=t, X=x.$$
\label{assum:indep_mediator}
\end{assum}
\vspace{-7mm}
Assumption \ref{assum:indep_mediator} requires the absence of unobserved confounders influencing both the mediator and the outcome.

\begin{assum}[Positivity assumption] 
For some positive constant $c$, the essential infimums $\inf _{t \in \mathcal{T}} \operatorname{ess}_{\inf _{x \in \mathcal{X}}} \txdensity(t \mid x) \geq c$ and $\inf _{t \in \mathcal{T}} \operatorname{ess}_{\inf _{x \in \mathcal{X}, m \in \mathcal{M}}} \txmdensity(t \mid x, m) \geq c$.
\label{assum:overlap}
\end{assum}

Note that the mild condition on the essential infimum is slightly stronger than the common positivity assumption; as in \citet{Colangelo2020} this will prove useful for the continuous task with kernel localization in Section \ref{sec:continuous_mediation}. Assumption \ref{assum:overlap} ensures that the treatment is not deterministic in the covariate $X$ and in the pair of covariate and mediator $X, M$. Given the identifications conditions, we introduce the mediation formula at the core of the mediation estimators we propose.

\subsection{Mediation formula}

 Let us note $\txdensity(t\mid X)$ (respectively $\txmdensity(t\mid X, M)$) the conditional density of $T$ given $X$ (respectively given $X$ and $M$) and $\mdensity(M \mid T, X)$ the conditional density of $M$ given $T$ and $X$ (if $M$ is discrete, this is a conditional probability). We now define fundamental quantities, starting with the conditional mean being defined as the expectation of the outcome given determined values of the treatment $t \in \mathcal{T}$, mediator $m\in \mathcal{M}$ and context $x\in \mathcal{X}$. 

\begin{definition}[Conditional mean outcome]
\begin{equation}
    \cmo(t, m, x)= \E [ Y | T=t, M=m, X=x]
\label{eq:conditional_mean_outcome}
\end{equation}
\end{definition}

Based on the previous expression, we can define the cross conditional mean outcome as follows:

\begin{definition}[Cross conditional mean outcome]

\begin{align}
\ccmo(t,t', x) &= \int_{\mathcal{M}} \cmo(t,m,x) \mdensity(m|t',x)dm \nonumber \\
&= \E \left[ \cmo(t, m, X) | T=t', X =x \right]  \label{eq:cross_conditional_mean_outcome}   
\end{align}

\end{definition}

Using the previous definitions, Pearl's mediation formula \citep{pearl2001} is as follows:

\begin{lemma}[Pearl's mediation formula]
\begin{equation}
\begin{aligned}
  \mr_{t,t^\prime} &= \int_{\mathcal{M}, \mathcal{X}} \cmo(t,m,x) \mdensity(m|t',x) f_X(x)dmdx \\
 &= \int_{\mathcal{X}} \ccmo(t, t^\prime,x) f_X(x) dx   
\end{aligned}  
\label{eq:pearl_formula}
\end{equation}
\end{lemma}

where $f_X$ is the density over the covariate space $\mathcal{X}$. This formula is the basis for several estimates that are defined in the next section. 



\section{CONTINUOUS TREATMENT MEDIATION}

\label{sec:continuous_mediation}

In this section, we introduce the tools used to analyze mediation with continuous treatments. 
For the sake of clarity, we consider a unidimensional treatment space~$\mathcal{T}$, the generalization to multidimensional~$\mathcal{T}$ is given in Appendix \ref{appendix:moment} and \ref{appendix:asymptotic_analysis}.
Let us then consider a kernel $k$ that satisfies Assumption \ref{assum:kernel}, which is standard in nonparametric kernel estimation and holds for commonly used kernel functions, such as Epanechnikov and Gaussian.

\begin{assum}
The second-order symmetric kernel function $k()$ is bounded differentiable, i.e. $\int_{-\infty}^{\infty} k(u) d u=1, \int_{-\infty}^{\infty} u k(u) d u=0$, and $0<\kappa<\infty$. For some finite positive constants $C, \bar{U}$, and for some $\nu>1,|d k(u) / d u| \leq C|u|^{-\nu}$ for $|u|>\bar{U}$.
\label{assum:kernel}
\end{assum} 

Let us consider a bandwidth $h$. We now define the kernel smoothing operator as 

\begin{equation}
K_h\left(T-t\right) = k\left(\left(T-t\right) / h\right) / h  
\end{equation}



With Assumptions \ref{assum:indep_treatment}, \ref{assum:indep_mediator}, \ref{assum:overlap} and the same reasoning for the binary treatment, it is possible to show the identification \citep{hsu2020} (see Appendix \ref{appendix:moment}):
\begin{align}
\mr_{t,t^\prime} &=\lim _{h \rightarrow 0} \mathbb{E}\left[\frac{K_h(T-t) Y}{\txdensity(t \mid X)} \frac{\mdensity(M \mid t', X)}{\mdensity(M \mid t, X)}\right] \nonumber \\ 
&=\lim _{h \rightarrow 0} \mathbb{E}\left[\frac{K_h(T-t) Y}{\txdensity(t^\prime \mid X)} \frac{\txmdensity(t' \mid X, M)}{\txmdensity(t \mid X, M)}\right]\label{eq:ipw_generalized} 
\end{align}
for $t, t^\prime \in \mathcal{T}$ and where the last equality follows a Bayes transformation on the mediator density as in \citet{farbmacher2022}.

The expression in Eq.~\eqref{eq:pearl_formula} motivates the class of G-computation  (or regression-based) estimators, while Eq.~\eqref{eq:ipw_generalized} motivates the class of inverse probability weighting estimators; see Appendix \ref{appendix:moment} for further discussion. The DML \citep{Chernozhukov2018} estimator leverages a moment function that combines both classes of methods, as in influence functions for semiparametric estimators \citep{newey1994, ichimura2022}. Similar to \citet{tchetgen2012} and \citet{ farbmacher2022} for the mediation analysis task, we introduce a kernel smoothed moment function with provable asymptotic convergence.

\begin{lemma}
Let $t, t^\prime \in \mathcal{T}$ and let the moment function
\begin{equation}
\begin{aligned}
\esf_{t,t^\prime}^h= & \frac{ K_h(T-t) \txmdensity(t' \mid X, M)}{\txdensity(t' \mid X) \txmdensity(t \mid X, M)} [Y-\cmo(t, M, X)] \\
& +\frac{K_h(T-t')}{\txdensity(t' \mid X)} [\cmo(t, M, X) -\ccmo(t,t', X)] \\
& +\ccmo(t,t', X) - \mr_{t,t^\prime}.
\end{aligned}    
\label{eq:moment_function}
\end{equation}
then, the moment function is asymptotically Neyman-Orthogonal.
\label{lemma:neyman_orthogonal}
\end{lemma}

A proof of Lemma \ref{lemma:neyman_orthogonal} is given in Appendix \ref{appendix:moment}. Note that in the moment function we propose, we do not make use of the mediator density function and instead propose to estimate the nuisance parameters $\widehat \txdensity, \widehat \cmo, \widehat \txdensity$ and $\widehat \ccmo$. Indeed, as explained in \citet{farbmacher2022}, the original multiply robust estimator for mediation \citep{tchetgen2012} proposed to only estimate the nuisances $\widehat \txdensity, \widehat \cmo$, and $\widehat \mdensity$ to build their efficient influence function but would estimate the cross conditional expectation $\ccmo$ as $\int_{M} \widehat \cmo(t, M, X) \mdensity(M \mid t', X) dM$ as done in \citet{sani2024}. However, we argue that such a formulation hurts the estimation of $\ccmo$ especially when the mediator is high dimensional. 


Subsequently, aside from the the Neyman Orthogonality condition on the estimator which motivated the use of an efficient score function, the DML approach \citep{Chernozhukov2018} also uses sample splitting to estimate separately the nuisance parameters and the causal quantity of interest. Let now $\{Y_i, T_i, M_i, X_i\}^n_{i=1}$ be i.i.d realizations of the random variables $Y, T, M, X$. An $L$-fold cross-fitting splits the sample into subsamples $I_{\ell}$ for $\ell \in \{1, \dots, L \}$. 
For a given split $I_\ell$ and $i \in I_\ell$, the nuisance function estimators for $\widehat{\txdensity}_{i \ell}(t^\prime) = \txdensity(t^\prime \mid X_i), \widehat{\txmdensity}_{i \ell}(t) = \txmdensity(t \mid X_i, M_i)$, $\widehat{\cmo}_{i \ell} = \cmo(t, M_i, X_i)$ and $\widehat{\ccmo}_{i \ell}=\ccmo(t, t^\prime, X_i)$ 
use observations in the other $(L-1)$ subsamples that do not contain observation $i$. The DML estimator then averages over the whole subsamples as stated in Algorithm \ref{alg:DML}.

\begin{algorithm}[t]
\caption{Continous Treatment Double Machine Learning}\label{alg:DML}
\KwData{Observational $\{Y_i, T_i, M_i, X_i\}^n_{i=1}$}
- Split the data in $L$ subsamples. For each subsample $\ell$, let $n_\ell$ denote its size, $Z^\ell$ the set of observations in the sample and $Z^{\ell, C}$ the complement set of all observations not in $Z^\ell$.

- For each $\ell$, use $Z^{\ell, C}$ to estimate $\txdensity$, $\txmdensity$, $\cmo$ and $\ccmo$.
- Predict the nuisance parameters for each observation $i$ in $\mathcal{D}_\ell$, i.e. $\widehat{\txdensity}_{i \ell}, \widehat{\txmdensity}_{i \ell}$, $\widehat{\cmo}_{i \ell}$ and $\widehat{\ccmo}_{i \ell}$.

- For each $\ell$, obtain an estimate of the kernel score function for each observation $i$ in $Z^{\ell}$, denoted by $\hat{\esf}_{t,t', i \ell}^h:$

\begin{equation}
\begin{aligned}
\hat{\esf}_{t,t', i \ell}^h= 
& \frac{K_h(T_i-t)\widehat \txmdensity_{i \ell}(t^\prime)}{\widehat \txmdensity_{i \ell }(t)\widehat \txdensity_{i \ell}} \cdot\left[Y-\widehat{\cmo}_{i \ell}\right] \\
& +\frac{K_h(T_i-t')}{\widehat \txdensity_{i \ell}} \cdot\left[\widehat{\cmo}_{i \ell}-\widehat{\ccmo}_{i \ell}\right] \\
& +\widehat{\ccmo}_{i \ell}
\end{aligned}
\label{eq:esf_hat}
\end{equation}

- Average the estimated scores $\hat{\esf}_{t,t', i \ell}^h$ over all observations across all $L$ subsamples to obtain an estimate of $\mr_{t,t'}$ in the total sample, denoted by 
\begin{equation}
\hat{\mr}_{t,t'}= \frac{1}{n} \sum_{\ell=1}^L \sum_{i=1}^{n_\ell} \hat{\esf}_{t,t', i \ell}^h
\label{eq:mr_hat}
\end{equation}
\KwResult{Estimated $\hat{\eta}_{t,t'}$}
\end{algorithm}

Note that a key distinction of our method compared to \citet{sani2024} is that they integrate the conditional mean outcome against the mediator density to estimate the cross conditional mean outcome. While this works only for discrete mediators, we avoid  integration over mediator densities, as in \citet{farbmacher2022} and directly estimate the cross-conditional mean outcome with an implicit integration as in \citet{singh2023sequential} which recently addressed the challenge related to continuous cross-counterfactual treatments $(t, t')$. This makes our approach suitable for continuous or high-dimensional mediators.

\section{ASYMPTOTIC ANALYSIS}

\label{sec:asymptotic}

In this section we analyze the asymptotic behavior of the DML estimator we proposed.

We define $\left\|\cdot \right\|_{{t M X}}^2$ (respectively $\left\|\cdot \right\|_{{X}}^2$ and $\left\|\cdot \right\|_{{t X}}^2$) the partial $L_2(t M X)$ norm (respectively $L_2(X)$ and $L_2(t X)$) for any $t \in \mathcal{T}$ as
\begin{equation*}
\left\|\cdot \right\|_{{t M X}}^2    = \int_{\mathcal{M}, \mathcal{X}}\left(\cdot\right)^2 f_{T M X}(t, m, x) d m d x    
\end{equation*}

where the integral is taken with regards in $\mathcal{M} \times \mathcal{X}$ (respectively $\mathcal{X}$) with the joint density $f_{T M X}$ (respectively $f_{X}$ and $f_{T X}$). We will also write $\delta_{\cmo}^\ell=\widehat{\cmo}_{\ell}-\widebar{\cmo}$, $\delta_{\ccmo}^\ell=\widehat{\ccmo}_{\ell}-\widebar{\ccmo}$, $\delta_{T\mid X}^\ell=\widehat{\txdensity}_{\ell}-\widebar{\txdensity}$ and $\delta_{T\mid X, M}^\ell=\widehat{\txmdensity}_{\ell}-\widebar{\txmdensity}$. We can now state conditions on the convergence of the nuisance parameters.



\begin{assum}[Nuisance convergence]
There exist functions $\widebar{\cmo}$, $\widebar{\ccmo}$, $\widebar{\txmdensity}$ and $\widebar{\txdensity}$ that are three-times differentiable with $\widebar{\txdensity}(t \mid x), \widebar{\txmdensity}(t \mid x, m) \geq c$ for some positive constant $c$, and satisfy the following: For each $\ell=1, \ldots, L$, 
\begin{enumerate}
    \item $\left\| \delta_{\cmo}^\ell \right\|_{{t MX}}=o_p(1)$, $\left\| \delta_{\ccmo}^\ell\right\|_{{X}}=o_p(1)$, $\left\|\delta_{T\mid X}^\ell\right\|_{{t X}}=o_p(1)$ and $\left\|\delta_{T\mid X, M}^\ell\right\|_{{t MX}}=o_p(1)$


    \item Either $\widebar{\cmo}=\cmo$ and $\widebar{\ccmo}=\ccmo$, or $\widebar{\cmo}=\cmo$ and $\widebar{\txdensity}=\txdensity$, or $\widebar{\ccmo}=\ccmo$ and $\widebar{\txmdensity}=\txmdensity$, or $\widebar{\txdensity}=\txdensity$ and $\widebar{\txmdensity}=\txmdensity$.
\end{enumerate}
\label{assum:convergence}
\end{assum}


    
    

\begin{assum}[Nuisance rates]
There exist functions $\widebar{\cmo}$, $\widebar{\ccmo}$, $\widebar{\txmdensity}$ and $\widebar{\txdensity}$ that satisfy the following: For each $\ell=1, \ldots, L$, 
\begin{enumerate}

    \item $\sqrt{n h}\left\|\delta_{T\mid X}^\ell\right\|_{{t X}}\left\|\delta_{\cmo}^\ell\right\|_{{t MX}} =o_p(1)$
    
    \item $\sqrt{n h} \left\|\delta_{T\mid X, M}^\ell\right\|_{{t M X}}\left\|\delta_{\cmo}^\ell\right\|_{{t MX}} =o_p(1)$
    
    \item $\sqrt{n h}\left\|\delta_{T\mid X}^\ell\right\|_{{t X}}\left\|\delta_{\ccmo}^\ell\right\|_{{X}} =o_p(1)$
\end{enumerate}
\label{assum:nuisance}
\end{assum}

The nuisance function estimators $\widehat{\cmo}_{\ell}$, $\widehat{\ccmo}_{\ell}$ $\widehat{\txdensity}_{\ell}$ and $\widehat{\txmdensity}_{\ell}$ converge to some fixed functions $\widebar{\cmo}$, $\widebar{\ccmo}$, $\widebar{\txdensity}$ and $\widebar{\txmdensity}$ respectively in the sense of Assumption \ref{assum:convergence}.1. Assumption \ref{assum:convergence}.2 allows some of the nuisance functions to be misspecified. Assumption \ref{assum:nuisance} and \ref{assum:convergence}.2 imply that if one nuisance function is misspecified, then the other needs to be estimated consistently at a convergence rate faster than $\sqrt{n h}$. This is the cost of the multiply robust inference. For example, if $\widebar{\cmo} \neq \cmo$, then $\left\|\widehat{\cmo}_{\ell}-\cmo\right\|_{{t MX}}=O_p(1)$. So Assumption \ref{assum:nuisance} requires $\sqrt{n h}\left\|\widehat{\txdensity}_{\ell}-\widebar \txdensity\right\|_{{t X}}~=~o_p(1)$. 

Moreover, we will require additional regularity assumptions on the density functions $f_Z$ and the nuisance parameters. 

\begin{assum}[Regularity]
We list the following:
\begin{enumerate}
    \item $f_Z(y, t, m, x)$, $\cmo$ and $\ccmo$ are three-times differentiable with respect to $t$ with all three derivatives being bounded uniformly over $\left(y, t, x, m \right) \in \mathcal{Z}$
    \item $\operatorname{var}(Y \mid T=t, X=x, M=m)$ and its derivatives with respect to $t$ are bounded uniformly over $\left(t^{\prime}, x \right) \in \mathcal{T} \times \mathcal{X}$.
\end{enumerate}
\label{assum:regularity}
\end{assum}

We can now state the main result of the asymptotic normality of the DML estimator in mediation.

\begin{thm}[Asymptotic normality] Let Assumptions~\ref{assum:indep_treatment}-\ref{assum:regularity} hold. Let $h \rightarrow 0, n h~\rightarrow \infty$ , and $n h^{5}~\rightarrow~C~\in[0, \infty)$. Then for any $t \in \mathcal{T}$,

\begin{align*}
&\sqrt{n h}\left(\hat{\mr}_{t,t'}-\mr_{t,t^\prime}\right)  \\
=\sqrt{\frac{h}{n}} \sum_{i=1}^n & \left[
\frac{  K_h(T_i-t) \cdot \widebar \txmdensity_{t^\prime, i}}{\widebar\txdensity_i \cdot \widebar \txmdensity_{t, i}} \cdot[Y-\widebar\cmo_i] \right.\\
 & + \frac{K_h(T_i-t')}{\widebar\txdensity_i} \cdot[\widebar\cmo_i -\widebar\ccmo_i] \\
& +  \left.  \widebar \ccmo_i
\right]+o_p(1) .
\end{align*}


where we abbreviated $\widebar\txdensity_i = \widebar\txdensity(t' \mid X_i)$, $\widebar\cmo_i = \widebar\cmo(t, M_i, X_i)$, $\widebar\ccmo_i = \widebar\ccmo(t,t', X_i)$, $\widebar \txmdensity_{t^\prime, i}= \widebar \txmdensity(t' \mid X_i, M_i)$ and $\widebar \txmdensity_{t, i}= \widebar \txmdensity(t \mid X_i, M_i)$.

Let $\mathbb{E}\left[|Y-\widebar{\cmo}(T, M, X)|^3 \mid T=t, M, X\right]$, $\mathbb{E}\left[(\widebar{\cmo}(t, M, X) - \widebar{\ccmo}(t, t', X))^3 \mid T=t', X\right]$ and their derivatives with respect to $t$ be bounded uniformly over $\left(t, m, x\right) \in \mathcal{T} \times \mathcal{M} \times \mathcal{X}$. Let $\int_{-\infty}^{\infty} k(u)^3 d u<\infty$. Then
\begin{equation*}
\sqrt{n h}\left(\hat{\mr}_{t,t'}-\mr_{t,t^\prime}-h^2 \mathrm{~B}_{t,t'}\right) \xrightarrow{d} \mathcal{N}\left(0, \mathrm{~V}_{t,t'}\right)    
\end{equation*}
where 
\begin{align*}
    \mathrm{~V}_{t,t'} &\equiv \mathbb{E}\left[\bar{V}_Y \frac{\txmdensity(t \mid X, M)}{\widebar{\txdensity}(t'\mid X)^2} \frac{\widebar{\txmdensity}(t'\mid X, M)^2}{\widebar{\txmdensity}(t\mid X, M)^2} \right. \\
    &\quad + \frac{\txdensity(t'\mid X)}{\widebar{\txdensity}(t'\mid X)^2} \bar{V}_{\cmo}  \Bigg] R_k \\
    \mathrm{B}_{t, t^\prime} &\equiv \mathbb{E}\left[\frac{\widebar{\txmdensity}(t'\mid X, M)}{\widebar{\txmdensity}(t\mid X, M)}\Bigg( \partial_{t} \cmo(t, M, X) \cdot \right. \\
    & \frac{\partial_{t}\txmdensity(t \mid X, M)}{\widebar{\txdensity}(t'\mid X)}
    +\frac{ \partial_{t}^2 \cmo(t, M, X)  }{\widebar{\txdensity}(t'\mid X) } \cdot   \\
    & \left( \txmdensity(t \mid X, M) + \cmo(t, M, X) - \widebar{\cmo}(t, M, X) \right) \Bigg)\\
    & + \Bigg( (\widebar{\cmo}(t, M, X)-\widebar{\ccmo}(t, t' x)) \cdot \\
    &  \left.  \frac{\partial_{t}^2 \txmdensity(t'\mid X, M)}{\widebar{\txdensity(t'\mid X)}}\Bigg) /2 )\right] \kappa.
\end{align*}

where $\bar{V}_Y = \mathbb{E}\left[|Y-\widebar{\cmo}(T, M, X)|^2 \mid T=t, M, X\right]$, $\bar{V}_{\cmo} = \mathbb{E}\left[(\widebar{\cmo}(t, M, X) - \widebar{\ccmo}(t, t', X))^2 \mid T=t', X\right]$, and where we denote the roughness of $k$ as $R_k \equiv \int_{-\infty}^{\infty} k(u)^2 d u$ and $\kappa \equiv \int_{-\infty}^{\infty} u^2 k(u) d u$.
\label{thm:asymptotic_normality}
\end{thm}

A proof of Theorem \ref{thm:asymptotic_normality} is given in Appendix \ref{appendix:asymptotic_analysis}. Note that the last part in the influence function in \eqref{eq:moment_function} $n^{-1} \sum_{i=1}^n \bar{\ccmo}\left(t, t^\prime, X_i\right)-\mr_{t,t^\prime}=O_p(1 / \sqrt{n})=$ $o_p\left(1 / \sqrt{n h}\right)$ and hence does not contribute to the first-order asymptotic variance $\mathrm{V}_{t,t'}$. 

\begin{rem}
 In Theorem \ref{thm:asymptotic_normality}, the scaling is by the square root of the local sample size $\sqrt{n h}$ rather than by the usual parametric rate $\sqrt{n}$. This slower rate of convergence is a typical bias-variance trade-off as in \citet{Colangelo2020} and \citet{kennedy2017}. Similar to them, the estimator is consistent but not quite centered at $\eta_{t, t'}$; there is a bias term of order $O\left(h^2\right)$. Moreover, balancing the depending  terms on $h$ requires \( h \sim n^{-1/5} \), giving a scaling \( 1 / \sqrt{n h} \sim h^2 \sim n^{-2/5} \) as in \citet{kennedy2017}.
\end{rem}

\paragraph{Comparison with \citet{sani2024}} After we prove our moment function is asymptotically Neyman Orthogonal, our analysis establishes stronger multiply robustness than \citet{sani2024} by requiring milder assumptions on nuisance parameter consistency. Unlike \citet{sani2024}, which assumes all nuisance parameters converge to their true values (see their Assumption 4), we allow for potential inconsistency and prove consistency of the DML estimator under weaker conditions (see our Assumption \ref{assum:convergence}, \ref{assum:nuisance}), similar to \citet{Colangelo2020} and \citet{takatsu2024}.



Our DML estimator uses a kernel for which the bandwidth $h$ needs to be selected. Theorem \ref{thm:asymptotic_normality} is key for inference, such as constructing confidence intervals and estimating the bandwidth $h$ that minimizes the asymptotic mean squared error in a plug-in fashion \citep{byeons1990, Sheather1991}. We propose an estimator for the leading bias $\mathrm{B}_{t, t^\prime}$, inspired by the idea in \citet{powell1996, Colangelo2020}. Let the notation $\hat{\mr}_{t, t^\prime}=\hat{\mr}_{t, t^\prime, b_n}$ be explicit on the bandwidth $b_n$ and

\begin{equation}
\hat{\mathrm{B}}_{t, t^\prime} \equiv \frac{\hat{\mr}_{t, t^\prime, b_n}-\hat{\mr}_{t, t^\prime, \epsilon b_n}}{b_n^2\left(1-\epsilon^2\right)}    
\label{eq:estimator_bias}
\end{equation}

with a pre-specified fixed scaling parameter $\epsilon \in(0,1)$. Corollary \ref{thm:amse_bandwidth} below shows the consistency of $\hat{\mathrm{B}}_{t, t^\prime}$ under Assumption \ref{assum:split_consistency}.

We can estimate the asymptotic variance $\mathrm{V}_{t, t^\prime}$ by the sample variance of the estimated influence function, 

\begin{equation}
\hat{\mathrm{V}}_{t, t^\prime} = b_n n^{-1} \sum_{\ell=1}^L \sum_{i \in I_t} (\hat{\esf}_{t,t', i \ell}^{b_n})^2,
\label{eq:estimator_variance}
\end{equation}

where $\hat{\esf}_{t,t', i \ell}^{b_n}$ is defined in Eq. \eqref{eq:esf_hat}. Then we propose a data-driven bandwidth 

\begin{equation}
\hat{h}_{t, t^\prime} = \left(\hat{\mathrm{V}}_{t, t^\prime} /\left(4 \hat{\mathrm{B}}_{t, t^\prime}^2\right)\right)^{1 /5} n^{-1 /5}    
\end{equation}

to consistently estimate the optimal bandwidth that minimizes the asymptotic mean squared error (AMSE) given in Theorem \ref{thm:asymptotic_normality}. We first state additional assumptions for it.

\begin{assum}
For each $\ell=1, \ldots, L$ and for any $t \in \mathcal{T}$,
\begin{enumerate}
    \item $\left\|{\delta_{\cmo}^\ell}{\delta_{T\mid X}^\ell}\right\|_{{\mathrm{t} X}}=o_p(1)$, $\left\|{\delta_{\cmo}^\ell}{\delta_{\ccmo}^\ell}\right\|_{{\mathrm{t} X}}=o_p(1)$, $\left\|{\delta_{\ccmo}^\ell}{\delta_{T\mid X, M}^\ell}\right\|_{{\mathrm{t} MX}}=o_p(1)$, $\left\|{\delta_{T\mid X}^\ell}^2{\delta_{T\mid X, M}^\ell}\right\|_{{\mathrm{t} MX}}=o_p(1)$,
    \item $\left\|{\delta_{\cmo}^\ell}^2 {\delta_{T\mid X}^\ell}^2\right\|_{{\mathrm{t} X}}=O_p(1)$, $\left\|{\delta_{\cmo}^\ell}^2{\delta_{\ccmo}^\ell}^2\right\|_{{\mathrm{t} X}}=O_p(1)$, $\left\|{\delta_{\ccmo}^\ell}^2{\delta_{T\mid X, M}^\ell}^2\right\|_{{\mathrm{t} MX}}=O_p(1)$, $\left\|{\delta_{T\mid X}^\ell}^2{\delta_{T\mid X, M}^\ell}^2\right\|_{{\mathrm{t} MX}}=O_p(1)$
    \item $\left\|{\delta_{\cmo}^\ell}^2\right\|_{{\mathrm{t} X}}=O_p(1)$, $\left\|{\delta_{\ccmo}^\ell}^2\right\|_{{\mathrm{t} X}}=O_p(1)$, $\left\|{\delta_{T\mid X, M}^\ell}^2\right\|_{{\mathrm{t} MX}}=O_p(1)$, $\left\|{\delta_{T\mid X}^\ell}^2\right\|_{{\mathrm{t} MX}}=O_p(1)$
    \item $\mathrm{E}\left[(Y- \widebar{\cmo}(t, M, X))^4 \mid T=t, M, X\right]$, $\mathrm{E}\left[(\widebar{\cmo}(t, M, X)- \widebar{\ccmo}(t, t^\prime, X))^4 \mid T=t^\prime, M, X\right]$ and their derivatives with respect to $t$ are bounded uniformly over $\left(t^{\prime}, x\right) \in$ $\mathcal{T} \times \mathcal{X}$ and $\int_{-\infty}^{\infty} k(u)^4 d u<\infty$
    \item The bandwidth $b_n \rightarrow 0$ and $n b_n^{5} \rightarrow \infty . \int_{-\infty}^{\infty} k(u) k(u / \epsilon) d u<\infty$ for $\epsilon \in(0,1)$.
\end{enumerate}
\label{assum:split_consistency}
\end{assum}

We are now in position to state another result related to our DML approach.

\begin{cor}[AMSE optimal bandwidth for $\widehat{\mr}_{t, t^\prime}$]
Let the conditions in Theorem \ref{thm:asymptotic_normality} hold. For $t, t^\prime \in \mathcal{T}$, if $\mathrm{B}_{t, t^\prime}$ is non-zero, then the bandwidth that minimizes the asymptotic mean squared error is $h_{t, t^\prime}^*=\left(\mathrm{~V}_{t, t^\prime} /\left(4 \mathrm{~B}_{t, t^\prime}^2\right)\right)^{1 /5} n^{-1 /5}$. Further let Assumption \ref{assum:split_consistency} hold. Then $\hat{\mathrm{V}}_{t, t^\prime}-\mathrm{V}_{t, t^\prime}=o_p(1)$, $\hat{\mathrm{B}}_{t, t^\prime}-\mathrm{B}_{t, t^\prime}=o_p(1)$, and $\hat{h}_{t, t^\prime} / h_{t, t^\prime}^*-1=o_p(1)$.  
\label{thm:amse_bandwidth}
\end{cor}

Assumption \ref{assum:split_consistency} are for the consistency of $\hat{\mathrm{V}}_{t, t^\prime}$. The first one is stricter than Assumption \ref{assum:convergence}, and the second is mild boundedness conditions that are implied when $\widehat{\cmo}_{\ell}$, $\widehat{\ccmo}_{\ell}$, $\widehat{\txdensity}_{\ell}$ and $\widehat{\txmdensity}_{\ell}$ are bounded uniformly, which is often verified in practice. We now state the last corollary of the asymptotic normality of the DML estimator.
\\

\begin{cor}[Asymptotic confidence interval for $\widehat{\mr}_{t, t^\prime}$]
Let the conditions in Corollary \ref{thm:amse_bandwidth} hold. For $t, t^\prime \in \mathcal{T}$, we can construct the usual $(1-\alpha) \times 100 \%$ point-wise confidence interval $\left[\hat{\mr}_{t, t^\prime} \pm \Phi^{-1}(1-\alpha / 2) \sqrt{\hat{\mathrm{V}}_{t, t^\prime} /\left(n h\right)}\right]$, where $\Phi$ is the CDF of $\mathcal{N}(0,1)$.
\label{cor:asymptotic_ci}
\end{cor}




\section{NUMERICAL EVALUATION}

\label{sec:numerical}

This section provides a simulation study to investigate the finite sample behaviour of our methods. The code to reproduce our experiments can be found at \url{https://github.com/houssamzenati/double-debiased-machine-learning-mediation-continuous-treatments}.

In our experiments, for the kernel smoothing, we considered second order Gaussian kernels with kernel bandwidth either with the optimal AMSE bandwidth or a \citet{scott2015multivariate}-type rule of thumb (see Appendix \ref{appendix:hyperparameters}) verifying the regularity conditions in Theorem \ref{thm:asymptotic_normality}. Furthermore, for the estimation of the nuisances parameters we used nonparametric estimators for the conditional and cross conditional mean outcomes $\cmo$ and $\ccmo$ as detailed in Appendix \ref{appendix:nuisances}, following \citet{singh2023sequential} which uses kernel mean embeddings. We performed parametric estimation of the generalized propensity scores, that is to say the conditional densities of the treatment. To this end, as in \citet{hsu2020}, we assume $T$ to be normally distributed given $X$ (or given $(X, M)$, respectively) and learned linear predictors $\hat t$ of $T$ conditional on $X$ (or given $(X, M)$). More specifically, $\widehat \txdensity = \mathcal{N}(\hat t, \hat \sigma)$ (respectively for $\widehat \txmdensity$) where the noise $\hat \sigma$ is the empirical standard deviation of treatments. 

In addition, we also estimate the direct and indirect effects with baseline methods. First, we implement the coefficient product method \citep{baron1986, vanderweele_explanation_2015}, that is a linear OLS regression of the outcome on a the treatment, the mediator, and the covariate and respectively the mediator on a the treatment and covariate. Second, we consider the kernel mean embedding (KME) method \citep{singh2023sequential} which is a G-computation method. We also implement the generalized importance weighting (GIPW) estimator from \citet{hsu2020}. Eventually, we also considered the multiply robust estimator of \citet{sani2024}, with their explicit integration of the conditional mean outcome for estimating the cross conditional mean outcome. We refer the reader to the Appendices~\ref{appendix:nuisances} and~\ref{appendix:experiment_details} for further details on the implementations. 

\subsection{Simulation}

In this part we evaluate our approach following the simulations in \citet{hsu2020}, and \citet{singh2023sequential}
.
In this simulation, the oracle for the mediated response $\mr_{t, t^\prime}$, the direct effect $\dce(t, t^{\prime})$ and the indirect effect by $\ice(t, t^{\prime})$ are accessible (see all details in Appendix~\ref{appendix:experiment_hsu}) and we can therefore evaluate the error of all estimators. We consider 100 simulations and three sample sizes $n=500,1000, 5000$ as in \citet{singh2023sequential} to investigate the performance of our approach. 

\begin{figure}
    \centering
    \includegraphics[width=\linewidth]{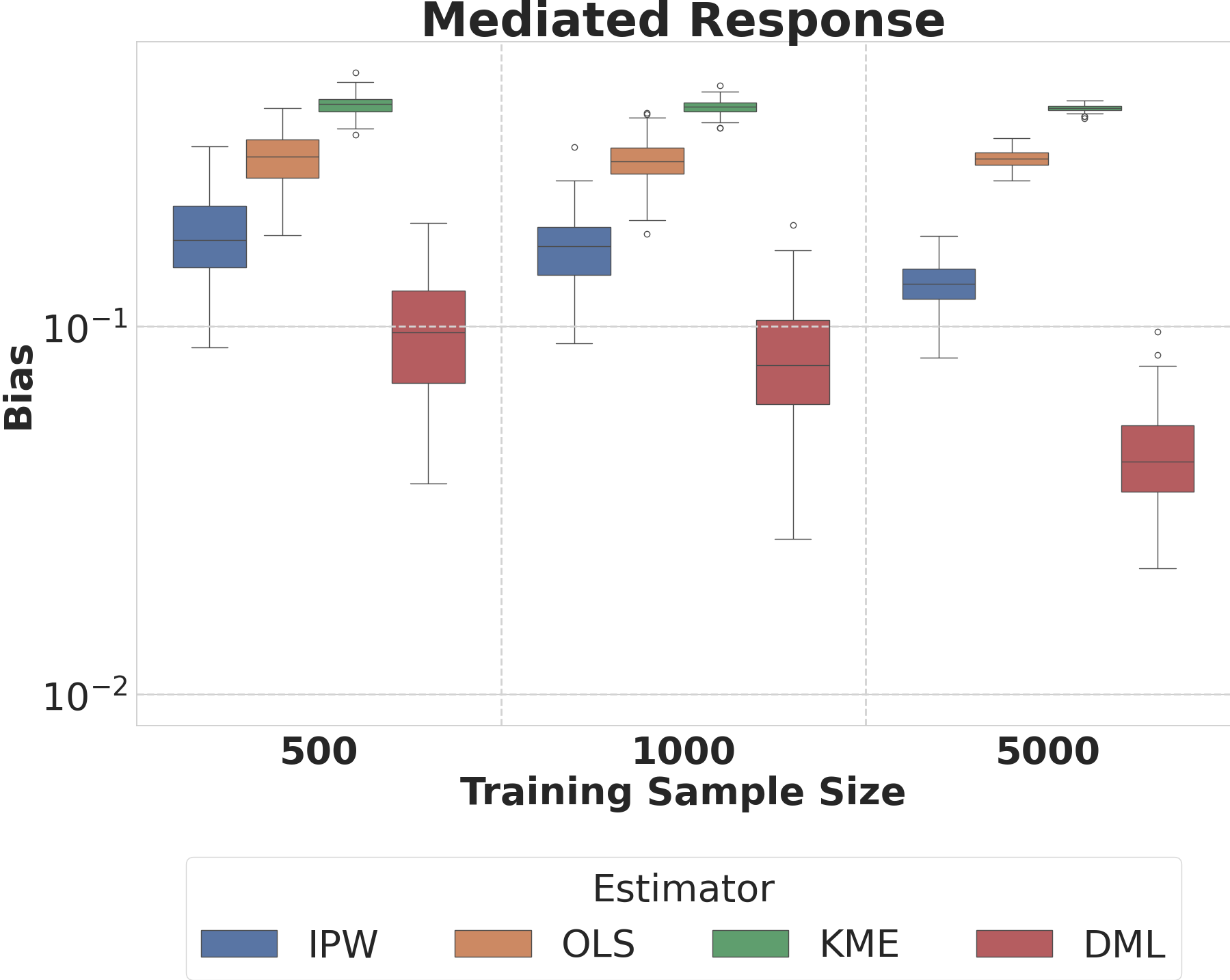}
    \caption{Bias of mediated response estimation on simulations with different sample sizes. DML significantly outperforms OLS and KME and also improves upon IPW.}
    \label{fig:comparison_dml}
\end{figure}

\paragraph{Comparison to other methods} In Figure \ref{fig:comparison_dml} we compare the DML estimator to other approaches. We see that the DML significantly outperforms the OLS and KME approaches while improving upon the IPW method as well. We therefore find as in binary treatment studies \citep{tchetgen2012, farbmacher2022} that the DML estimation of the mediated response practically outperforms alternative approaches.


\paragraph{Comparison to DML variants} 

We compare different DML formulations for continuous treatments. We implemented the estimator from \citet{sani2024}, which estimates the cross-conditional mean outcome $\ccmo$ using Explicit Integration (EI) of $\cmo$ and a ratio of mediator densities $\mdensity$ (MD). (EI) estimates $\widehat \omega_Y(t,t', x) = \sum_{\mathcal{M}} \widehat \mu_Y(t,m,x) \widehat f_{M \mid T, X}(m|t',x)$ (using Eq. \eqref{eq:cross_conditional_mean_outcome}), while Implicit Integration (II) estimates $\widehat \omega_Y$ via two regressions \citep{farbmacher2022, singh2023sequential}. Our approach instead uses (II) of $\mu_Y$ and a ratio of treatment propensities $\txdensity, \txmdensity$. Notably, (EI) is unsuitable for continuous mediators, and the mediator densities (MD) approach struggles with high-dimensional mediators. We also compare our DML estimator to (II) combined with (MD) to highlight its advantages.

\begin{table}[h]
    \centering
    \begin{tabular}{llcccc}
    \toprule
    Estimator & $n$ & $d_M$ & Bias & Std & RMSE \\
    \midrule
    EI-MD & 500 & 1 & 0.2429 & 0.1672 & 0.2979 \\
    II-MD & 500 & 1 & 0.1011 & 0.0389 & 0.1208 \\
    \textbf{(Ours)} & 500 & 1 & \textbf{0.1006} & \textbf{0.0375} & \textbf{0.1211} \\

    \midrule
    EI-MD & 1000 & 1 & 0.3265 & 0.2745 & 0.3890 \\
    II-MD & 1000 & 1 & 0.0834 & 0.0329 & 0.0995 \\
    \textbf{(Ours)} & 1000 & 1 & \textbf{0.0828} & \textbf{0.0319} & \textbf{0.0985} \\
    \midrule
    EI-MD & 5000 & 1 & 0.7319 & 0.6298 & 0.8475 \\
    II-MD & 5000 & 1 & 0.0460 & 0.0146 & 0.0557 \\
    \textbf{(Ours)} & 5000 & 1 & \textbf{0.0449} & \textbf{0.0145} & \textbf{0.0544} \\
    \hline
    II-MD & 500 & 5 & 0.7271 & 0.1742 & 0.9002 \\
    \textbf{(Ours)} & 500 & 5 & \textbf{0.1709} & \textbf{0.0462} & \textbf{0.2178} \\
    \midrule 
    II-MD & 1000 & 5 & 0.5835 & 0.0924 & 0.7278 \\
    \textbf{(Ours)} & 1000 & 5 & \textbf{0.1623} & \textbf{0.0369} & \textbf{0.2122} \\
    \midrule
    II-MD & 5000 & 5 & 0.3393 & 0.0182 & 0.4209 \\
    \textbf{(Ours)} & 5000 & 5 & \textbf{0.1319} & \textbf{0.0173} & \textbf{0.1843} \\
    \bottomrule
    \end{tabular}
    \caption{Comparison to different DML variants on our setting with mediator of dimension $d_M=1$ and $d_M=5$: (EI-MD) explicit integration and mediator densities \citep{sani2024}, (II-MD) implicit integration and mediator densities, (Ours) with (II-TP) implicit integration and treatment propensities}
    \label{tab:merged_dml}
\end{table}

Table \ref{tab:merged_dml} compares these variants for a one-dimensional mediator and of dimension 5. For a one-dimensional mediator our approach outperforms \citet{sani2024} but performs similarly to the II-MD variant, as mediator density estimation is straightforward in this setting.
%
However, for a higher dimension our approach outperforms the II-MD variant, where mediator density estimation is more challenging. In this multidimensional setting, the EI procedure is not suitable. To extend the comparison, we replicated the numerical simulation of \citet{sani2024} in Appendix \ref{appx:sani_setting} using binary mediators, which favor (EI) of the conditional mean outcome. Using their metric—the average absolute bias—we compare our DML estimator to their variant, as well as IPS, OLS, and KME, and still find that our approach significantly outperforms theirs.


 \paragraph{Bandwidth selection} Moreover, in Appendix \ref{appendix:experiment_hsu}, we provide additional results where we compare the empirical performance of our DML approach under two bandwidth selection methods, one following the AMSE strategy from Corollary \ref{thm:amse_bandwidth} and the other following the \citet{scott2015multivariate} rule of thumb. Overall, the AMSE strategy provides less variance but shows slightly more bias than the heuristic. Moreover, we conclude that both approaches are suited for practical applications. 
 
 \paragraph{Additional experiments} Eventually, we also provide in Appendix \ref{appx:additional_results} an experiment to validate the coverage of the asymptotic confidence interval and an additional experiment to compare parametric and nonparametric estimators of the nuisances functions.

\subsection{Application to cognitive function}


We consider UK Biobank \citep{ukbb} imaging data, with around 40,000 brain scans, to study brain health. Prior research shows that while imaging-derived phenotypes (IDPs) have some predictive power for traits like fluid intelligence and neuroticism, adding IDPs to socio-demographic variables yields little to no improvement in predictive performance \citep{dadi2021, cox2019structural}. In this application, we assess whether the effect of poor glycemic control on cognitive function is mediated by brain structure, using glycated hemoglobin (as a proxy for the glycemic control\citep{jha2022accelerated}.
In Appendix~\ref{appendix:experiment_ukbb}, along with additional details, we provide an experiment to measure the effect mediated by brain structural information.  The DML results are on par with existing methods and show no total nor indirect effect of brain-derived features, which corroborates \citet{dadi2021, cox2019structural}. However, we observe greater stability than IPW in regions where there is limited overlap in the data and for which identifiability is therefore not ensured.

\section{DISCUSSION}

In this work, we propose a DML estimator for mediation analysis with continuous treatments, using kernel smoothing \citep{hsu2020, Colangelo2020}. Our estimator is root-n consistent under mild regularity conditions and, as a targeted DML estimator \citep{kennedy2022semiparametric}, exhibits asymptotic normality, allowing for asymptotic confidence intervals and mean squared error optimal bandwidth. We provide numerical evaluations in non-asymptotic regimes and apply the method to cognitive function analysis. Future work could explore targeted learning \citep{vanderLaanRubin+2006, zhengVanderLaan2012} with kernel smoothing, as it has proven effective without requiring cross-fitting, which can reduce effective sample size and does not improve finite-sample performance in simpler problems \citep{williamson2023general, qingliang2022}.

\acknowledgments{HZ and BT were supported by the KARAIB AI chair (ANR-20-CHIA-0025-01) and the H2020 Research Infrastructures Grant EBRAIN-Health 101058516. This research has been conducted using the UK Biobank Resource under Application Number 49314.}

\bibliographystyle{abbrvnat}
\bibliography{references}

\newpage

\section*{Checklist}

 \begin{enumerate}

 \item For all models and algorithms presented, check if you include:
 \begin{enumerate}
   \item A clear description of the mathematical setting, assumptions, algorithm, and/or model. \\
   \textbf{Yes}
   The setting along with asumptions, algorithm are detailed in Section \ref{sec:background} and \ref{sec:continuous_mediation}.
   \item An analysis of the properties and complexity (time, space, sample size) of any algorithm. \\
   \textbf{Yes}
   The analysis of the algorithm presented in Section \ref{sec:continuous_mediation} is shown in Section \ref{sec:asymptotic}.
   \item (Optional) Anonymized source code, with specification of all dependencies, including external libraries. \\
   \textbf{Yes}
   The anonymized source code is provided in the supplementary materials and contains a Read.me file to run the experiments.
 \end{enumerate}

 \item For any theoretical claim, check if you include:
 \begin{enumerate}
   \item Statements of the full set of assumptions of all theoretical results. \\
   \textbf{Yes}
    The analysis of the algorithm presented in Section \ref{sec:continuous_mediation} is shown in Section \ref{sec:asymptotic}.
   \item Complete proofs of all theoretical results. \\
    \textbf{Yes}
    The analysis of the algorithm presented in Section \ref{sec:continuous_mediation} is shown in Section \ref{sec:asymptotic}.
   \item Clear explanations of any assumptions. \\
    \textbf{Yes}
    A commentary and explanation of all assumptions in Section \ref{sec:continuous_mediation} and \ref{sec:asymptotic} are provided.    
 \end{enumerate}

 \item For all figures and tables that present empirical results, check if you include:
 \begin{enumerate}
   \item The code, data, and instructions needed to reproduce the main experimental results (either in the supplemental material or as a URL). \\
    \textbf{Yes}
    The instructions to reproduce the tables, plots are in Read.me file of the code in the supplementary file.
   \item All the training details (e.g., data splits, hyperparameters, how they were chosen). \\
    \textbf{Yes} The training details are provided in Appendix \ref{appendix:hyperparameters} for hyperparameter selection and Appendix \ref{appendix:experiment_hsu} and \ref{appendix:experiment_ukbb} for the remaining.
    \item A clear definition of the specific measure or statistics and error bars (e.g., with respect to the random seed after running experiments multiple times). \\
    \textbf{Yes} The definition of the metrics and the way to obtain error bars are described in Section \ref{sec:numerical} and in  Appendix \ref{appendix:experiment_hsu} and \ref{appendix:experiment_ukbb}.
        \item A description of the computing infrastructure used. (e.g., type of GPUs, internal cluster, or cloud provider). \\
         \textbf{Yes} The computational infrastructure is described in Appendix \ref{appendix:computation infrastructure}.
 \end{enumerate}

 \item If you are using existing assets (e.g., code, data, models) or curating/releasing new assets, check if you include:
 \begin{enumerate}
   \item Citations of the creator If your work uses existing assets. \\
    \textbf{Yes}
    We cite the original paper of the UKBB project \citep{ukbb}.
   \item The license information of the assets, if applicable. \\
   \textbf{No}
    We address the reader to the paper explaining the UKBB project \citep{ukbb} for this matter.
   \item New assets either in the supplemental material or as a URL, if applicable. \\
   \textbf{Not applicable.}
   \item Information about consent from data providers/curators. \\
   \textbf{No}
    We address the reader to the paper explaining the UKBB project \citep{ukbb} for this matter.
   \item Discussion of sensible content if applicable, e.g., personally identifiable information or offensive content. \textbf{No}.
 \end{enumerate}

 \item If you used crowdsourcing or conducted research with human subjects, check if you include:
 \begin{enumerate}
   \item The full text of instructions given to participants and screenshots. \textbf{Not applicable}.
   \item Descriptions of potential participant risks, with links to Institutional Review Board (IRB) approvals if applicable. \textbf{Not applicable}.
   \item The estimated hourly wage paid to participants and the total amount spent on participant compensation. \textbf{Not applicable}.
 \end{enumerate}

 \end{enumerate}

\newpage

\newpage
\onecolumn
\section*{APPENDIX}

This appendix is organized as follows: 
\begin{itemize}[nosep, label={--}]
    \item Appendix~\ref{appendix:notations}: summary of the notations used in the analysis.
    \item Appendix~\ref{appendix:kernel_localization}: discussion on kernel smoothing and proof for the construction of the moment function in Section~\ref{sec:continuous_mediation}.
    \item Appendix~\ref{appendix:asymptotic_analysis}: proof for the asymptotic analysis in Section~\ref{sec:asymptotic}.
    \item Appendix~\ref{appendix:nuisances}: contains the implementation information of the nuisance parameters used in Section~\ref{sec:numerical}.
    \item Appendix~\ref{appendix:experiment_details}: details on the implementation of the algorithms and additional experiment details, discussions and results. 
\end{itemize}

All the code to reproduce our experiment can be found at \url{https://github.com/houssamzenati/double-debiased-machine-learning-mediation-continuous-treatments}.

\section{NOTATIONS}

\label{appendix:notations}
In this appendix, we recall useful notations that are used throughout the paper.

\smallskip

Below are generic notations related to the mediation task:

\smallskip
\begin{itemize}[nosep, label={--},topsep=-4pt]
    \item $T$ is the treatment random variable, $Y$ the outcome random variable, $M$ the mediator(s) and the covariate(s) is $X$. We respectively write $t,t^\prime$ for treatment values, $y, m, x$ for the outcome mediator and covariate.
    \item $\mathcal{Y}, \mathcal{T}, \mathcal{X}$ and $ \mathcal{M}$ are respectively the support of $Y, T, X$ and $M$.
    \item $M(t)$ and $Y(t, M(t^\prime)$ are respectively the potential mediator state and the potential outcome. 
    \item $\mathcal{Z}= \mathcal{Y} \times \mathcal{T} \times \mathcal{X} \times \mathcal{M}$ and $F_{Z}(Z)$ is the cumulative density function over such a space. 
    \item $\txdensity(t\mid X)$ is the conditional density of $T$ given $X$,  $\txmdensity(t\mid X, M)$ is the conditional density of $T$  given $X$ and $M$, $\mdensity(M \mid T, X)$ the conditional density of $M$ given $T$ and $X$ (if $M$ is discrete, this is a conditional probability).
    \item $\tce(t,t'), \dce(t,t')$ and $\ice(t,t')$ are respectively the total, direct and indirect effects for the treatment values $t, t^\prime \in \mathcal{T}$. 
    \item $\mr_{t,t^\prime} = \E [Y\left(t,M(t')\right)]$ is the mediated response. 
    \item $\cmo(t, m, x)= \E [ Y | T=t, M=m, X=x]$ is the conditional mean outcome, and $\ccmo(t,t', x)=\E \left[ \cmo(t, m, X) | T=t', X =x \right]$ is the cross conditional mean outcome.
\end{itemize}

\smallskip

Below are notations related to the continuous mediation task with the kernel and finite samples. 

\smallskip
\begin{itemize}[nosep, label={--},topsep=-4pt]
    \item $k: \mathcal{T} \to \R$ is a bounded positive definite kernel 
    \item $R_k \equiv \int_{-\infty}^{\infty} k(u)^2 d u$ is the roughness of the kernel.
    \item $\kappa \equiv \int_{-\infty}^{\infty} u^2 k(u) d u$.
    \item $h$ is the bandwidth. 
    \item $K_h\left(T-t\right) \equiv \Pi_{j=1}^{d_T} k\left(\left(T_{j}-t_j\right) / h\right) / h^{d_T} $ is the kernel product for smoothing.
    \item $\esf_{t,t^\prime}^h$ is the kernel moment function associated to the bandwidth $h$ and the treatment values $t, t^\prime \in \mathcal{T}$.
    \item $n$ is the total sample size, $L$ is the number of splits and $i$ refers to an index, typically belonging to a split $I_\ell$ for ${\ell \in 1, \dots L}$. $n_\ell$ is the size of the $\ell$-th split. 
    \item $Y_i, T_i, X_i, M_i$ are realizations of the random variables $Y, T, X, M$ for $i \in I_\ell$ for $\ell \in {1, \dots, L}$.
    \item $d_T$ is the dimension of the treatment space.
\end{itemize}

\smallskip

Below are notations related to the asymptotic analysis: 

\smallskip

\begin{itemize}[nosep, label={--},topsep=-4pt]
    \item $o_p(1)$ and $O_p(1)$ respectively refer to convergence and boundedness in probability.
    \item $\left\|\cdot \right\|_{{t M X}}^2    = \int_{\mathcal{M}, \mathcal{X}}\left(\cdot\right)^2 f_{T M X}(t, m, x) d m d x $ is a $L_2(t M X)$ partial norm for a given $t \in \mathcal{T}$.
    \item $\left\|\cdot \right\|_{{X}}^2    = \int_{\mathcal{X}}\left(\cdot \right)^2 f_{X}(x) d x$ is a $L_2(X)$ partial norm.
    \item $\left\|\cdot \right\|_{{t X}}^2 =  \int_{\mathcal{X}}\left(\cdot \right)^2 f_{T X}(t, x) d x$ is a $L_2(tX)$ partial norm for a given $t \in \mathcal{T}$.
    \item $\widehat{\cmo}_{\ell}$, $\widehat{\ccmo}_{\ell}$ $\widehat{\txdensity}_{\ell}$ and $\widehat{\txmdensity}_{\ell}$ are estimators for the nuisance functions $\cmo$, $\ccmo$ $\txdensity$ and $\txmdensity$ associated to the split $\ell \in {1 \dots L}$. 
    \item $\widebar{\cmo}$, $\widebar{\ccmo}$, $\widebar{\txdensity}$ and $\widebar{\txmdensity}$ are their limits. 
    \item $\mathrm{B}_{t, t^\prime}$ and $\mathrm{~V}_{t,t'}$ are respectively the asymptotic bias and variance in the asymptotic normality. 
    \item $\hat{\mathrm{B}}_{t, t^\prime}$ and $\hat{\mathrm{V}}_{t, t^\prime}$ are estimators of the latter quantities. 
    \item $h^*$ is the optimal bandwidth, $b_n$ is another bandwidth.
\end{itemize}

\smallskip
Eventually, we highlight the most important notations related to the proofs used in the Appendix \ref{appendix:kernel_localization} and \ref{appendix:asymptotic_analysis}.

\smallskip

\begin{itemize}[nosep,label={--},topsep=-4pt]
    \item $F^{\epsilon h}$ defines through $\epsilon$ a trajectory CDF between the true CDF $F^0$ and the CDF $F^h$ induced by the bandwidth $h$. $f^{\epsilon h}, f^0, f^h$ are the respective PDF.
    \item $\lambda(t^\prime, X)=1/\txdensity(t^\prime \mid X)$ is the inverse of the propensity score. $r(t, t^\prime, M, X)=\txmdensity(t^\prime, M, X)/\txmdensity(t \mid X, M)$.
    \item $Z_{\ell}^c$ denote the observations $()$ for $i \in I_{\ell}$
    \item $\widehat{\cmo}_{i \ell} = \widehat{\cmo}_{\ell}\left(t, X_i, M_i\right)$, $\widehat{\ccmo}_{i \ell} = \widehat{\ccmo}_{\ell}\left(t, t^\prime, X_i \right)$, $\hat{\lambda}_{i \ell} = \widehat{\txdensity}_{\ell}\left(t^\prime \mid X_i\right)^{-1}$ and $\hat{r}_{i \ell} = \widehat{\txmdensity}_{\ell}\left(t^\prime \mid X_i, M_i \right)/{\widehat{\txmdensity}_{\ell}\left(t \mid X_i, M_i \right)}$ for $i \in I_{\ell}$. 
    \item $\widebar{\cmo}_i = \widebar{\cmo}\left(t, X_i, M_i \right)$, $\widebar{\ccmo}_{i} = \widebar{\ccmo}\left(t, t^\prime, X_i \right)$ , $\bar{\lambda}_{i} = \widebar{\txdensity}_{\ell}\left(t^\prime \mid X_i\right)^{-1}$ and $\bar{r}_{i} = \widebar{\txmdensity}\left(t^\prime \mid X_i, M_i \right)/{\widebar{\txmdensity}\left(t \mid X_i, M_i \right)}$.
\end{itemize}

\section{MOMENT FUNCTION AND KERNEL SMOOTHING}

In this section we provide a discussion on the moment function and kernel smoothing techniques and start with a Lemma that is derived from the analysis of \citet{Colangelo2020}. 

\label{appendix:kernel_localization}
\label{appendix:moment}
\begin{lemma}
Let $W \in \mathcal{W}$ be a random variable and $f_{TW}$ be the joint probability on $\mathcal{T} \times \mathcal{W}$, for any $t \in \mathcal{T}$ 
\begin{equation}
\int_{\mathcal{T}} K_h(s-t) f_{T W}(s, x) d s=f_{T W}(t, w)+O\left(h^2\right)    
\end{equation}
uniformly in $x \in \mathcal{X}$.    
\label{lemma:kernel_smoothing}
\end{lemma}

\begin{proof}
 We use change of variables $u=\left(u_1, u_2, \ldots, u_{d_T}\right)=(T-t) / h$, a Taylor expansion, the mean value theorem where $\bar{t}$ is between $t$ and $t+u h, f_{T W}(t, w)$ being bounded away from zero, and the second derivatives of $f_{T W}(t, w)$ being bounded uniformly over $\left(t^{\prime}, x\right) \in \mathcal{T} \times \mathcal{W}$ to show that

\begin{align}
& \int_{\mathcal{T}} K_h(s-t) f_{T W}(s, w) d s \\
& =\int_{\mathcal{R}^{d_T}} \Pi_{j=1}^{d_T} k\left(u_j\right) f_{T W}(t+u h, w) d u \\
& =\int_{\mathcal{R}^{d_T}} \Pi_{j=1}^{d_T} k\left(u_j\right)\left(f_{T W}(t, w)+h \sum_{j=1}^{d_T} u_j \frac{\partial f_{T W}(t, w)}{\partial t_j}\right. \\
& \left.\quad+\left.\frac{h^2}{2} \sum_{j=1}^{d_T} u_j^2 \frac{\partial^2 f_{T W}(t, w)}{\partial t_j^2}\right|_{t=\bar{t}}+\left.\frac{h^2}{2} \sum_{j=1}^{d_T} \sum_{l=1, l \neq j}^{d_T} u_j u_l \frac{\partial^2 f_{T W}(t, w)}{\partial t_j \partial t_l}\right|_{t=\bar{t}}\right) d u_1 \cdots d u_{d_T} \\
& =\int_{\mathcal{R}^{d_T}} \Pi_{j=1}^{d_T} k\left(u_j\right) f_{T W}(t, w)\left\{1+h \sum_{j=1}^{d_T} u_j \frac{\partial f_{T W}(t, w)}{\partial t_j} \frac{1}{f_{T W}(t, w)}\right. \\
& \left.\quad+\frac{h^2}{2} \sum_{j=1}^{d_T}\left(\left.u_j^2 \frac{\partial^2 f_{T W}(t, w)}{\partial t_j^2}\right|_{t=\bar{t}}+\left.\sum_{l=1, l \neq j}^{d_T} u_j u_l \frac{\partial^2 f_{T W}(t, x)}{\partial t_j \partial t_l}\right|_{t=\bar{t}}\right) \frac{1}{f_{T W}(t, x)}\right\} d u_1 \cdots d u_{d_T} \\
& =f_{T W}(t, w)\left(1+h^2 C\right) \label{eq:kernel_density_approximation} \\
& =f_{T W}(t, w)+O\left(h^2\right)
\end{align}

for some positive constant $C$, for any $t \in \mathcal{T}$, uniformly over $w \in \mathcal{W}$ and where the term linear in $h$ equals zero because $\int u k(u) = 0$ by Assumption \ref{assum:kernel}. 
\end{proof}

The latter result will be at the core of the analysis that will follow in Appendix \ref{appendix:asymptotic_analysis}. We now discuss the construction of the doubly robust moment function with the Gateaux derivatives. Importantly the expression in \eqref{eq:lim_moment_function} will be the building block to construct estimators for the mediated response $\mr_{t,t^\prime}$ and its linear functionals. We will prove Lemma \ref{lemma:neyman_orthogonal} that we restate below.


\begin{customlemma}{\ref{lemma:neyman_orthogonal}}
Let $t, t^\prime \in \mathcal{T}$ and let the moment function
\begin{equation*}
\begin{aligned}
\esf_{t,t^\prime, h}= & \frac{ K_h(T-t) \txmdensity(t' \mid X, M)}{\txdensity(t' \mid X) \txmdensity(t \mid X, M)} [Y-\cmo(t, M, X)] \\
& +\frac{K_h(T-t')}{\txdensity(t' \mid X)} [\cmo(t, M, X) -\ccmo(t,t', X)] \\
& +\ccmo(t,t', X) - \mr_{t,t^\prime}.
\end{aligned}    
\end{equation*}
then, the asymptotic moment function is Neyman-Orthogonal.
\end{customlemma}

\begin{proof}
One way to obtain the influence function is to calculate the limit of the Gateaux derivative with respect to a smooth deviation from the true distribution, as the deviation approaches a point mass, following \citet{carone2016} and \citet{ichimura2022}.
\\

Let $\mathcal{F}$ be the set of cumulative density functions (CDF) on the observations $Z \in \mathcal{Z}$. Let $F^0$ be the true CDF of $Z$ (respectively $f^0$ the PDF) and the CDF $F_Z^h$ (respectively $f^z$ the PDF) approach a point mass at $Z$ as the bandwidth $h \rightarrow 0$. Consider

\begin{equation}
F^{\epsilon h}=(1-\epsilon) F^0+\epsilon F_Z^h,   
\end{equation}

for $\epsilon \in[0,1]$ such that for all small enough $\epsilon, F^{\epsilon h} \in \mathcal{F}$ and the corresponding PDF 

\begin{equation}
f^{\epsilon h}=f^0+\epsilon\left(f_Z^h-f^0\right).   
\end{equation}

For any $t \in \mathcal{T}$, let $\mr_{t,t^\prime}(\cdot): \mathcal{F} \rightarrow \mathcal{R}$. We note that:

$$
\begin{aligned}
\mr_{t,t^\prime}(F) & =\int_{\mathcal{X}, \mathcal{M}, \mathcal{Y}} \mathbb{E}[Y \mid T=t, M=m, X=x] \mdensity(m | t', x) f_X(x) d x d m d y \\
& =\int_{\mathcal{X}, \mathcal{M}, \mathcal{Y}} y \frac{f_Z(y, t, m, x) f_X(x) f_{T M X}(t', m, x) }{f_{T M X}(t, m, x)f_{T X}(t', x)}  d x d m d y
\end{aligned}
$$

\citep{ichimura2022} show that if a semiparametric estimator is asymptotically linear and locally regular, then the influence function is $\lim _{h \rightarrow 0} d \mr_{t,t^\prime}\left(F^{\epsilon h}\right) /\left.d \epsilon\right|_{\epsilon=0}$. We calculate the Gateaux derivative of the functional $\mr_{t,t^\prime}\left(F^{\epsilon h}\right)$ with respect to a deviation $F_Z^h-F^0$ from the true distribution $F^0$. Then we show that our estimator is asymptotically equivalent to a sample average of the moment function.  

$$
\begin{aligned}
\left.\frac{d}{d \epsilon} \mr_{t,t^\prime}\left(F^{\epsilon h}\right)\right|_{\epsilon=0} &=  \left.\int_{\mathcal{X}, \mathcal{M}, \mathcal{Y}} y \frac{d}{d \epsilon}\left(\frac{f_Z(y, t, m, x) f_X(x) f_{T M X}(t', m, x)}{f_{T M X}(t, m, x) f_{T X}(t', x)}\right) d y d m d x\right|_{\epsilon=0} \\
& =  \int_{\mathcal{X}, \mathcal{M}, \mathcal{Y}} \frac{y}{f_{T M X}(t, m, x)f_{T X}(t', x)}\left[\left(f_Z^h(y, t, m, x)-f_Z^0(y, t, m, x)\right) f_{T M X}(t', m, x) f_X(x)\right. \\
& +f_Z(y, t, m, x)\left(f_{TMX}^h(t', m, x)-f_{TMX}^0(t',m, x)\right) f_X(x) \\
& \left.+f_Z(y, t, m, x)\left(f_X^h(x)-f_X^0(x)\right) f_{T M X}(t', m, x) \right] d y d m d x \\
& -\int_{\mathcal{X}, \mathcal{M}, \mathcal{Y}} y \frac{f_Z(y, t, m, x) f_X(x) f_{T M X}(t', m, x)}{\left[f_{T M X}(t, m, x)f_{T X}(t', x)\right]^2}\left(f_{T M X}^h(t, m, x)-f_{T M X}^0(t, m, x)\right)f_{T X}(t', x) d y d m d x \\
& -\int_{\mathcal{X}, \mathcal{M}, \mathcal{Y}} y \frac{f_Z(y, t, m, x) f_X(x) f_{T M X}(t', m, x)}{\left[f_{T M X}(t, m, x)f_{T X}(t', x)\right]^2}\left(f_{T X}^h(t', x)-f_{T X}^0(t', x)\right)f_{T M X}(t, m, x) d y d m d x
\end{aligned}
$$

Specifically the term associated with $f_{T M X}^h$ in the previous term contributes in the total term by calculating
$$
\begin{aligned}
& \lim _{h \rightarrow 0} \int_{\mathcal{X}, \mathcal{M}, \mathcal{Y}} y  \frac{f_Z(y, t, m, x) f_X(x) f_{T M X}(t', m, x)}{f_{T M X}(t, m, x)^2 f_{T X}(t', x)} f_{T M X}^h(t, m, x) d y d m d x \\
& =\lim _{h \rightarrow 0} \int_{\mathcal{X}, \mathcal{M}} \cmo(t, m, x) \frac{f_X(x) f_{T M X}(t', m, x)}{f_{T M X}(t, m, x) f_{T X}(t', x)} f_{T M X}^h(t, m, x) d m d x \\
& =\lim _{h \rightarrow 0} \int_{\mathcal{X}, \mathcal{M}} \frac{\cmo(t, m, x)}{f_{T \mid X}(t' \mid x)} \frac{\txmdensity(t' \mid m, x)}{\txmdensity(t \mid m, x)} f_{T M X}^h(t, m, x) d m d x \\
& =\frac{\cmo(t, M, X)}{f_{T \mid X}(t' \mid X)} \frac{\txmdensity(t' \mid M, X)}{\txmdensity(t \mid M, X)} \lim _{h \rightarrow 0} f_T^h(t)
\end{aligned}
$$

Moreover, the term associated with $f_{T X}^h$ in the previous term contributes in the total term by calculating

$$
\begin{aligned}
& \lim _{h \rightarrow 0} \int_{\mathcal{X}} \int_{\mathcal{M}} \int_{\mathcal{Y}} y  \frac{f_Z(y, t, m, x) f_X(x) f_{T M X}(t', m, x)}{f_{T M X}(t, m, x) f_{T X}^2(t', x)} f_{T X}^h(t', x) d y d m d x \\
& =\lim _{h \rightarrow 0} \int_{\mathcal{X}}  \int_{\mathcal{M}} \cmo(t, m, x) \frac{f_{T M X}(t', m, x)}{f_{T X}(t', x)} \frac{f_X(x) }{f_{T X}(t', x)} f_{T X}^h(t', x) d m d x \\
& =\lim _{h \rightarrow 0} \int_{\mathcal{X}}  \frac{\ccmo(t, t', x)}{f_{T \mid X}(t' \mid x)}  f_{T X}^h(t',x) d x \\
& =\frac{\ccmo(t, t', X)}{f_{T \mid X}(t' \mid X)}  \lim _{h \rightarrow 0} f_T^h(t')
\end{aligned}
$$

With same arguments for the other terms, we obtain that the Gateaux derivative for the direction $f_Z^h-f^0$ is

\begin{equation}
\begin{aligned}
\left.\lim _{h \rightarrow 0} \frac{d}{d \epsilon} \mr_{t,t^\prime}\left(F^{\epsilon h}\right)\right|_{\epsilon=0} & =\ccmo(t, t', X)-\mr_{t,t^\prime}+\lim_{h \rightarrow 0} \int_{\mathcal{X}, \mathcal{M}, \mathcal{Y}} \frac{y-\cmo(t, m, x)}{f_{T \mid X}(t' \mid x)} \frac{\txmdensity(t' \mid m, x)}{\txmdensity(t \mid m, x)} f_Z^h(y, t, x) d y d x \\
& + \lim_{h \rightarrow 0} \int_{\mathcal{X}, \mathcal{M}, \mathcal{Y}} \frac{\cmo(t, m, x) - \ccmo(t,t', x)}{f_{T \mid X}(t' \mid x)}  f_Z^h(y, t', x) d y d x\\
& =\ccmo(t, t', X)-\mr_{t,t^\prime}+\frac{Y-\cmo(t, M, X)}{f_{T \mid X}(t' \mid X)} \frac{\txmdensity(t' \mid M, X)}{\txmdensity(t \mid M, X)} \lim _{h \rightarrow 0} f_T^h(t) \\
& + \frac{\cmo(t, M, X) - \ccmo(t,t', X)}{f_{T \mid X}(t' \mid X)}  \lim _{h \rightarrow 0} f_T^h(t') 
\end{aligned}
\label{eq:lim_moment_function}
\end{equation}

Note that the last term in \eqref{eq:lim_moment_function} is a partial mean that is a marginal integration over $\mathcal{Y} \times \mathcal{M} \times \mathcal{X}$, fixing the value of $T$ at $t$. Thus the Gateaux derivative depends on the choice of $f_T^h$. 
\\

We then choose $f_Z^h(z)=K_h(Z-z) 1\left\{f^0(z)>h\right\}$, following \citet{ichimura2022}, so $\lim _{h \rightarrow 0} f_T^h(t)=$ $\lim _{h \rightarrow 0} K_h(T-t)$. Indeed, we specify $F_Z^h$ following \citep{ichimura2022}. Let $K_h(Z)=\Pi_{j=1}^{d_z} k\left(Z_j / h\right) / h$, where $Z=\left(Z_1, \ldots, Z_{d_z}\right)^{\prime}$ and $k$ satisfies Assumption \ref{assum:kernel} and is continuously differentiable of all orders with bounded derivatives. Let $F^{\epsilon h}=(1-\epsilon) F^0+\epsilon F_Z^h$ with pdf with respect to a product measure given by $f^{\epsilon h}(z)=(1-\epsilon) f^0(z)+\epsilon f^0(z) \delta_Z^h(z)$, where $\delta_Z^h(z)=K_h(Z-z) / f^0(z)$, a ratio of a sharply peaked pdf to the true density. Thus 

\begin{equation*}
f_Z^h(y, t, m, x)=K_h(Y-y) K_h(T-t) K_h(M-m) K_h(X-x)    
\end{equation*}

It follows that $\lim _{h \rightarrow 0} f_T^h(t)=\lim _{h \rightarrow 0} K_h(T-t)$ and

\begin{align*}
\lim _{h \rightarrow 0} \int_{\mathcal{X}, \mathcal{M}, \mathcal{Y}} & \frac{y-\cmo(t, m, x)}{f_{T \mid X}(t' \mid x)} \frac{\txmdensity(t' \mid m, x)}{\txmdensity(t \mid m, x)} f_Z^h(y, t, m, x) d y d m d x \\
&=\frac{Y-\cmo(t, M, X)}{f_{T \mid X}(t' \mid X)} \frac{\txmdensity(t' \mid M, X)}{\txmdensity(t \mid M, X)} \lim _{h \rightarrow 0} K_h(T-t)    
\end{align*}

\begin{align*}
\mathbb{E}\left[\left.\frac{d}{d \epsilon} \mr_{t,t^\prime}\left(F^{\epsilon h}\right)\right|_{\epsilon=0}\right]=& \mathbb{E}\left[\ccmo(t, t', X)-\mr_{t,t^\prime}+\frac{\cmo(t, M, X)-\ccmo(t, t', X)}{\txdensity(t' \mid X)} K_h(T-t') \right. \\
& + \left. \frac{Y-\cmo(t, M, X)}{\txdensity(t^\prime \mid X)} \frac{\txmdensity(t' \mid M, X)}{\txmdensity(t \mid M, X)} K_h(T-t)\right]
\end{align*}

As in Lemma \ref{lemma:kernel_smoothing}, we can show that

\begin{align*}
    & \mathbb{E}\left[\frac{Y-\cmo(t, M, X)}{\txdensity(t^\prime \mid X)} \frac{\txmdensity(t' \mid M, X)}{\txmdensity(t \mid M, X)} K_h(T-t)\right] \\ &= \mathbb{E}\left[\mathbb{E}\left[Y-\cmo(t, M, X) K_h(T-t) \mid M, X\right]\frac{\txmdensity(t' \mid M, X)}{\txmdensity(t \mid M, X)\txdensity(t^\prime \mid X)} \right] \\
    & \mathbb{E}\left[\mathbb{E}\left[Y-\cmo(t, M, X) \mid M, X\right]\frac{\txmdensity(t' \mid M, X)\txmdensity(t \mid M, X)}{\txmdensity(t \mid M, X)\txdensity(t^\prime \mid X)} \right] + O\left(h^2\right)\\
    &= O\left(h^2\right)
\end{align*}

and similarly by definition of $\cmo$ and $\ccmo$:

\begin{align*}
    \mathbb{E}\left[\frac{\cmo(t, M, X)-\ccmo(t, t', X)}{\txdensity(t' \mid X)} K_h(T-t')\right] &= \mathbb{E}\left[\frac{\left(\cmo(t, M, X)-\ccmo(t, t', X) \right)\txmdensity(t' \mid X, M)}{\txdensity(t' \mid X)} \right] + O\left(h^2\right) \\
    &= O\left(h^2\right)
\end{align*}

Therefore, 

\begin{equation*}
    \mathbb{E}\left[\left.\frac{d}{d \epsilon} \mr_{t,t^\prime}\left(F^{\epsilon h}\right)\right|_{\epsilon=0}\right]=O\left(h^2\right)   
\end{equation*}

So Neyman orthogonality holds when $h \rightarrow 0$.

\end{proof}

Note also that another motivation for this moment function lies in the common ombination of two alternative estimators for $\mr_{t, t^\prime}$, the regression estimator $\hat{\mr}_{t, t^\prime}^{R E G}=n^{-1} \sum_{i=1}^n \hat{\mr}\left(t, M_i, X_i\right)$ that is based on the identification in Eq.\eqref{eq:pearl_formula} in the main text, and the inverse probability weighting (IPW) estimator $\hat{\mr}_{t, t^\prime}^{I P W}=$ $n^{-1} \sum_{i=1}^n K_h\left(T_i-t\right) Y_i \frac{\widehat{\txmdensity}(t^\prime, X_i, M_i)}{\widehat{\txmdensity}(t, X_i, M_i) \widehat{\txdensity}\left(t^\prime \mid X_i\right)}$ that is based on the identification in \ref{eq:ipw_generalized} in the main text. Adding the influence function that accounts for the one-step estimation partials out the first order effect of the one-step estimation on the final estimator, as discussed in \citep{vanderLaanRubin+2006, chernozhukov2022}. For $\hat{\mr}_{t, t^\prime}^{I P W}$, when $\widehat{\txdensity}$ and $\widehat{\txmdensity}$ are standard kernel density estimator with bandwidth $h$, \citet{hsu2020} derive the asymptotically linear representation of $\hat{\mr}_{t, t^\prime}^{I P W}$ that is first-order equivalent to our DML estimator. The moment function here is constructed by adding the influence function adjustments for estimating the nuisance functions $\txdensity, \txmdensity$ to the original moment function $K_h\left(T-t\right) Y \frac{\widehat{\txmdensity}(t^\prime, X, M)}{\widehat{\txmdensity}(t, X, M) \widehat{\txdensity}\left(t^\prime \mid X\right)}$.

Moreover, note that in binary treatment, under Assumptions \ref{assum:indep_treatment}, \ref{assum:indep_mediator}, and \ref{assum:overlap}, the counterfactual mediated response $\E[Y(t, M(t^\prime))]$ is identified by the following efficient moment function:

\begin{equation}
\begin{aligned}
\E[Y(t, M(t^\prime))]= & \E\left[\esf_{t,t^\prime} \right], \\
\text { with } \esf_{t,t^\prime}= & \frac{\mathbf{1}\{T=t\} \cdot \mdensity(M \mid t', X)}{\Prb(T=t |X) \cdot \mdensity(M \mid t, X)} \cdot[Y-\cmo(t, M, X)] \\
& +\frac{\mathbf{1}\{T=t^\prime\}}{\Prb(T=t' | X)} \cdot[\cmo(t, M, X) \\
& \left.-\int_{m \in \mathcal{M}} \cmo(t, m, X) \cdot \mdensity(m \mid t^\prime, X) d m\right] \\
& +\int_{m \in \mathcal{M}} \cmo(t, m, X) \cdot \mdensity(m \mid t^\prime, X) d m
\end{aligned} 
\label{eq:efficient_score_function}
\end{equation}

Moreover, the efficient score function can alternatively be written as:

\begin{equation}
\begin{aligned}
\mr(t,t')= & \E\left[\esf_{t,t'}^*\right], \\
\text { with } \esf_{t,t'}^*= & \frac{\mathbf{1}\{T=t\} \cdot\left(\Prb(T=t' |M,X)\right)}{\Prb(T=t |M,X) \cdot\left(\Prb(T=t' | X)\right)} \cdot[Y-\cmo(t, M, X)] \\
& +\frac{ \mathbf{1}\{T=t'\}}{\Prb(T=t' | X)} \cdot[\cmo(t, M, X)-\E[\cmo(t, M, X) \mid T=t', X]] \\
& +\E[\cmo(t, M, X) \mid T=t', X].
\end{aligned} 
\label{eq:efficient_score_function_alternative}
\end{equation}

This comes from:
\begin{equation}
\E \left[ Y(t, M(t')) \right] = \E \left[ Y \dfrac{ \mathbf{1} \{T=t \} }{\Prb(T=t |M, X)}  \dfrac{\Prb(T=t' |M,X)}{\Prb(T=t' |X)} \right]    
\end{equation}

and Bayes' law:

\begin{align}
    \dfrac{\Prb(T=t' | X, M)}{\Prb(T=t | X, M)\Prb(T=t' | X)} &= \frac{1}{\Prb(T=t | X)} \dfrac{\Prb(T=t' | X, M) f_M(M| X)}{\Prb(T=t' | X)}\dfrac{\Prb(T=t | X)}{\Prb(T=t | X, M)f_M(M| X)} \\
    &= \dfrac{f(M| T=t', X)}{\Prb(T=t | X)f(M| T=t, X)} \\
\end{align}

\section{ASYMPTOTIC ANALYSIS}

\label{appendix:asymptotic_analysis}

In this section we will prove the main analysis results of this paper. We start by restating Assumption \ref{assum:nuisance} with multidimensional $d_T$.

\begin{assum}[Nuisance rates]
There exist functions $\widebar{\cmo}$, $\widebar{\ccmo}$, $\widebar{\txmdensity}$ and $\widebar{\txdensity}$ that satisfy the following: For each $\ell=1, \ldots, L$, 
\begin{enumerate}

    \item $\sqrt{n h^{d_T}}\left\|\delta_{T\mid X}^\ell\right\|_{{t X}}\left\|\delta_{\cmo}^\ell\right\|_{{t MX}} =o_p(1)$
    
    \item $\sqrt{n h^{d_T}} \left\|\delta_{T\mid X, M}^\ell\right\|_{{t M X}}\left\|\delta_{\cmo}^\ell\right\|_{{t MX}} =o_p(1)$
    
    \item $\sqrt{n h^{d_T}}\left\|\delta_{T\mid X}^\ell\right\|_{{t X}}\left\|\delta_{\ccmo}^\ell\right\|_{{X}} =o_p(1)$
\end{enumerate}
\end{assum}

We will now state again Theorem \ref{thm:asymptotic_normality} with general dimension $d_T$ of the treatment space $\mathcal{T}$ and prove it.

\begin{customthm}{\ref{thm:asymptotic_normality}}
Let Assumptions \ref{assum:indep_treatment}-\ref{assum:nuisance} hold. Let $h \rightarrow 0, n h^{d_T} \rightarrow \infty$, and $n h^{d_T+4} \rightarrow C \in[0, \infty)$. Then for any $t \in \mathcal{T}$,
\begin{align*}
\sqrt{n h^{d_T}}\left(\hat{\mr}_{t,t'}-\mr_{t,t^\prime}\right)=\sqrt{\frac{h^{d_T}}{n}} \sum_{i=1}^n & \left[
\frac{  K_h(T_i-t) \cdot \widebar \txmdensity_{t^\prime, i}}{\widebar\txdensity_i \cdot \widebar \txmdensity_{t, i}} \cdot[Y-\widebar\cmo_i]  + \frac{K_h(T_i-t')}{\widebar\txdensity_i} \cdot[\widebar\cmo_i -\widebar\ccmo_i] + \widebar \ccmo_i
\right]\\
& +o_p(1) .
\end{align*}

Further $\mathbb{E}\left[|Y-\widebar{\cmo}(T, M, X)|^3 \mid T=t, M, X\right]$, $\mathbb{E}\left[(\widebar{\cmo}(t, M, X) - \widebar{\ccmo}(t, t', X))^3 \mid T=t', X\right]$ and their derivatives with respect to $t$ be bounded uniformly over $\left(t, m, x\right) \in \mathcal{T} \times \mathcal{M} \times \mathcal{X}$. Let $\int_{-\infty}^{\infty} k(u)^3 d u<\infty$. Then
\begin{equation*}
\sqrt{n h^{d_T}}\left(\hat{\mr}_{t,t'}-\mr_{t,t^\prime}-h^2 \mathrm{~B}_{t,t'}\right) \xrightarrow{d} \mathcal{N}\left(0, \mathrm{~V}_{t,t'}\right)    
\end{equation*}

where

\begin{align*}
    \mathrm{~V}_{t,t'} &\equiv \mathbb{E}\left[\mathbb{E}\left[(Y-\widebar{\cmo}(t, M, X))^2 \mid T=t, M, X\right] \frac{\txmdensity(t, X, M)}{\widebar{\txdensity}(t', X)^2} \frac{\widebar{\txmdensity}(t', X, M)^2}{\widebar{\txmdensity}(t, X, M)^2} \right. \\
    &\quad + \frac{\txdensity(t'\mid X)}{\widebar{\txdensity}(t'\mid X)^2} \mathbb{E}\left[(\widebar{\cmo}(t, M, X) - \widebar{\ccmo}(t, t', X))^2 \mid T=t', X\right]  \Bigg] R_k^{d_T} \\
    \mathrm{B}_{t, t^\prime} &\equiv \sum_{j=1}^{d_T} \mathbb{E}\left[\frac{\widebar{\txmdensity}(t', X, M)}{\widebar{\txmdensity}(t, X, M)}\left(\partial_{t_j} \cmo(t, M, X)  \frac{\partial_{t_j}\txmdensity(t \mid X, M)}{\widebar{\txdensity}(t'\mid X)} \right. \right.\\ 
    +&\partial_{t_j}^2 \cmo(t, M, X) \frac{\left( \txmdensity(t \mid X, M) + \cmo(t, M, X) - \widebar{\cmo}(t, M, X) \right)}{\widebar{\txdensity}(t' \mid X)} \\
    &\left. +(\widebar{\cmo(}t, M, X)-\widebar{\ccmo}(t, t' x)) \partial_{t_j}^2 \frac{\txmdensity(t', X, M)}{\widebar{\txdensity(t'\mid X)}}\Bigg) /2 )\right] \kappa
\end{align*}
\end{customthm}

\begin{proof}

The proof is twofolds: 

\begin{itemize}
    \item Asymptotical linear representation of the estimated mediated response $\hat{\eta}_{t,t'}$
    \item Asymptotic normality
\end{itemize}

\subsubsection{Asymptotical linear representation of the estimated mediated response $\hat \mr$}

We give an outline of deriving the asymptotically linear representation in Theorem \ref{thm:asymptotic_normality}. Let the moment function for identification 
\begin{equation}
m\left(Z_i, \mr_{t, t^\prime}, \ccmo \right) \equiv \ccmo\left(t, t^\prime, X_i\right)-\mr_{t, t^\prime}
\label{eq:ccmo_moment}
\end{equation}
 by \eqref{eq:pearl_formula}, $\mathbb{E}\left[m\left(Z_i, \mr_{t, t^\prime}, \ccmo\left(t, t^\prime X_i\right)\right)\right]=0$ uniquely defines $\beta_t$. Let the adjustment terms

 \begin{align}
\phi_{t}\left(Z_i, \mr_{t, t^\prime}, \cmo, \txdensity, \txmdensity\right) &\equiv \frac{K_h\left(T_i-t\right)\txmdensity(t^\prime \mid x, m)}{\txdensity(t^\prime \mid x)\txmdensity(t \mid x, m)} \left(Y_i-\cmo\left(t, X_i, \psi_i \right)\right) \label{eq:outcome_moment}  \\  
\varphi_{t^\prime}\left(Z_i, \mr_{t, t^\prime} \cmo, \ccmo, \txdensity \right) &\equiv \frac{K_h\left(T_i-t^\prime \right)}{\txdensity(t^\prime \mid x)} \left(\cmo\left(t, X_i, \psi_i \right) - \ccmo\left(t, t^\prime, X_i \right)\right)  \label{eq:cmo_moment}
 \end{align}

The doubly robust moment function then writes as 
$$\psi\left(Z_i, \mr_{t, t^\prime} \cmo, \ccmo, \txdensity, \txmdensity \right) \equiv m\left(Z_i, \mr_{t, t^\prime}, \ccmo \right)+\phi_t \left(Z_i, \mr_{t, t^\prime}, \cmo, \txdensity, \txmdensity \right) + \varphi_{t^\prime}\left(Z_i, \mr_{t, t^\prime} \cmo, \ccmo, \txdensity \right).$$

Let $r_i \equiv \frac{\txmdensity\left(t^{\prime} \mid M_{i}, X_{i}\right)}{\txmdensity\left(t \mid M_{i}, X_{i}\right)}$, $Z^{\ell}$ denote the observations $z_i$ for $i \in I_{\ell}$. Let $\widehat{\cmo}_{i \ell} \equiv \widehat{\cmo}_{\ell}\left(t, X_i, M_i\right)$, $\widehat{\ccmo}_{i \ell} \equiv \widehat{\ccmo}_{\ell}\left(t, t^\prime, X_i \right)$ and $\widehat{\txdensity}_{i \ell} \equiv \widehat{\txdensity}_{\ell}\left(t^\prime \mid X_i\right)$, $\hat{\lambda}_{i \ell} \equiv \widehat{\txdensity}_{\ell}\left(t^\prime \mid X_i\right)^{-1}$, $\hat{r}_{i \ell} \equiv \frac{\widehat{\txmdensity}_{\ell}\left(t^\prime \mid X_i, M_i \right)}{\widehat{\txmdensity}_{\ell}\left(t \mid X_i, M_i \right)}$ using $Z^{\ell}$ for $i \in I_{\ell}$. Let $\widebar{\cmo}_i \equiv \widebar{\cmo}\left(t, X_i, M_i \right)$, $\widebar{\ccmo}_{i} \equiv \widebar{\ccmo}\left(t, t^\prime, X_i \right)$ and $\widebar{\txdensity}_i \equiv \widebar{\txdensity}\left(t^\prime, X_i\right)$, $\bar{\lambda}_{i} \equiv \widebar{\txdensity}_{\ell}\left(t^\prime \mid X_i\right)^{-1}$, $\bar{r}_{i} \equiv \frac{\widebar{\txmdensity}\left(t^\prime \mid X_i, M_i \right)}{\widebar{\txmdensity}\left(t \mid X_i, M_i \right)}$. 
We can write, with $n_{\ell}=n / L$. Then 

\begin{align*}
\widehat{\mr}_{t, t^\prime} &=L^{-1} \sum_{\ell=1}^L \widehat{\mr}_{t, t^\prime, \ell} \\ 
\widehat{\mr}_{t, t^\prime, \ell}&=n_{\ell}^{-1} \sum_{i \in I_{\ell}} \psi\left(Z_i, \mr_{t, t^\prime}, \widehat{\cmo}_{i \ell}, \widehat{\ccmo}_{i \ell}, \widehat{\txdensity}_{i \ell}, \widehat{\txmdensity}_{i \ell}\right)+\mr_{t, t^\prime}
\end{align*}

\begin{equation}
\sqrt{n h^{d_T}}\left(\widehat{\mr}_{t, t^\prime}-\mr_{t, t^\prime}\right)=\sqrt{n h^{d_T}} L^{-1} \sum_{\ell=1}^L\left(\widehat{\mr}_{t, t^\prime, \ell}-\mr_{t, t^\prime}\right)=
L^{-1 / 2} \sum_{\ell=1}^L \sqrt{n_{\ell} h^{d_T}}\left(\widehat{\mr}_{t, t^\prime, \ell}-\mr_{t, t^\prime}\right)    
\end{equation}

 We show below $\sqrt{n_{\ell} h^{d_T}}\left(\widehat{\mr}_{t, t^\prime, \ell}-\mr_{t, t^\prime}\right)=\sqrt{h^{d_T} / n_{\ell}} \sum_{i \in I_{\ell}} \psi\left(Z_i, \mr_{t, t^\prime}, \widebar{\cmo}_{i}, \widebar{\ccmo}_{i}, \widebar{\txdensity}_{i}, \widebar{\txmdensity}_{i}\right)+ o_p(1)$ for each $\ell \in\{1, \ldots, L\}$.
 \\
 
 Since $L$ is fixed and $\left\{I_{\ell}\right\}_{\ell=1, \ldots, L}$ are randomly partitioned distinct subgroups, the result follows from 
\begin{align*}
\sqrt{n h^{d_T}}\left(\widehat{\mr}_{t, t^\prime}-\mr_{t, t^\prime}\right)&=L^{-1 / 2} \sum_{\ell=1}^L \sqrt{n_{\ell} h^{d_T}}\left(\widehat{\mr}_{t, t^\prime, \ell}-\mr_{t, t^\prime}\right)\\
&=L^{-1 / 2} \sum_{\ell=1}^L \sqrt{h^{d_T} / n_{\ell}} \sum_{i \in I_{\ell}} \psi\left(Z_i, \mr_{t, t^\prime}, \widebar{\cmo}_{i}, \widebar{\ccmo}_{i}, \widebar{\txdensity}_{i}, \widebar{\txmdensity}_{i}\right)+o_p(1) \\
&=\sqrt{h^{d_T} / n} \sum_{\ell=1}^L \sum_{i \in I_{\ell}} \psi\left(Z_i, \mr_{t, t^\prime}, \widebar{\cmo}_{i}, \widebar{\ccmo}_{i}, \widebar{\txdensity}_{i}, \widebar{\txmdensity}_{i}\right)+o_p(1)
\end{align*}
\\

We decompose the remainder term for each $\ell \in\{1, \ldots, L\}$,
\begin{align*}
& \sqrt{n_{\ell} h^{d_T}} \frac{1}{n_{\ell}} \sum_{i \in I_{\ell}}\left\{\psi\left(Z_i, \mr_{t, t^\prime}, \widehat{\cmo}_{i \ell}, \widehat{\ccmo}_{i \ell}, \widehat{\txdensity}_{i \ell}, \widehat{\txmdensity}_{i \ell}\right)-\psi\left(Z_i, \mr_{t, t^\prime}, \widebar{\cmo}_{i}, \widebar{\ccmo}_{i}, \widebar{\txdensity}_{i}, \widebar{\txmdensity}_{i}\right)\right\} \nonumber \\
& =K_{h}\left(T_{i}-t\right)\left\{
\hat r_{i \ell} \hat{\lambda}_{i}\left\{Y_{i}-\widehat{\cmo}_{i \ell}\right\}-
\bar r_{i} \bar{\lambda}_{i } \left\{Y_{i}-\widebar{\cmo}_{i }\right\}\right\} +K_{h}\left(T_{i}-t^{\prime}\right)\left\{\hat{\lambda}_{i \ell}\left\{\widehat{\cmo}_{i \ell}-\widehat{\ccmo}_{i \ell} \right\}-\bar{\lambda}_{i}\left\{ \widebar{\cmo}_{i}-\widebar{\ccmo}_{i} \right\}\right\} \\
&+ \widehat{\ccmo}_{i \ell}-\widebar{\ccmo}_{i} 
\end{align*}


The method for the analysis below lies on the decomposition of the outcome residual:

\begin{equation}
\begin{aligned}
\hat r_{i \ell} \hat{\lambda}_{i}\left\{Y_{i}-\widehat{\cmo}_{i \ell}\right\}-
\bar r_{i } \bar{\lambda}_{i } \left\{Y_{i}-\widebar{\cmo}_{i }\right\} &= \left( \hat r_{i \ell} - \bar r_{i } \right) \left( \hat \lambda_{i \ell} - \bar \lambda_{i} \right) \left( Y_i - \widebar \cmo_{i} \right) -  \left( \hat r_{i \ell} - \bar r_{i } \right) \left( \hat \lambda_{i \ell} - \bar \lambda_{i } \right) \left( \widehat \cmo_{i \ell} - \widebar \cmo_{i } \right)  \\
&  - \left( \hat r_{i \ell} - \bar r_{i } \right) \left( \widehat \cmo_{i \ell} - \widebar \cmo_{i } \right) \bar \lambda_{i } - \left( \hat \lambda_{i \ell} - \bar \lambda_{i } \right) \left( \widehat \cmo_{i \ell} - \widebar \cmo_{i } \right) \bar r_{i } \\
&  + \left( \hat r_{i \ell} - \bar r_{i } \right) \left( Y_i - \widebar \cmo_{i } \right)  \bar \lambda_{i } + \left( \hat \lambda_{i \ell} - \bar \lambda_{i } \right) \left( Y_i - \widebar \cmo_{i } \right) \bar r_{i } - \left( \widehat \cmo_{i \ell} - \widebar \cmo_{i} \right)    \bar r_{i } \bar \lambda_{i }, 
\end{aligned}   
\label{eq:decomposition_outcome_residual}
\end{equation}

and the decomposition of the conditional mean outcome residual:

\begin{equation}
\begin{aligned}
\hat{\lambda}_{i}\left\{\widehat{\cmo}_{i \ell} - \widehat{\ccmo}_{i \ell}\right\}-
\bar{\lambda}_{i \ell} \left\{Y_{i}-\widebar{\cmo}_{i \ell}\right\} & =   \left( \hat \lambda_{i \ell} - \bar \lambda_{i} \right) \left( \widehat \cmo_{i \ell} - \widebar \cmo_{i} \right)  - \left( \hat \lambda_{i \ell} - \bar \lambda_{i \ell} \right) \left( \widehat \ccmo_{i \ell} - \widebar \ccmo_{i} \right) + \left( \hat \lambda_{i \ell} - \bar \lambda_{i} \right)  \widebar \cmo_{i} \\
& + \left( \widehat \cmo_{i \ell} - \widebar \cmo_{i} \right)  \widebar \lambda_{i}  - \left( \widehat \ccmo_{i \ell} - \widebar \ccmo_{i} \right)  \bar \lambda_{i} - \left( \hat \lambda_{i \ell} - \bar \lambda_{i} \right)  \widebar \ccmo_{i}. \\ 
\end{aligned}   
\label{eq:decomposition_cmo_residual}
\end{equation}

This intuition is to use such decompositions to leverage the assumptions on convergence and multiply robust inference.

We then expand the remainder into the terms: 

\begin{align*}
& \sqrt{n_{\ell} h^{d_T}} \frac{1}{n_{\ell}} \sum_{i \in I_{\ell}} \Biggr[ K_{h}\left(T_{i}-t\right)\left\{
\hat r_{i \ell} \hat{\lambda}_{i}\left\{Y_{i}-\widehat{\cmo}_{i \ell}\right\}-
\bar r_{i} \bar{\lambda}_{i } \left\{Y_{i}-\widebar{\cmo}_{i }\right\}\right\} \\
& +K_{h}\left(T_{i}-t^{\prime}\right)\left\{\hat{\lambda}_{i \ell}\left\{\widehat{\cmo}_{i \ell}-\widehat{\ccmo}_{i \ell} \right\}-\bar{\lambda}_{i}\left\{ \widebar{\cmo}_{i}-\widebar{\ccmo}_{i} \right\}\right\} +  \widehat{\ccmo}_{i \ell}-\widebar{\ccmo}_{i} \Biggr] \\
&= \sqrt{\frac{h^{d_T}}{n_{\ell}}} \sum_{i \in I_{\ell}} \Biggr[ \widehat{\ccmo}_{i \ell}- \widebar{\ccmo}_i-\mathbb{E}\left[\widehat{\ccmo}_{i \ell}-\widebar{\ccmo}_i \mid Z^{\ell}\right]  \label{eq:R1-1} \tag{R1-1} \\
&-K_{h}\left(T_{i}-t^{\prime}\right)\left(\widehat{\ccmo}_{i \ell} -\widebar \ccmo_i \right) \bar \lambda_i -\mathbb{E}\left[K_{h}\left(T_{i}-t^{\prime}\right)\left(\widehat{\ccmo}_{i \ell} -\widebar \ccmo_i\right) \bar \lambda_i  \mid Z_{i_{\ell}}^{c}\right] \label{eq:R1-2} \tag{R1-2}\\
& +\mathbb{E}\left[\left(\widehat{\ccmo}_{i \ell} -\widebar \ccmo_i\right)\left(1- K_{h}\left(T_{i}-t^{\prime}\right)\bar \lambda_i  \right) \mid Z_{i_{\ell}}^{c}\right] \label{eq:MR-1} \tag{MR-1} \\
& -K_h\left(T_i-t\right)\left(\widehat{\cmo}_{i \ell}-\widebar{\cmo}_i\right) \bar \lambda_i \bar r_i -\mathbb{E}\left[K_h\left(T_i-t\right)\left(\widehat{\cmo}_{i \ell}-\widebar{\cmo}_i\right) \bar \lambda_i \bar r_i \mid Z_{i_{\ell}}^c\right] \label{eq:R1-3} \tag{R1-3} \\
& + K_h\left(T_i-t^{\prime}\right)\left(\widehat{\cmo}_{i \ell}-\widebar{\cmo}_i\right) \bar \lambda_i-\mathbb{E}\left[K_h\left(T_i-t^{\prime}\right)\left(\widehat{\cmo}_{i \ell}-\widebar{\cmo}_i\right) \bar \lambda_i \mid Z_{i_i}^c\right] \label{eq:R1-4} \tag{R1-4} \\
& + \mathbb{E}\left[\left(\widehat{\cmo}_{i \ell}-\widebar{\cmo}_i\right)\left\{K_h\left(T_i-t^{\prime}\right) \bar \lambda_i-K_h\left(T_i-t\right) \bar \lambda_i \bar r_i\right\} \mid Z_{i_{\ell}}^c\right] \label{eq:MR-2} \tag{MR-2} \\
 & + K_{h}\left(T_{i}-t^{\prime}\right)\left(\hat{\lambda}_{i \ell}-\bar \lambda_i  \right) \left(\widebar\cmo_i - \widebar \ccmo_i\right) -\mathbb{E}\left[K_{h}\left(T_{i}-t^{\prime}\right)\left(\hat{\lambda}_{i \ell}-\bar \lambda_i \right) \left( \widebar\cmo_i - \widebar \ccmo_i \right) \mid Z_{i_{\ell}}^{c}\right] \label{eq:R1-5} \tag{R1-5}\\
& +\mathbb{E}\left[K_{h}\left(T_{i}-t^{\prime}\right)\left(\hat{\lambda}_{i \ell}-\bar \lambda_i  \right)\left\{\widebar\cmo_i-\widebar \ccmo_i\right\} \mid Z_{i_{\ell}}^{c}\right] \label{eq:MR-3} \tag{MR-3}\\
&+K_{h}\left(T_{i}-t\right)\left(\hat{r}_{i \ell}-\bar r_i\right) \bar{\lambda}_i\left(Y_{i}-\widebar\cmo_i\right)-\mathbb{E}\left[K_{h}\left(T_{i}-t\right)\left(\hat{r}_{i \ell}-\bar r_i\right) \bar{\lambda}_i\left(Y_{i}-\widebar\cmo_i\right) \mid Z_{i_{\ell}}^{c}\right] \label{eq:R1-6} \tag{R1-6}\\
& + \mathbb{E}\left[K_{h}\left(T_{i}-t\right)\left(\hat{r}_{i \ell}-\bar r_i\right) \bar{\lambda}_i\left(Y_{i}-\widebar\cmo_i\right) \mid Z_{i_{\ell}}^{c}\right] \label{eq:MR-4} \tag{MR-4}\\
& +K_{h}\left(T_{i}-t\right)\left(\hat{\lambda}_{i \ell}-\bar{\lambda}_i\right) \bar r_i\left(Y_{i}-\widebar\cmo_i\right)-\mathbb{E}\left[K_{h}\left(T_{i}-t\right)\left(\hat{\lambda}_{i \ell}-\bar{\lambda}_i\right) \bar r_i\left(Y_{i}-\widebar\cmo_i\right) \mid Z_{i_{\ell}}^{c}\right]  \label{eq:R1-7} \tag{R1-7}\\
&+\mathbb{E}\left[K_{h}\left(T_{i}-t\right)\left(\hat{\lambda}_{i \ell}-\bar{\lambda}_i\right) \bar r_i\left(Y_{i}-\widebar\cmo_i\right) \mid Z_{i_{\ell}}^{c}\right] \label{eq:MR-5} \tag{MR-5}\\
&- K_{h}\left(T_{i}-t\right) \left( \hat r_{i \ell} - \bar r_{i } \right) \left( \hat \lambda_{i \ell} - \bar \lambda_{i } \right) \left( \widehat \cmo_{i \ell} - \widebar \cmo_{i } \right) \label{eq:R2-1} \tag{R2-1} \\ 
&  - K_{h}\left(T_{i}-t\right) \left( \hat r_{i \ell} - \bar r_{i } \right) \left( \widehat \cmo_{i \ell} - \widebar \cmo_{i } \right) \bar \lambda_{i } \label{eq:R2-2} \tag{R2-2} \\
& - K_{h}\left(T_{i}-t\right) \left( \hat \lambda_{i \ell} - \bar \lambda_{i } \right) \left( \widehat \cmo_{i \ell} - \widebar \cmo_{i } \right) \bar r_{i } \label{eq:R2-3} \tag{R2-3} \\
&+ K_{h}\left(T_{i}-t\right) \left( \hat r_{i \ell} - \bar r_{i } \right) \left( \hat \lambda_{i \ell} - \bar \lambda_{i} \right) \left( Y_i - \widebar \cmo_{i} \right) \label{eq:R2-4} \tag{R2-4} \\
& + K_{h}\left(T_{i}-t^\prime\right)  \left( \hat \lambda_{i \ell} - \bar \lambda_{i} \right) \left( \widehat \cmo_{i \ell} - \widebar \cmo_{i} \right) \label{eq:R2-5} \tag{R2-5}  \\
& - K_{h}\left(T_{i}-t^\prime\right)  \left( \hat \lambda_{i \ell} - \bar \lambda_{i \ell} \right) \left( \widehat \ccmo_{i \ell} - \widebar \ccmo_{i} \right) \Biggr] \label{eq:R2-6} \tag{R2-6}  \\
\end{align*}

We will now bound \eqref{eq:R1-1}-\eqref{eq:R1-7}, \eqref{eq:R2-1}-\eqref{eq:R2-6} and \eqref{eq:MR-1}-\eqref{eq:MR-5}. The statements in the following hold for $i \in I_{\ell}, \ell \in\{1, \ldots, L\}$, and for all $t, t^\prime$.

\paragraph{Bounding the residuals \eqref{eq:R1-1}-\eqref{eq:R1-7}} The remainder terms \ref{eq:R1-1} to \ref{eq:R1-7} are stochastic equicontinuous terms that are controlled to be $o_p(1)$ by the mean-square consistency conditions in Assumption \ref{assum:convergence} and cross-fitting. 

For \eqref{eq:R1-1}, define $\Delta_{i \ell}=\widehat{\ccmo}_{i \ell}-\widebar{\ccmo}_i-\mathbb{E}\left[\widehat{\ccmo}_{i \ell}-\widebar{\ccmo}_i \mid Z^{\ell}\right]$. By construction and independence of $Z^{\ell}$ and $Z_i$ for $i \in I_{\ell}, \mathbb{E}\left[\Delta_{i \ell} \mid Z^{\ell}\right]=0$ and $\mathbb{E}\left[\Delta_{i \ell} \Delta_{j \ell} \mid Z^{\ell}\right]=0$ for $i, j \in I_{\ell}$. By Assumptions \ref{assum:overlap} and \ref{assum:convergence},

\begin{equation*}
h^{d_T} \mathbb{E}\left[\Delta_{i \ell}^2 \mid Z^{\ell}\right]=O_p\left(h^{d_T} \int_{\mathcal{X}}\left(\widehat{\ccmo}_{\ell}(t, t^\prime, x)-\widebar{\ccmo}(t, t^\prime, x)\right)^2 f_X(x) d x\right)=o_p(1)    
\end{equation*}
\begin{align*}
\mathbb{E}\left[\left(\sqrt{\frac{h^{d_T}}{n}} \sum_{i \in I_{\ell}} \Delta_{i \ell}\right)^2 \mid Z^{\ell}\right]=\frac{h^{d_T}}{n} \sum_{i \in I_{\ell}} \mathbb{E}\left[\Delta_{i \ell}^2 \mid Z^{\ell}\right]&=O_p\left(h^{d_T} \int_{\mathcal{X}}\left(\widehat{\ccmo}_{\ell}(t, t^\prime, x)-\widebar{\ccmo}(t, t^\prime, x)\right)^2 f_X(x) d x\right)\\
&=o_p(1)
\end{align*}

The conditional Markov's inequality implies that $\sqrt{h^{d_T} / n} \sum_{i \in I_{\ell}} \Delta_{i \ell}=o_p(1)$.

The analogous results hold for \eqref{eq:R1-2}, \eqref{eq:R1-3}, \eqref{eq:R1-4} due to convergence Assumption \ref{assum:convergence} and the boundedness Assumption \ref{assum:regularity}. For \eqref{eq:R1-5}, \eqref{eq:R1-6} and \eqref{eq:R1-7}, a standard algebra using change of variables, a Taylor expansion, the mean value theorem, and Assumption \ref{assum:regularity} yields for example for $\Delta_{i \ell}=K_{h}\left(T_{i}-t^{\prime}\right)\left(\hat{\lambda}_{i \ell}-\bar \lambda_i  \right) \left(\widebar\cmo_i - \widebar \ccmo_i\right) -\mathbb{E}\left[K_{h}\left(T_{i}-t^{\prime}\right)\left(\hat{\lambda}_{i \ell}-\bar \lambda_i \right) \left( \widebar\cmo_i - \widebar \ccmo_i \right) \mid Z_{i_{\ell}}^{c}\right]$ in \eqref{eq:R1-5}:

\begin{align}
& h^{d_T} \mathbb{E}\left[\Delta_{i \ell}^2 \mid Z^{\ell}\right] \nonumber \\
& \leq h^{d_T} \mathbb{E}\left[K_h\left(T_i-t^\prime\right)^2\left(\hat{\lambda}_{i \ell}-\bar \lambda_i \right)^2\left(\widebar{\cmo}_i-\widebar{\ccmo}_i\right)^2 \mid Z^{\ell}\right] \nonumber \\
&= h^{d_T} \int_{\mathcal{X}, \mathcal{M}, \mathcal{T}} K_h(s-t^\prime)^2\left(\hat{\lambda}_{\ell}(t^\prime, x)-\bar \lambda(t^\prime, x)  \right)^2 \mathbb{E}\left[(\widebar{\cmo}(t, x, m)-\widebar{\ccmo}(t, t^\prime, x))^2 \mid T=s, X=x, M=m\right] \nonumber \\
& \times f_{T X M}(s, x, m) d s d x d m \nonumber \\
&= \int_{\mathcal{X}, \mathcal{M}} \int_{\mathcal{R}^{d_T}} \Pi_{j=1}^{d_T} k\left(u_j\right)^2\left(\hat{\lambda}_{\ell}(t^\prime, x)-\bar \lambda(t^\prime, x)  \right)^2 \mathbb{E}\left[(\widebar{\cmo}(t, x, m)-\widebar{\ccmo}(t, t^\prime, x))^2 \mid T=t^\prime+u h, X=x, M=m\right] \nonumber \\
& \times f_{T X M}(t^\prime+u h, x, m) d u d x d m \nonumber \\
&= \int_{\mathcal{X}, \mathcal{M}} \int_{\mathcal{R}^{d_T}}\left(\mathbb{E}\left[(\widebar{\cmo}(t, x, m)-\widebar{\ccmo}(t, t^\prime, x))^2 \mid T=t^\prime, X=x\right] \right. \nonumber  \\ 
& +\left. \left.\sum_{j=1}^{d_T} u_j h \frac{\partial}{\partial t_j} \mathbb{E}\left[(\widebar{\cmo}(t, x, m)-\widebar{\ccmo}(t, t^\prime, x))^2 \mid T=t^\prime, X=x\right]\right|_{t=\bar{t}}\right) \nonumber \\
& \times\left(f_{T X M}(t^\prime, x, m)+\left.\sum_{j=1}^{d_T} u_j h \frac{\partial}{\partial t_j} f_{T X M}(t^\prime, x, m)\right|_{t=\grave{t}}\right) \Pi_{j=1}^{d_T} k\left(u_j\right)^2 d u\left(\hat{\lambda}_{\ell}(t^\prime, x)-\bar \lambda(t^\prime, x)  \right)^2 d x d m \nonumber \\
&= O_p\left(\int_{\mathcal{X}, \mathcal{M}}\left(\hat{\lambda}_{\ell}(t^\prime, x)-\bar \lambda(t^\prime, x)  \right)^2 f_{T X M}(t^\prime, x, m) d x d m \right), \label{eq:kernel_difference_bounding} 
\tag{KD}
\end{align}

where $\bar{t}$ and $\grave{t}$ are between $t$ and $t+u h$. So $h^{d_T} \mathbb{E}\left[\Delta_{i \ell}^2 \mid Z^{\ell}\right]=o_p(1)$ by Assumption \ref{assum:convergence}. The conditional Markov's inequality implies that $\sqrt{h^{d_T} / n} \sum_{i \in I_{\ell}} \Delta_{i \ell}=o_p(1)$.

We have then proven that all residuals  $\eqref{eq:R1-1}=o_p(1)$ up until $\eqref{eq:R1-7}=o_p(1)$.

\paragraph{Bounding the error differences \eqref{eq:R2-1}-\eqref{eq:R2-6}} The second-order remainder terms \eqref{eq:R2-1}-\eqref{eq:R2-6} are controlled by Assumption \ref{assum:convergence}.2. All of these terms involve the product of two or more errors and can be addressed in a similar manner. We present a detailed proof for \eqref{eq:R2-6}, with the same approach applicable to the remaining terms.

For \eqref{eq:R2-6},
$$
\begin{aligned}
\mathbb{E} & {\left[\left|\sqrt{h^{d_T} / n_{\ell}} \sum_{i \in I_{\ell}} K_h\left(T_i-t^\prime \right) \left(\hat{\lambda}_{i \ell}-\bar \lambda_i \right)\left(\widebar{\ccmo}_i-\widehat{\ccmo}_{i \ell}\right)\right| \mid Z^{\ell}\right] } \\
\leq & \sqrt{n_{\ell} h^{d_T}} \int_{\mathcal{X}, \mathcal{T}}\left|\left(\hat{\lambda}_{\ell}(t^\prime, x)-\bar \lambda(t^\prime, x)  \right)\left(\widebar{\ccmo}(t, t^\prime, x)-\widehat{\ccmo}_{\ell}(t,t^\prime, x)\right)\right| K_h(s-t^\prime) f_{T X}(s, x) d s d x \\
\leq & \sqrt{n_{\ell} h^{d_T}}\left(\int_{\mathcal{X}, \mathcal{T}}\left(\hat{\lambda}_{\ell}(t^\prime, x)-\bar \lambda(t^\prime, x)  \right)^2 K_h(s-t^\prime) f_{T X}(s, x) d s d x\right)^{1 / 2} \\
& \times\left(\int_{\mathcal{X}, \mathcal{T}}\left(\widebar{\ccmo}(t, t^\prime, x)-\widehat{\ccmo}_{\ell}(t, t^\prime, x)\right)^2 K_h(s-t^\prime) f_{T X}(s, x) d s d x\right)^{1 / 2} \\
= & \sqrt{n_{\ell} h^{d_T}}\left(\int_{\mathcal{X}}\left(\hat{\lambda}_{\ell}(t^\prime, x)-\bar \lambda(t^\prime, x)  \right)^2 f_{T X}(t, x) d x\right)^{1 / 2}\left(\int_{\mathcal{X}}\left(\widehat{\ccmo}_{\ell}(t, t^\prime, x)-\widebar{\ccmo}(t, t^\prime, x)\right)^2 f_{T X}(t, x) d x\right)^{1 / 2} \\
& +o_p\left(h^2\right) \\
= & o_p(1)
\end{aligned}
$$
by Cauchy-Schwartz inequality, Assumption \ref{assum:convergence}.2, and an application of Lemma \eqref{lemma:kernel_smoothing}. So  $\eqref{eq:R2-6}=o_p(1)$ follows by the conditional Markov's and triangle inequalities.

\paragraph{Bounding \eqref{eq:MR-1}-\eqref{eq:MR-5}} The remainder terms \ref{eq:MR-1}-\ref{eq:MR-5} are the key to multiply robust inference. Note that in the binary treatment case when $K_h\left(T_i-t\right)$ is replaced by $1\left\{T_i=t\right\}$, the sum of those terms is zero because $\psi$ is the Neyman-orthogonal influence function, under correct specification $\widebar{\cmo}=\cmo$, $\widebar{\ccmo}=\ccmo$, $\bar \lambda = \lambda$ and $\bar{r}=r$. In our continuous treatment case, the Neyman orthogonality holds as $h \rightarrow 0$.
\\

For \eqref{eq:MR-1}, we write
$$
\begin{aligned}
& \sqrt{n h^{d_{T}}} \mathbb{E}\left[\left(\widehat{\ccmo}_{i \ell} -\widebar \ccmo_i\right)\left(1- K_{h}\left(T_{i}-t^{\prime}\right)\bar \lambda_i  \right) \mid Z_{i_{\ell}}^{c}\right] \\
& =\sqrt{n h^{d_{T}}} \int_{\mathcal{X}, \mathcal{T}}\left(\widehat{\ccmo}\left(t, t^{\prime}, x\right)-\widebar \ccmo(t, t^{\prime}, x)\right)\left(1- K_{h}\left(s-t^{\prime}\right) \bar \lambda \left(t^{\prime}, x\right) \right) f_{TX}\left(s, x\right) d s d x \\
& =\sqrt{n h^{d_{T}}} \int_{\mathcal{X}}\left(\widehat{\ccmo}\left(t, t^{\prime}, x \right)-\widebar \ccmo \left(t, t^{\prime}, x\right) \right)\left(1-\left\{\int_{\mathcal{T}} K_{h}\left(s-t^{\prime}\right) \txdensity\left(s \mid x\right) d s\right\} \bar \lambda \left(t^{\prime}, x\right)  \right)f_X\left(x\right) dx \\
& = \sqrt{n h^{d_{T}}} \int_{\mathcal{X}}\left(\widehat{\ccmo}\left(t, t^{\prime}, x\right)-\widebar \ccmo \left(t, t^{\prime}, x\right) \right) \left(1-\txdensity\left(t^{\prime} \mid x\right) \bar \lambda \left(t^{\prime}, x\right) \right) f_X\left(x\right) d x \\
& \quad+\sqrt{n h^{d_{T}}} \left(\int_{\mathcal{X}}\left(\widehat{\ccmo}\left(t, t^{\prime}, x\right)-\widebar \ccmo \left(t, t^{\prime}, x\right) \right)  \bar \lambda \left(t^{\prime}, x\right)  f_X\left(x\right) dx \right) \times O\left(h^{2}\right) \\
& =  \sqrt{n h^{d_{T}}} \mathbb{E}\left[ \left(\widehat{\ccmo}_{i \ell}-\widebar \ccmo_i \right) \left(1-{\txdensity}_{i \ell} \bar \lambda_i \right) \mid Z_{i_{\ell}}^{c} \right] + o_{p}(1)
\end{aligned}
$$
where the second last equality follows from Lemma \ref{lemma:kernel_smoothing}, and the last equality follows from the definition of $\bar \lambda_i $, $n h^{d_{T}+4} \rightarrow C_{h}$, Assumption \ref{assum:convergence}, and Assumption \ref{assum:regularity} along with an application of the Cauchy-Schwartz inequality.

Now, we use the multiply robust assumptions to see that the remaining expectation term is 0 if $\lambda = \bar \lambda$, and otherwise if $\lambda \neq \bar \lambda$, then by Assumption \ref{assum:nuisance}, $\sqrt{n h^{d_T}}\left\|\widehat{\ccmo}_{\ell}-\widebar{\ccmo}\right\|_{{X}} = o_{p}(1)$. Therefore $\eqref{eq:MR-1}=o_{p}(1)$.

For \eqref{eq:MR-2}, we write it as

\begin{equation*}
\sqrt{n h^{d_{T}}} \mathbb{E}\left[\left(\widehat{\cmo}_{i \ell}-\widebar\cmo_i\right)\left\{K_{h}\left(T_{i}-t^{\prime}\right) \bar r_i \right\}\right] 
- \sqrt{n h^{d_{T}}} \mathbb{E}\left[\left(\widehat{\cmo}_{i \ell}-\widebar\cmo_i\right)\left\{K_{h}\left(T_{i}-t\right) \bar{\lambda}_i \bar r_i\right\}\right] 
\end{equation*}


The left term with $t^\prime$ can be written as
$$
\begin{aligned}
& \sqrt{n h^{d_{T}}} \mathbb{E}\left[\left(\widehat{\cmo}_{i \ell}-\widebar\cmo_i\right)\left\{K_{h}\left(T_{i}-t^{\prime}\right) \bar \lambda_i \right\}\right] \\
& =\sqrt{n h^{d_{T}}} \int_{\mathcal{M} \times \mathcal{X}}\left(\widehat{\cmo}\left(t, m, x\right)-\widebar\cmo(t, m, x)\right)  \bar \lambda(t^\prime, x) \left\{\int_{\mathcal{T}} K_{h}\left(s-t^{\prime}\right) \txmdensity\left(s \mid  x, m \right) d s \right\} f_{MX}\left(m, x\right) d m d x
\end{aligned}
$$

An application of Lemma \ref{lemma:kernel_smoothing} once again provides
$$
\begin{aligned}
& =\sqrt{n h^{d_{T}}} \int_{\mathcal{M} \times \mathcal{X}}\left(\widehat{\cmo}\left(t, m, x \right)-\widebar\cmo\left(t, m, x \right)\right) \bar \lambda(t^\prime, x) \txmdensity\left(t^{\prime} \mid x, m \right) f_{MX}\left(m, x\right) d m d x \\
& +\sqrt{n h^{d_{T}}} \left(\int_{\mathcal{M} \times \mathcal{X}}\left(\widehat{\cmo}\left(t, m, x \right)-\widebar\cmo\left(t, m, x \right)\right) \bar \lambda(t^\prime, x)  f_{MX}\left(m, x\right) d m d x \right) \times O\left(h^{2}\right)
\end{aligned}
$$

Similarly the right term of \eqref{eq:MR-2} writes as
$$
\begin{aligned}
&\sqrt{n h^{d_{T}}} \int_{\mathcal{M} \times \mathcal{X}}\left(\widehat{\cmo}\left(t, m, x\right)-\widebar\cmo\left(t, m, x\right)\right) \bar{\lambda}\left(t^\prime, x \right) \bar r\left(t, t^\prime, m, x \right) \txmdensity(t \mid x, m) f_{MX}\left(m, x\right) d m d x \\
& +\sqrt{n h^{d_{T}}} \left(\int_{\mathcal{M} \times \mathcal{X}}\left(\widehat{\cmo}\left(t, m, x\right)-\widebar\cmo\left(t, m, x\right)\right) \bar{\lambda} (t^\prime, x) \bar r\left(t, t^\prime, m, x \right)  f_{MX}\left(m, x\right) d m d x \right) \times O\left(h^{2}\right)
\end{aligned}
$$

Therefore, due to boundedness Assumption \ref{assum:regularity} and the convergence Assumption \ref{assum:convergence}, we obtain:

\begin{equation*}
    \eqref{eq:MR-2} = \mathbb{E}\left[ \left(\widehat{\cmo}_{i \ell}-\widebar\cmo_i \right) \bar \lambda_i \left\{{\txmdensity}\left(t^{\prime} \mid x_i, \psi_i\right) - \bar r_i \txmdensity(t \mid x_i, \psi_i) \right\} \mid Z_{i_{\ell}}^{c} \right] + o_{p}(1)
\end{equation*}

We use the multiply robust assumptions to see that the remaining expectation term is 0 when $\txmdensity = \widebar \txmdensity$, and otherwise by Assumption \ref{assum:nuisance}, $\sqrt{n h^{d_T}}\left\|\widehat{\cmo}_{\ell}-\widebar{\cmo}\right\|_{{tXM}} = o_{p}(1)$. Therefore $\eqref{eq:MR-2}=o_{p}(1)$.

For \eqref{eq:MR-3}, we note that
$$
\begin{aligned}
& \sqrt{n h^{d_{T}}} \mathbb{E}\left[K_{h}\left(T_{i}-t^{\prime}\right)\left(\hat{\lambda}_{i \ell}-\bar \lambda_i \right)\left\{\widebar\cmo_i-\widebar \ccmo_i \right\} \mid Z_{i_{\ell}}^{c}\right] \\
& =\sqrt{n h^{d_{T}}} \int K_{h}\left(s-t^{\prime}\right)\left(\hat{\lambda}(t^\prime, x)-\bar \lambda (t^\prime, x) \right)\left\{\widebar\cmo (t, m, x) -\widebar \ccmo (t, t^\prime, x) \right\} f_{TMX}\left(s, m, x \right) ds dm dx \\
& =\sqrt{n h^{d_{T}}} \int\left\{\int K_{h}\left(s-t^{\prime}\right) \txmdensity\left(s \mid x,  m\right) d s \right\}\left(\hat{\lambda}(t^\prime, x)-\bar \lambda(t^\prime, x) \right) \\
& \times\left\{\widebar\cmo(t, m, x)-\widebar \ccmo (t, t^\prime, x)\right\} f_{MX}\left(m, x\right) d m dx \\
& \sqrt{n h^{d_{T}}} \int\left(\hat{\lambda}(t^\prime, x)-\bar \lambda(t^\prime, x) \right)\left\{\widebar\cmo(t, m, x)-\widebar \ccmo(t, t^\prime, x)\right\} f_{TMX}\left(t^{\prime}, m, x\right) d m d x \\
& +\sqrt{n h^{d_{T}}} \left( \int\left(\hat{\lambda}(t^\prime, x)-\bar \lambda(t^\prime, x) \right)\left\{\widebar\cmo(t, m, x)-\widebar \ccmo(t, t^\prime, x)\right\}  f_{MX}\left(m, x\right) d m d x \right)  O\left(h^{2}\right).
\end{aligned}
$$

by application of Lemma \ref{lemma:kernel_smoothing}. The term

\begin{equation*}
\sqrt{n h^{d_{T}}} \left( \int\left(\hat{\lambda}(t^\prime, x)-\bar \lambda(t^\prime, x) \right)\left\{\widebar\cmo(t, m, x)-\widebar \ccmo(t, t^\prime, x)\right\}  f_{MX}\left(m, x\right) d m d x \right)  O\left(h^{2}\right) = o_{p}(1)   
\end{equation*}

by an application of Cauchy-Schwartz and the use of Assumption \ref{assum:convergence}. Next note that:

\begin{equation}
\widebar \cmo-\widebar\ccmo = \widebar \cmo- \cmo - (\widebar \ccmo - \ccmo) + \cmo - \ccmo
\label{eq:decomp_cmo} 
\end{equation}

Therefore with this decomposition and by definition of $\ccmo$,

\begin{align*}
&\sqrt{n h^{d_{T}}} \int\left(\hat{\lambda}(t^\prime, x)-\bar \lambda(t^\prime, x) \right)\left\{\widebar\cmo(t, m, x)-\widebar \ccmo(t, t^\prime, x)\right\} f_{TMX}\left(t^{\prime}, m, x\right) d m d x \\
&=\sqrt{n h^{d_{T}}} \int\left(\hat{\lambda}(t^\prime, x)-\bar \lambda(t^\prime, x) \right)\left\{\widebar\cmo(t, m, x)-\cmo(t, m, x)\right\} f_{TMX}\left(t^{\prime}, m, x\right) d m d x \\
+&\sqrt{n h^{d_{T}}} \int\left(\hat{\lambda}(t^\prime, x)-\bar \lambda(t^\prime, x) \right)\left\{\ccmo(t, t^\prime, x)-\widebar \ccmo(t, t^\prime, x)\right\} f_{TX}\left(t^{\prime}, x\right) d x
\end{align*}

Therefore, 

\begin{align*}
\eqref{eq:MR-3} &= \mathbb{E}\left[\left(\hat{\lambda}_{i \ell}-\bar \lambda_i \right)\left\{\widebar\cmo_i -{\cmo}_{i \ell }\right\} \mid Z_{i_{\ell}}^{c}\right] \\
& + \mathbb{E}\left[\left(\hat{\lambda}_{i \ell}-\bar \lambda_i \right)\left\{{\ccmo}_{i \ell}-\widebar \ccmo_i\right\} \mid Z_{i_{\ell}}^{c}\right] +  o_{p}(1) 
\end{align*}

Using the multiply robust assumptions, the first remaining expectation term is 0 when $\cmo = \widebar \cmo$, and otherwise by Assumption \ref{assum:nuisance}, $\sqrt{n h^{d_T}}\left\|\widehat{\lambda}_{\ell}-\widebar{\lambda}\right\|_{{tX}} = o_{p}(1)$. The second remaining expectation is 0 when $\ccmo = \widebar \ccmo$, and otherwise by Assumption \ref{assum:nuisance}, $\sqrt{n h^{d_T}}\left\|\widehat{\lambda}_{\ell}-\widebar{\lambda}\right\|_{{tX}} = o_{p}(1)$. Therefore $\eqref{eq:MR-3}=o_{p}(1)$.

With similar expansions on the kernel smoothing, we can show that:

\begin{equation*}
\eqref{eq:MR-4} = \mathbb{E}\left[\left(\hat{r}_{i \ell}-\bar r_i\right)  {\txmdensity}_{i \ell} \bar{\lambda}_i\left({\cmo}_{i \ell} -\widebar\cmo_{i}\right) \mid Z_{i_{\ell}}^{c}\right] +  o_{p}(1) 
\end{equation*}

the remaining expectation term is 0 when $\cmo = \widebar \cmo$, and otherwise by Assumption \ref{assum:nuisance}, $\sqrt{n h^{d_T}}\left\|\widehat{\txmdensity}_{\ell}-\widebar{\txmdensity}\right\|_{{tXM}} = o_{p}(1)$. Thus $\eqref{eq:MR-4}=o_{p}(1)$.

and

\begin{equation*}
    \eqref{eq:MR-5} = \mathbb{E}\left[\left(\hat{\lambda}_{i \ell}-\bar \lambda_i\right) {\txmdensity}_{i \ell }\bar{r}_i\left({\cmo}_{i \ell}-\widebar\cmo_{i}\right) \mid Z_{i_{\ell}}^{c}\right] +  o_{p}(1) 
\end{equation*}

the remaining expectation term is 0 when $\cmo = \widebar \cmo$, and otherwise by Assumption \ref{assum:nuisance}, $\sqrt{n h^{d_T}}\left\|\widehat{\lambda}_{\ell}-\widebar{\lambda}\right\|_{{tXM}} = o_{p}(1)$. Thus $\eqref{eq:MR-5}=o_{p}(1)$.

\paragraph{Assembling all terms}

By the triangle inequality, we obtain the asymptotically linear representation

\begin{equation}
\sqrt{n h^{d_T}} n^{-1} \sum_{i=1}^n\left(\hat{\psi}\left(Z_i, \mr_{t, t^\prime}, \widehat{\cmo}_{i \ell}, \widehat{\ccmo}_{i \ell}, \widehat{\txdensity}_{i \ell}, \widehat{\txmdensity}_{i \ell} \right)-\psi\left(Z_i, \mr_{t, t^\prime}, \widebar{\cmo}_{i}, \widebar{\ccmo}_{i}, \widebar{\txdensity}_{i}, \widebar{\txmdensity}_{i}\right)\right)=o_p(1).   
\end{equation}


\subsubsection{Asymptotic normality}

The proof for asymptotic normality follows from an application of the Lyapunov Central Limit theorem to the terms $\sqrt{n h^{d_{T}}} n^{-1} \psi\left(Z_{i}, \widebar \txdensity, \widebar \txmdensity, \widebar \cmo, \widebar \ccmo, \mr_{t, t^{\prime}}\right)$. We will compute the expectation and variance of:

$$\sqrt{n h^{d_{T}}} n^{-1} \psi\left(Z_{i} ; \txdensity, \txmdensity, \cmo, \ccmo, \mr_{t, t^{\prime}}\right)$$

\paragraph{Calculation for $\mathrm{B}_{t,t^\prime}$ and $\mu_{i}$}

Given

\begin{align}
\psi\left(Z_{i} ; \widebar{\txdensity}, \widebar{\txmdensity}, \widebar{\cmo}, \widebar{\ccmo}, \mr_{t, t^{\prime}}\right) & =\frac{K_{h}\left(T_{i}-t\right) \widebar \txmdensity\left(t'\mid M_{i}, X_{i}\right)}{\widebar \txmdensity\left(t \mid M_{i}, X_{i}\right) \widebar \txdensity\left(t' \mid X_{i}\right)}\left\{Y_{i}-\widebar \cmo(X_i, \psi_i, t) \right\} \label{eq:moment_out} \tag{R2-1} \\
& +\frac{K_{h}\left(T_{i}-t^{\prime}\right)}{\widebar \txdensity\left(t^{\prime} \mid X_{i}\right)}\left\{\widebar \cmo(X_i, \psi_i, t) -\widebar \ccmo_i\right\} \label{eq:moment_cmo} \tag{R2-2} \\
&+\widebar \ccmo_i-\mr_{t, t^{\prime}} \label{eq:delta_CCMO} \tag{$\Delta$-3}
\end{align}

 We start by focusing on $\frac{K_{h}(T-t) \widebar \txmdensity\left(t^\prime \mid M, X\right)}{\widebar \txmdensity(t \mid M, X) \widebar \txdensity(t^\prime \mid X)}\{Y-\widebar \cmo(t, M, X)\}+\frac{K_{h}\left(T-t^{\prime}\right)}{\widebar \txdensity\left(t^{\prime} \mid X\right)}\{\widebar \cmo(X, M, t )-\widebar \ccmo\left(t, t^{\prime}, X\right)\}$.
We start by computing the expectation of each the individual terms one at a time.

\paragraph{Expanding the outcome residual \eqref{eq:moment_out}}

$$
\begin{aligned}
& \mathbb{E}\left[K_{h}(T-t) \bar \lambda(t^\prime, X) \bar r(t, t^\prime, M, X )\widebar \txmdensity\left(t^{\prime} \mid M, X\right)\{Y- \widebar \cmo(t, M, X)\}\right] \\
= & \mathbb{E}\left\{ \bar \lambda(t^\prime, X) \bar r(t, t^\prime, M, X )  \mathbb{E}\left[K_{h}(T-t) (\cmo(T, M, X, )-\widebar \cmo(t, M, X)) \mid X, M\right]\right\}
\end{aligned}
$$

The inner product further expands as follows,

$$
\begin{aligned}
& \mathbb{E}\left[K_h(T-t)(\cmo(T, M, X)-\widebar{\cmo}(t, M, X)) \mid X, M\right] \\
& =\int_{\mathcal{T}} K_h(s-t)(\cmo(s, M, X)-\widebar{\cmo}(t, M, X)) \txmdensity(s \mid X, M) d s \\
& =\int_{\mathcal{R}^{d_T}} k(u)(\cmo(t+u h, M, X)-\widebar{\cmo}(t, M, X)) \txmdensity(t+u h \mid X, M) d u \\
& =\int_{\mathcal{R}^{d_T}}\left(\cmo(t, M, X)-\widebar{\cmo}(t,M, X)+\sum_{j=1}^{d_T} h u_j \partial_{t_j} \cmo(t, M, X)+\frac{h^2}{2} u_j^2 \partial_{t_j}^2 \cmo(t, M, X) \right. \\
& \left. +\frac{h^2}{2} \sum_{l=1, l \neq j}^{d_T} u_j u_l \partial_{t_j} \partial_{t_l} \cmo(t, M, X)\right) \left(\txmdensity(t \mid X, M)+\sum_{j=1}^{d_T} h u_j \partial_{t_j} \txmdensity(t \mid X, M)+\frac{h^2}{2} u_j^2 \partial_{t_j}^2 \txmdensity(t \mid X, M) \right. \\
& \left. +\frac{h^2}{2} \sum_{l=1, l \neq j}^{d_T} u_j u_l \partial_{t_j} \partial_{t_l} \txmdensity(t \mid X, M)\right) k\left(u_1\right) \cdots k\left(u_{d_T}\right) d u_1 \cdots d u_{d_T}+O\left(h^3\right) \\
& =(\cmo(t, M, X)-\widebar{\cmo}(t, M, X)) \txmdensity(t \mid X, M)+h^2 \sum_{j=1}^{d_T}\left(\partial_{t_j} \cmo(t, M, X) \partial_{t_j} \txmdensity(t \mid X, M)\right. \\
& \left.+\frac{1}{2} \partial_{t_j}^2 \cmo(t, M, X) \txmdensity(t \mid X, M) +(\cmo(t, M, X)-\widebar{\cmo}(t, M, X)) \frac{1}{2} \partial_{t_j}^2 \cmo(t, M, X)\right) \int_{-\infty}^{\infty} u^2 k(u) d u+O\left(h^3\right)
\end{aligned}
$$

for all $X, M$ in respective range. Inserting this back into the original expectation we get, 

\begin{align}
&\mathbb{E}\left[\bar \lambda(t^\prime, X) \bar r(t, t^\prime, M, X )  \mathbb{E}\left[K_h(T-t)(\cmo(t, M, X)-\widebar{\cmo}(X, M, t)) \mid X\right]\right]  \nonumber \\
= & \mathbb{E}\left[(\cmo(t, M, X)-\widebar{\cmo}(t, M, X))  \bar \lambda(t^\prime, X) \bar r(t, t^\prime, M, X ) \txmdensity(t \mid X, M)  \right]  \nonumber \\
& +h^2 \sum_{j=1}^{d_T} \mathbb{E}\left[\partial_{t_j} \cmo(t, M, X) \partial_{t_j} \txmdensity(t \mid X, M) \bar \lambda(t^\prime, X) \bar r(t, t^\prime, M, X ) \nonumber \right. \\
& +\partial_{t_j}^2 \cmo(t, M, X)  \frac{\bar \lambda(t^\prime, X) \bar r(t, t^\prime, M, X ) \txmdensity(t \mid X, M)}{2}  \nonumber \\
& \left.+(\cmo(t, M, X)-\widebar{\cmo}(t, M, X))\partial_{t_j}^2 \cmo(t, M, X) \frac{\bar \lambda(t^\prime, X) \bar r(t, t^\prime, M, X )}{2 }\right] \int_{-\infty}^{\infty} u^2 k(u) d u+O\left(h^3\right) \nonumber \\
= & \eqref{eq:delta_0}  +h^2 \mathrm{B}_{t,t^\prime}^{\text{1}} +O\left(h^3\right)
\label{eq:double_robust_moment_approximation}
\end{align}

where 

\begin{equation}
\Delta{\text{-1}} = \mathbb{E}\left[(\cmo(t, M, X)-\widebar{\cmo}(t, M, X))  \bar \lambda(t^\prime, X) \bar r(t, t^\prime, M, X ) \txmdensity(t \mid X, M)\right]  
\label{eq:delta_0} \tag{$\Delta$-1}
\end{equation}

\begin{align*}
    \mathrm{B}_{t,t^\prime}^{1} &=\sum_{j=1}^{d_T} \mathbb{E}\left[\partial_{t_j} \cmo(t, M, X) \partial_{t_j} \txmdensity(t \mid X, M) \bar \lambda(t^\prime, X) \bar r(t, t^\prime, M, X ) \nonumber \right. \\
    & +\partial_{t_j}^2 \cmo(t, M, X)  \frac{\bar \lambda(t^\prime, X) \bar r(t, t^\prime, M, X ) \txmdensity(t \mid X, M)}{2}  \nonumber \\
    & \left.+(\cmo(t, M, X)-\widebar{\cmo}(t, M, X))\partial_{t_j}^2 \cmo(t, M, X) \frac{\bar \lambda(t^\prime, X) \bar r(t, t^\prime, M, X )}{2 }\right] \int_{-\infty}^{\infty} u^2 k(u) d u
\end{align*}


\paragraph{Expanding the conditional mean outcome residual \eqref{eq:moment_cmo}}
$$
\begin{aligned}
& \mathbb{E}\left[K_{h}\left(T-t^{\prime}\right) \bar \lambda(t^\prime, X) \left\{\widebar \cmo(t, M, X)- \widebar \ccmo\left(t, t^{\prime}, X\right)\right\}\right] \\
= & \mathbb{E}\left[ \left(\widebar \cmo(t, M, X)- \widebar\ccmo\left(t, t^{\prime}, X\right) \right) \bar \lambda(t^\prime, X) \mathbb{E}\left[K_{h}\left(T-t^{\prime}\right) \mid X, M\right]\right]
\end{aligned}
$$

As in Lemma \ref{lemma:kernel_smoothing}, the inner expectation can be written as

\begin{equation}
\mathbb{E}\left[K_{h}\left(T-t^{\prime}\right) \mid X, M\right] = \txmdensity\left(t^{\prime} \mid X, M\right)+\frac{1}{2} h^{2} \int u^{2} k(u) d u \sum_{j=1}^{d_{T}} \partial_{t_{j}}^{2} \txmdensity\left(t^{\prime} \mid X, M\right)+O\left(h^{3}\right)
\label{eq:inner_expect_kernel_smoothing}
\end{equation}

Plugging this back into the above expectation
$$
\begin{aligned}
& \mathbb{E}\left[\left(\widebar \cmo(t, M, X)- \widebar\ccmo\left(t, t^{\prime}, X\right) \right) \bar \lambda(t^\prime, X)\left(\txmdensity\left(t^{\prime} \mid X, M\right)+\frac{1}{2} h^{2} \int u^{2} k(u) d u \sum_{j=1}^{h_{d_{T}}} \partial_{t_{j}}^{2} \txmdensity\left(t^{\prime} \mid X, M\right)\right)\right] \\
&+O\left(h^{3}\right) \\
= & \eqref{eq:delta_CMO} +h^2 \mathrm{B}_{t,t^\prime}^{2} +O\left(h^{3}\right).
\end{aligned}
$$

where 
\begin{equation}
 \mathrm{B}_{t,t^\prime}^{2} =  h^{2}\left[\int_{\mathcal{R}^{d_T}} u^{2} k(u) d u\right] \mathbb{E}\left[\left\{\widebar\cmo(t, M, X)-\widebar\ccmo\left(t, t^{\prime}, X\right)\right\} \frac{1}{2} \sum_{j=1}^{h_{d_{T}}} \partial_{t_{j}}^{2} \txmdensity\left(t^{\prime} \mid X, M\right)  \bar \lambda(t^\prime, X)\right]  
\end{equation}

and

\begin{equation}
 \Delta{\text{-2}} =  \mathbb{E}\left[\left(\widebar \cmo(t, M, X)-\widebar\ccmo\left(t, t^{\prime}, X\right)\right) \bar \lambda(t^\prime, X)\txmdensity\left(t^{\prime} \mid X, M\right) \right] \label{eq:delta_CMO} \tag{$\Delta$-2}
\end{equation}.

We use the same reasoning as \eqref{eq:decomp_cmo} to write:

\begin{equation*}
    \eqref{eq:delta_CMO} = \mathbb{E}\left[\left\{\widebar \cmo(t, M, X)-\cmo\left(t, M, X\right)\right\} \frac{\txmdensity\left(t^{\prime} \mid X, M\right)}{\widebar \txdensity\left(t^{\prime} \mid X\right)}\right] + \mathbb{E}\left[\left\{\ccmo\left(t, t^{\prime}, X\right) - \widebar \ccmo(t, t^\prime, X)\right\} \frac{\txdensity\left(t^{\prime} \mid X\right)}{\widebar \txdensity\left(t^{\prime} \mid X\right)}\right]
\end{equation*}

\paragraph{Dealing with \eqref{eq:delta_0}, \eqref{eq:delta_CMO} and \eqref{eq:delta_CCMO} with the MR property}

Let 

\begin{equation}
\Delta =\eqref{eq:delta_0} + \eqref{eq:delta_CMO} + \eqref{eq:delta_CCMO} \tag{$\Delta$} \label{eq:delta}
\end{equation}

Expanding the term \eqref{eq:delta}, 

\begin{align*}
    \Delta = \mathbb{E} & \left[(\cmo(t, M, X)-\widebar{\cmo}(t, M, X)) \frac{\widebar \txmdensity(t^\prime \mid X, M) \txmdensity(t \mid X, M)}{\widebar \txmdensity(t \mid X, M) \widebar \txdensity(t^\prime \mid X)} \right. \\
    & \left\{\widebar \cmo(t, M, X)-\cmo\left(t, M, X\right)\right\} \frac{\txmdensity\left(t^{\prime} \mid X, M\right)}{\widebar \txdensity\left(t^{\prime} \mid X\right)} \\
    & + \left\{\ccmo\left(t, t^{\prime}, X\right) - \widebar \ccmo(t, t^\prime, X) \right\}  \frac{\txdensity\left(t^{\prime} \mid X \right)}{\widebar \txdensity\left(t^{\prime} \mid X\right)}  \\
    & + \left. \widebar \ccmo(t, t^\prime, X)-\mr_{t, t^{\prime}} \right]
\end{align*}

Due to the multiply robust property, \eqref{eq:delta} is zero as stipulated in Assumption \ref{assum:convergence}.2 when one of the conditions is met:

\begin{itemize}
    \item $\widebar{\cmo}=\cmo$ and $\widebar{\ccmo}=\ccmo$
    \item $\widebar{\cmo}=\cmo$ and $\widebar{\txdensity}=\txdensity$
    \item $\widebar{\ccmo}=\ccmo$ and $\widebar{\txmdensity}=\txmdensity$
    \item $\widebar{\txdensity}=\txdensity$ and $\widebar{\txmdensity}=\txmdensity$.
\end{itemize}

\paragraph{Conclusion}

Hence, letting
\begin{equation}
\mathrm{B}_{t,t^\prime} = \mathrm{B}_{t,t^\prime}^{1} + \mathrm{B}_{t,t^\prime}^{2}    
\end{equation}

we have $\mathbb{E}\left[\psi\left(Z_{i} ; \widebar \txdensity, \widebar \txmdensity, \widebar \cmo, \widebar \ccmo, \mr_{t, t^{\prime}}\right)\right]=h^{2} \mathrm{B}_{t,t^\prime} +O\left(h^{3}\right)$. 

Next, we prove the properties of variance.

\paragraph{Calculation for $\mathrm{V}_{t,t^\prime}$ and $s_{n}^{2}$}

From the definition of $s_{n}^{2}$, we have

\begin{align*}
s_{n}^{2}&=\sum_{i=1}^{n} \sigma_{i}^{2}=\sum_{i=1}^{n} \operatorname{var}\left(\sqrt{n h^{d_{T}}} n^{-1} \psi\left(Z_{i} ; \widebar \txdensity, \widebar \txmdensity, \widebar \cmo, \widebar \ccmo, \mr_{t, t^{\prime}}\right)\right)   \\  
&=h^{d_{T}} \operatorname{var}\left(\psi\left(Z_{i} ; \widebar \txdensity, \widebar \txmdensity, \widebar \cmo, \widebar \ccmo, \mr_{t, t^{\prime}}\right)\right)
\end{align*}

To compute the variance, we need to compute the squared expectation and the second moment. The squared expectation can obtained using the previous calculation of the bias:

\begin{align*}
&h^{d_{T}} \mathbb{E}\left[ \psi\left(Z_{i} ; \widebar \txdensity, \widebar \txmdensity, \widebar \cmo, \widebar \ccmo, \mr_{t, t^{\prime}}\right) \right]^2 \\
& = h^{d_{T}} \mathbb{E}\left[\frac{K_{h}\left(T-t\right) \widebar \txmdensity\left(t'\mid M, X\right)}{\widebar \txmdensity\left(t \mid M, X\right) \widebar \txdensity\left(t' \mid X\right)}\left\{Y-\widebar \cmo(t, M, X) \right\} \right. \\
&\left. +\frac{K_{h}\left(T-t^{\prime}\right)}{\widebar \txdensity\left(t^{\prime} \mid X\right)}\left\{\widebar \cmo(t, M, X) -\widebar \ccmo\left(t, t^{\prime}, X\right)\right\} +\widebar \ccmo\left(t, t^{\prime}, X\right) \right]^{2} \\
& =O\left(h^{d_{T}+4}\right)
\end{align*}

Let us then consider computations of the second moment of the following quantity:

\begin{align*}
&h^{d_{T}} \mathbb{E}\left\{ \psi\left(Z_{i} ; \widebar \txdensity, \widebar \txmdensity, \widebar \cmo, \widebar \ccmo, \mr_{t, t^{\prime}}\right)^2 \right\} \\
& = h^{d_{T}} \mathbb{E}\left\{\left[\frac{K_{h}\left(T-t\right) \widebar \txmdensity\left(t'\mid M, X\right)}{\widebar \txmdensity\left(t \mid M, X\right) \widebar \txdensity\left(t' \mid X\right)}\left\{Y-\widebar \cmo(t, M, X) \right\}\right.\right. \\
& \left.\left.+\frac{K_{h}\left(T-t^{\prime}\right)}{\widebar \txdensity\left(t^{\prime} \mid X\right)}\left\{\widebar \cmo(t, M, X) -\widebar \ccmo\left(t, t^{\prime}, X\right)\right\} +\widebar \ccmo\left(t, t^{\prime}, X\right) \right]^{2}\right\} \\
\end{align*}

Examining each of the terms above one by one, the first term can be expanded as

\begin{align}
= & h^{d_{T}} \mathbb{E}\left\{\left[\frac{K_{h}\left(T-t\right) \widebar \txmdensity\left(t'\mid M, X\right)}{\widebar \txmdensity\left(t \mid M, X\right) \widebar \txdensity\left(t' \mid X\right)}\left\{Y-\widebar \cmo(t, M, X) \right\}\right]^{2}\right\} \label{eq:variance_outcome} \tag{V1-1} \\
& +h^{d_{T}} \mathbb{E}\left\{\left[\frac{K_{h}\left(T-t^{\prime}\right)}{\widebar \txdensity\left(t^{\prime} \mid X\right)}\left\{\widebar \cmo(t, M, X) -\widebar \ccmo\left(t, t^{\prime}, X\right)\right\}\right]^{2}\right\} \label{eq:variance_cmo} \tag{V1-2}\\
&+h^{d_{T}} \mathbb{E}\left\{\widebar \ccmo\left(t, t^{\prime}, X\right) \right\} \label{eq:variance_ccmo} \tag{V1-3}
\\
& +2 h^{d_{T}} \mathbb{E}\left\{\left[\frac{K_{h}\left(T-t\right) \widebar \txmdensity\left(t'\mid M, X\right)}{\widebar \txmdensity\left(t \mid M, X\right) \widebar \txdensity\left(t' \mid X\right)}\left\{Y-\widebar \cmo(t, M, X) \right\}\right]\right. \nonumber \\
& \left.\times\left[\frac{K_{h}\left(T-t^{\prime}\right)}{\widebar \txdensity\left(t^{\prime} \mid X\right)}\left\{\widebar \cmo(t, M, X) -\widebar \ccmo\left(t, t^{\prime}, X\right)\right\}\right]\right\}  \label{eq:variance_outcome_cmo} \tag{V1-4} \\
& +2 h^{d_{T}} \mathbb{E}\left\{\widebar \ccmo\left(t, t^{\prime}, X\right) 
 \left[\frac{K_{h}\left(T-t\right) \widebar \txmdensity\left(t'\mid M, X\right)}{\widebar \txmdensity\left(t \mid M, X\right) \widebar \txdensity\left(t' \mid X\right)}\left\{Y-\widebar \cmo(t, M, X) \right\}\right]\right\} \label{eq:variance_outcome_ccmo} \tag{V1-5} \\
& +2 h^{d_{T}} \mathbb{E}\left\{\widebar \ccmo\left(t, t^{\prime}, X\right) 
 \left[\frac{K_{h}\left(T-t^{\prime}\right)}{\widebar \txdensity\left(t^{\prime} \mid X\right)}\left\{\widebar \cmo(t, M, X) -\widebar \ccmo\left(t, t^{\prime}, X\right)\right\}\right]\right\} \label{eq:variance_cmo_ccmo} \tag{V1-6}
\end{align}

We analyze each of these terms part by part

\paragraph{Computing \eqref{eq:variance_outcome}}
$$
\begin{aligned}
& h^{d_{T}} \mathbb{E}\left\{\left[\frac{K_{h}\left(T-t\right) \widebar \txmdensity\left(t'\mid M, X\right)}{\widebar \txmdensity\left(t \mid M, X\right) \widebar \txdensity\left(t' \mid X\right)}\left\{Y-\widebar \cmo(t, M, X) \right\}\right]^{2}\right\} \\
= & h^{d_{T}} \mathbb{E}\left\{\mathbb{E}\left\{\left.\left[\frac{K_{h}\left(T-t\right) \widebar \txmdensity\left(t'\mid M, X\right)}{\widebar \txmdensity\left(t \mid M, X\right) \widebar \txdensity\left(t' \mid X\right)}\left\{Y-\widebar \cmo(t, M, X) \right\}\right]^{2} \right\rvert\, X, M\right\}\right\} \\
= & h^{d_{T}} \mathbb{E}\left\{\frac{\widebar \txmdensity\left(t'\mid M, X\right)}{\widebar \txmdensity\left(t \mid M, X\right) \widebar \txdensity\left(t' \mid X\right)} \mathbb{E}\left\{K_{h}(T-t)^{2}(Y-\widebar \cmo(t, M, X))^{2} \mid X, M\right\}\right\} \\
= & h^{d_{T}} \mathbb{E}\left\{\frac{\widebar \txmdensity\left(t'\mid M, X\right)}{\widebar \txmdensity\left(t \mid M, X\right) \widebar \txdensity\left(t' \mid X\right)}  \times \mathbb{E}\left\{K_{h}(T-t)^{2} \mathbb{E}\left\{(Y-\widebar \cmo(t, M, X))^{2} \mid X, M, T=t\right\} \mid X, M\right\}\right\} 
\end{aligned}
$$

With a computation similar to \eqref{eq:kernel_difference_bounding} the inner
expectation can be written as
$$
\begin{aligned}
& h^{d_{T}} \mathbb{E}\left\{K_{h}(T-t)^{2} \mathbb{E}\left\{(Y-\widebar \cmo(t, M, X))^{2} \mid X, M, T\right\} \mid X, M\right\} \\
= & {\left[\int_{\mathcal{R}^{d_T}} \Pi_{j=1}^{d_T} k\left(u_j\right)^2 d u\right] \times \mathbb{E}\left\{(Y-\widebar \cmo(t, M, X))^{2} \mid X, M, T\right\} \txmdensity(t \mid X, M)+O\left(h^{2}\right) }
\end{aligned}
$$
Hence, the part \eqref{eq:variance_outcome} of the variance boils down to

\begin{equation*}
{\left[\int  k\left(u\right)^2 d u\right]^{d_{T}} \mathbb{E}\left\{\frac{\widebar \txmdensity\left(t'\mid M, X\right)^{2} \txmdensity(t\mid X, M)}{\widebar \txmdensity\left(t \mid M, X\right)^{2} \widebar \txdensity\left(t' \mid X\right)^{2}}  \mathbb{E}\left\{(Y-\widebar \cmo(t, M, X))^{2} \mid X, M, T=t\right\} \right\}+O\left(h^{2}\right) }.   
\end{equation*}

\paragraph{Computing \eqref{eq:variance_cmo}}

Similarly, we can obtain:

$$
\begin{aligned}
& h^{d_{T}} \mathbb{E}\left\{\left[\frac{K_{h}\left(T-t^{\prime}\right)}{\widebar \txdensity\left(t^{\prime} \mid X\right)^{2}}\left\{\widebar \cmo(t, M, X) -\widebar \ccmo\left(t, t^{\prime}, X\right)\right\}\right]^{2}\right\} \\
= & h^{d_{T}} \mathbb{E}\left\{\frac{1}{\widebar \txdensity\left(t^{\prime} \mid X\right)^{2}} \mathbb{E}\left[K_{h}\left(T-t^{\prime}\right)^{2}\left(\widebar \cmo(t, M, X) -\widebar \ccmo\left(t, t^{\prime}, X\right)\right)^{2} \mid X\right]\right\} \\
= & h^{d_{T}} \mathbb{E}\left\{\frac{1}{\widebar \txdensity\left(t^{\prime} \mid X\right)^{2}} \mathbb{E}\left[K_{h}\left(T-t^{\prime}\right)^{2} \mathbb{E}\left[\left(\widebar \cmo(t, M, X) -\widebar \ccmo\left(t, t^{\prime}, X\right)\right)^{2}\mid X, M, T=t^\prime\right] \mid X\right]\right\}
\end{aligned}
$$

Again, similarly to \eqref{eq:kernel_difference_bounding} with a Taylor expansion on the density, we obtain for the inner expectation 
$$
\begin{aligned}
& h^{d_{T}} \mathbb{E}\left[K_{h}\left(T-t^{\prime}\right)^{2} \mathbb{E}\left[\left(\widebar \cmo(t, M, X) -\widebar \ccmo\left(t, t^{\prime}, X\right)\right)^{2}\mid X, M, T=t^\prime\right] \mid X\right] \\
= & {\left[\int k(u)^{2} d u\right]^{d_{T}} \txdensity\left(t^{\prime} \mid X\right) \mathbb{E}\left[\left(\widebar \cmo(t, M, X) -\widebar \ccmo\left(t, t^{\prime}, X\right)\right)^{2}\mid X, M, T=t^\prime\right]+O\left(h^{2}\right) }
\end{aligned}
$$

Hence, the part \eqref{eq:variance_cmo} of the variance accounts for
$$
\begin{aligned}
& {\left[\int k(u)^{2} d u\right]^{d_{T}} \times \mathbb{E}\left\{\frac{ \txdensity\left(t^{\prime} \mid X\right)}{\widebar \txdensity\left(t^{\prime} \mid X\right)^{2}} \mathbb{E}\left[\left(\widebar \cmo(t, M, X) -\widebar \ccmo\left(t, t^{\prime}, X\right)\right)^{2}\mid X, M, T=t^\prime\right]\right\}+O\left(h^{2}\right) }
\end{aligned}
$$

\paragraph{Computing \eqref{eq:variance_ccmo}}
$$
h^{d_{T}} \mathbb{E}\left[\widebar \ccmo^2\left(t, t^{\prime}, X\right) \right]=O\left(h^{d_{T}}\right)
$$

This holds because we assume $\widebar \ccmo$ is bounded.

\paragraph{Computing \eqref{eq:variance_outcome_cmo}}
$$
\begin{aligned}
& h^{d_{T}} \mathbb{E}\left\{\left[\frac{K_{h}\left(T-t\right) \widebar \txmdensity\left(t'\mid M, X\right)}{\widebar \txmdensity\left(t \mid M, X\right) \widebar \txdensity\left(t' \mid X\right)}\left\{Y-\widebar \cmo(t, M, X) \right\}\right]\right. \nonumber \\
& \left.\times\left[\frac{K_{h}\left(T-t^{\prime}\right)}{\widebar \txdensity\left(t^{\prime} \mid X\right)}\left\{\widebar \cmo(t, M, X) -\widebar \ccmo\left(t, t^{\prime}, X\right)\right\}\right]\right\} \\
& =h^{d_{T}} \mathbb{E}\left\{\frac{K_{h}\left(T-t\right) K_{h}\left(T-t^\prime\right)}{\widebar \txdensity\left(t' \mid X\right)^2} \frac{\widebar \txmdensity\left(t'\mid M, X\right)}{\widebar \txmdensity\left(t\mid M, X\right)}[Y-\widebar \cmo(t, M, X)]\left[\widebar \cmo(t, M, X) -\widebar \ccmo\left(t, t^{\prime}, X\right)\right]\right\} \\
& =h^{d_{T}} \mathbb{E}\left\{\frac{1}{\widebar \txdensity\left(t' \mid X\right)^2} \frac{\widebar \txmdensity\left(t'\mid M, X\right)}{\widebar \txmdensity\left(t\mid M, X\right)}\left[\widebar \cmo(t, M, X) -\widebar \ccmo\left(t, t^{\prime}, X\right)\right]\right. \\
& \left.\times \mathbb{E}\left\{K_{h}\left(T-t\right)K_{h}\left(T-t^\prime\right)[Y-\widebar \cmo(t, M, X)] \mid X, M\right\}\right\} \\
& =h^{d_{T}} \mathbb{E}\left\{\frac{1}{\widebar \txdensity\left(t' \mid X\right)^2} \frac{\widebar \txmdensity\left(t'\mid M, X\right)}{\widebar \txmdensity\left(t\mid M, X\right)}\left[\widebar \cmo(t, M, X) -\widebar \ccmo\left(t, t^{\prime}, X\right)\right]\right. \\
& \left.\times \mathbb{E}\left\{K_{h}\left(T-t\right)K_{h}\left(T-t^\prime\right)[Y-\widebar \cmo(t, M, X)] \mid X, M\right\}\right\}
\end{aligned}
$$

The inner expectation
$$
\begin{aligned}
& h^{d_{T}} \mathbb{E}\left\{K_{h}(T-t) K_{h}\left(T-t^{\prime}\right)[\cmo(T, M, X)- \widebar \cmo(t, M, X)] \mid X, M\right\} \\
= & h^{d_{T}} \int\left[\prod_{j=1}^{d_{T}} \frac{1}{h^{2}} k\left(\frac{s-t}{h}\right) k\left(\frac{s-t^{\prime}}{h}\right)\right][\cmo(s, M, X)-\widebar \cmo(t, M, X)] \txmdensity(s \mid X, M) d s \\
= & \int k\left(u_{1}\right) \cdots k\left(u_{d_{T}}\right) k\left(u_{1}+\frac{t-t^{\prime}}{h}\right) \cdots k\left(u_{d_{T}}+\frac{t-t^{\prime}}{h}\right)[\cmo(u h+t, M, X)- \widebar \cmo(t, M, X)] \txmdensity(u h+t \mid X, M) d u \\
= & \int k\left(u_{1}\right) \cdots k\left(u_{d_{T}}\right) k\left(u_{1}+\frac{t-t^{\prime}}{h}\right) \cdots k\left(u_{d_{T}}+\frac{t-t^{\prime}}{h}\right) \\
& \times\left[[\cmo(t, M, X)- \widebar \cmo(t, M, X)] + \sum_{j=1}^{d_{T}} u_{j} h \partial_{t_{j}} \cmo(t, M, X)+\frac{u_{j}^{2} h^{2}}{2} \partial_{t_{j}}^{2} \cmo(t, M, X)+\frac{u_{j}^{3} h^{3}}{6} \partial_{t_{j}}^{3} \cmo(t, M, X)\right] \\
& \times\left[\txmdensity(t \mid X, M)+\sum_{j=1}^{d_{T}} u_{j} h \partial_{t_{j}} \txmdensity(t \mid X, M)+\frac{u_{j}^{2} h^{2}}{2} \partial_{t_{j}}^{2} \txmdensity(t \mid X, M)\right] d u_{1} \cdots d u_{d_{T}} \\
& =[\cmo(t, M, X)- \widebar \cmo(t, M, X)]\txmdensity(t \mid X, M) \int k\left(u_{1}\right) \cdots k\left(u_{d_{T}}\right) k\left(u_{1}+\frac{t-t^{\prime}}{h}\right) \cdots k\left(u_{d_{T}}+\frac{t-t^{\prime}}{h}\right) + O(h) \\
&= O(h)
\end{aligned}
$$

\paragraph{Computing \eqref{eq:variance_outcome_ccmo}}
$$
\begin{aligned}
& 2 h^{d_{T}} \mathbb{E}\left\{\widebar \ccmo\left(t, t^{\prime}, X\right) 
 \left[\frac{K_{h}\left(T-t\right) \widebar \txmdensity\left(t'\mid M, X\right)}{\widebar \txmdensity\left(t \mid M, X\right) \widebar \txdensity\left(t' \mid X\right)}\left\{Y-\widebar \cmo(t, M, X) \right\}\right]\right\} \\
& =2 h^{d_{T}} \mathbb{E}\left\{
\frac{\widebar \ccmo\left(t, t^{\prime}, X\right)\widebar \txmdensity\left(t'\mid M, X\right)}{\widebar \txmdensity\left(t \mid M, X\right) \widebar \txdensity\left(t' \mid X\right)}
\mathbb{E}[K_{h}\left(T-t\right) \{Y-\widebar \cmo(t, M, X) \mid X, M\}]\right\}
\end{aligned}
$$

Applying the same expansion as in \eqref{eq:moment_out}, we can write the inner expectation as

\begin{align*}
& =(\cmo(t, M, X)-\widebar{\cmo}(t, M, X)) \txmdensity(t \mid X, M)+O\left(h^{2}\right)
\end{align*}

Inserting this back into the full expectation, combined with the boundedness of $\cmo, \ccmo, \txdensity \txmdensity$ and their limits, we get
$$
\eqref{eq:variance_outcome_ccmo}= O\left(h^{d_{T}}\right)
$$

\paragraph{Computing \eqref{eq:variance_cmo_ccmo}}
$$
\begin{aligned}
& 2 h^{d_{T}} \mathbb{E}\left\{\widebar \ccmo\left(t, t^{\prime}, X\right) 
 \left[\frac{K_{h}\left(T-t^{\prime}\right)}{\widebar \txdensity\left(t^{\prime} \mid X\right)}\left\{\widebar \cmo(t, M, X) -\widebar \ccmo\left(t, t^{\prime}, X\right)\right\}\right]\right\} \\
= & 2 h^{d_{T}} \mathbb{E}\left\{\frac{\widebar \ccmo\left(t, t^{\prime}, X\right) \left( \widebar \cmo(t, M, X) -\widebar \ccmo\left(t, t^{\prime}, X\right) \right)}{\widebar \txdensity\left(t^{\prime} \mid X\right)} \mathbb{E}\left[K_{h}\left(T-t^{\prime}\right) \mid X, M\right]\right\}
\end{aligned}
$$

Using \eqref{eq:inner_expect_kernel_smoothing} on $\mathbb{E}\left[K_{h}\left(T-t^{\prime}\right) \mid X, M\right]$ and plugging this back into the full expectation, we get
$$
\eqref{eq:variance_cmo_ccmo}=O\left(h^{d_{T}}\right)
$$

\paragraph{Computing the total variance}

Finally, with all terms of the variance together, we have shown that 
$$
\begin{aligned}
h^{d_{T}} & \times \operatorname{var}\left(\psi\left(Z_{i} ; \widebar \txdensity, \widebar \txmdensity, \widebar \cmo, \widebar \ccmo, \mr_{t, t^{\prime}}\right)\right)  =\mathrm{V}_{t,t^\prime}+O(h)
\end{aligned}
$$
where the term converges to $\mathrm{V}_{t,t^\prime}$ as $h \rightarrow 0$ and

\begin{align*}
\mathrm{V}_{t,t^\prime}= & {R_k^{d_{T}} } \\
& \times \mathbb{E}\left\{\frac{\widebar \txmdensity\left(t'\mid M, X\right)^2 \txmdensity(t\mid X, M)}{\widebar \txmdensity\left(t \mid M, X\right)^2 \widebar \txdensity\left(t' \mid X\right)^2}  \mathbb{E}\left\{(Y-\widebar \cmo(t, M, X))^{2} \mid X, M, T=t\right\} \right. \\
& \left. +\frac{ \txdensity\left(t^{\prime} \mid X\right)}{\widebar \txdensity\left(t^{\prime} \mid X\right)^{2}} \mathbb{E}\left[\left(\widebar \cmo(t, M, X) -\widebar \ccmo\left(t, t^{\prime}, X\right)\right)^{2}\mid X, M, T=t^\prime\right]\right\}
\end{align*}

Asymptotic normality follows from the Lyapunov central limit theorem with the third absolute moment. Specifically, having derived the bias and variance terms, we will now prove the Lyapunov condition for $\delta=1$, i.e.
$$
\lim _{n \rightarrow \infty} \frac{1}{s_{n}^{3}} \sum_{i=1}^{n} \mathbb{E}\left[\left|\sqrt{n h^{d_{T}}} n^{-1} \psi\left(Z_{i} ; \widebar \txdensity, \widebar \txmdensity, \widebar \cmo, \widebar \ccmo, \mr_{t, t^{\prime}}\right)-\mu_{i}\right|^{3}\right]=0
$$

Where $\mu_{i}$ equals $\mathbb{E}\left[\sqrt{n h^{d_{T}}} n^{-1} \psi\left(Z_{i} ; \txdensity, \txmdensity, \cmo, \ccmo, \mr_{t, t^{\prime}}\right)\right]$ and $s_{n}^{2}=\sum_{i=1}^{n} \sigma_{i}^{2}$ where $\sigma_{i}^{2}$ is the variance of of $\sqrt{n h^{d_{T}}} n^{-1} \psi\left(Z_{i} ; \txdensity, \txmdensity, \cmo, \ccmo, \mr_{t, t^{\prime}}\right)$.

Noticing that we can expand the previous term 

$$
\begin{aligned}
\left|\sqrt{n h^{d_{T}}} n^{-1} \psi\left(Z_{i} ; \widebar \txdensity, \widebar \txmdensity, \widebar \cmo, \widebar \ccmo, \mr_{t, t^{\prime}}\right)-\mu_{i}\right|^{3} & \leq\left(h^{d_{T}} n^{-1}\right)^{3 / 2}\left|\psi\left(Z_{i} ; \widebar \txdensity, \widebar \txmdensity, \widebar \cmo, \widebar \ccmo, \mr_{t, t^{\prime}}\right)\right|^{3}+\left|\mu_{i}\right|^{3} \\
& +3\left(h^{d_{T}} n^{-1}\right)\left|\psi\left(Z_{i} ; \widebar \txdensity, \widebar \txmdensity, \widebar \cmo, \widebar \ccmo, \mr_{t, t^{\prime}}\right)\right|^{2}\left|\mu_{i}\right| \\
& +3\left(h^{d_{T}} n^{-1}\right)^{1 / 2}\left|\psi\left(Z_{i} ; \widebar \txdensity, \widebar \txmdensity, \widebar \cmo, \widebar \ccmo, \mr_{t, t^{\prime}}\right)\right|\left|\mu_{i}\right|^{2}
\end{aligned}
$$

Then it suffices to note that 
\begin{align*}
\mathbb{E} & \left[\left|\sqrt{n h^{d_T}} n^{-1} \psi\left(Z_{i} ; \widebar \txdensity, \widebar \txmdensity, \widebar \cmo, \widebar \ccmo, \mr_{t, t^{\prime}}\right)\right|^3\right] \\
=&O\left(\left(n^{-1} h^{d_T}\right)^{3 / 2} \mathbb{E}\left[K_h(T-t)^3 \vert Y-\widebar{\cmo}(t, M, X)\vert^3 \bar{r}(t, t', M, X)^3  \bar \lambda(t^\prime, X)^{3}\right]\right) \\
+&O\left(\left(n^{-1} h^{d_T}\right)^{3 / 2} \mathbb{E}\left[K_h(T-t^\prime)^3 \vert \widebar{\cmo}(t, M, X)-\widebar{\ccmo}(t, t^\prime, X)\vert^3  \bar \lambda(t^\prime, X)^{3}\right]\right) \\
=&O\left(\left(n^3 h^{d_T}\right)^{-1 / 2}\right)    
\end{align*}

by the same arguments as \eqref{eq:kernel_difference_bounding} under the condition that $\mathbb{E}\left[|Y-\widebar{\cmo}(t, M, X)|^3 \mid T=t, M, X\right]$, $\mathbb{E}\left[|\widebar{\cmo}(t, M, X)-\widebar{\ccmo}(t, t^\prime, X)|^3 \mid T=t', M, X\right]$ and their first derivative w.r.t. $t$ are bounded uniformly in $x \in \mathcal{X}$. 

Hence
$$
\sum_{i=1}^{n} \mathbb{E}\left[\left|\sqrt{n h^{d_{T}}} n^{-1} \psi\left(Z_{i} ; \widebar \txdensity, \widebar \txmdensity, \widebar \cmo, \widebar \ccmo, \mr_{t, t^{\prime}}\right)\right|^{3}\right]=O\left(\left(n h^{d_{T}}\right)^{-1 / 2}\right)=o(1)
$$

Moreover, recall that
\begin{equation}
s_n^2 \equiv \sum_{i=1}^n \operatorname{var}\left(\sqrt{n h^{d_T}} n^{-1} \psi\left(Z_{i} ; \widebar \txdensity, \widebar \txmdensity, \widebar \cmo, \widebar \ccmo, \mr_{t, t^{\prime}}\right)\right)=h^{d_T} \operatorname{var}(\psi)=\mathrm{V}_{t, t^\prime}+o(1).
\end{equation}

Thus the Lyapunov condition holds:  
\begin{equation}
\sum_{i=1}^n \mathbb{E}\left[\left|\sqrt{n h^{d_T}} n^{-1} \psi\left(Z_{i} ; \widebar \txdensity, \widebar \txmdensity, \widebar \cmo, \widebar \ccmo, \mr_{t, t^{\prime}}\right)\right|^3\right] / s_n^3=O\left(\left(n h^{d_T}\right)^{-1 / 2}\right)=o(1)    
\end{equation}

and hence,
$$
\frac{1}{s_{n}} \sum_{i=1}^{n}\left(\sqrt{\frac{h^{d_{T}}}{n}} \psi\left(Z_{i} ; \widebar \txdensity, \widebar \txmdensity, \widebar \cmo, \widebar \ccmo, \mr_{t, t^{\prime}}\right)-\mu_{i}\right) \xrightarrow{d} \mathcal{N}(0,1)
$$

To finish, an application of Slutsky's theorem provides the desired result that
$$
\sqrt{n h^{d_T}}\left(\hat{\mr}_{t,t'}-\mr_{t,t^\prime}-h^2 \mathrm{~B}_{t,t'}\right) \xrightarrow{d} \mathcal{N}\left(0, \mathrm{~V}_{t,t'}\right) 
$$

\end{proof}

\subsection{Optimal bandwidth}

In this part we will prove the corollaries of the Theorem \ref{thm:asymptotic_normality}. We will start by stating Assumption \ref{assum:split_consistency} with the general dimension $d_T$.

\begin{assum}
For each $\ell=1, \ldots, L$ and for any $t \in \mathcal{T}$,
\begin{enumerate}
    \item $\left\|{\delta_{\cmo}^\ell}{\delta_{T\mid X}^\ell}\right\|_{{\mathrm{t} X}}=o_p(1)$, $\left\|{\delta_{\cmo}^\ell}{\delta_{\ccmo}^\ell}\right\|_{{\mathrm{t} X}}=o_p(1)$, $\left\|{\delta_{\ccmo}^\ell}{\delta_{T\mid X, M}^\ell}\right\|_{{\mathrm{t} MX}}=o_p(1)$, $\left\|{\delta_{T\mid X}^\ell}^2{\delta_{T\mid X, M}^\ell}\right\|_{{\mathrm{t} MX}}=o_p(1)$,
    \item $\left\|{\delta_{\cmo}^\ell}^2 {\delta_{T\mid X}^\ell}^2\right\|_{{\mathrm{t} X}}=O_p(1)$, $\left\|{\delta_{\cmo}^\ell}^2{\delta_{\ccmo}^\ell}^2\right\|_{{\mathrm{t} X}}=O_p(1)$, $\left\|{\delta_{\ccmo}^\ell}^2{\delta_{T\mid X, M}^\ell}^2\right\|_{{\mathrm{t} MX}}=O_p(1)$, $\left\|{\delta_{T\mid X}^\ell}^2{\delta_{T\mid X, M}^\ell}^2\right\|_{{\mathrm{t} MX}}=O_p(1)$
    \item $\left\|{\delta_{\cmo}^\ell}^2\right\|_{{\mathrm{t} X}}=O_p(1)$, $\left\|{\delta_{\ccmo}^\ell}^2\right\|_{{\mathrm{t} X}}=O_p(1)$, $\left\|{\delta_{T\mid X, M}^\ell}^2\right\|_{{\mathrm{t} MX}}=O_p(1)$, $\left\|{\delta_{T\mid X}^\ell}^2\right\|_{{\mathrm{t} MX}}=O_p(1)$
    \item $\mathrm{E}\left[(Y- \widebar{\cmo}(t, M, X))^4 \mid T=t, M, X\right]$, $\mathrm{E}\left[(\widebar{\cmo}(t, M, X)- \widebar{\ccmo}(t, t^\prime, X))^4 \mid T=t^\prime, M, X\right]$ and their derivatives with respect to $t$ are bounded uniformly over $\left(t^{\prime}, x\right) \in$ $\mathcal{T} \times \mathcal{X}$ and $\int_{-\infty}^{\infty} k(u)^4 d u<\infty$
    \item The bandwidth $b_n \rightarrow 0$ and $n b_n^{d_T+4} \rightarrow \infty . \int_{-\infty}^{\infty} k(u) k(u / \epsilon) d u<\infty$ for $\epsilon \in(0,1)$.
\end{enumerate}
\end{assum}

We now restate Corollary \ref{thm:amse_bandwidth} with the general treatment dimension as well.

\begin{customcor}{\ref{thm:amse_bandwidth}}
Let the conditions in Theorem \ref{thm:asymptotic_normality} hold. For $t \in \mathcal{T}$, if $\mathrm{B}_{t, t^\prime}$ is non-zero, then the bandwidth that minimizes the asymptotic mean squared error is $h_t^*=\left(d_T \mathrm{~V}_t /\left(4 \mathrm{~B}_t^2\right)\right)^{1 /\left(d_T+4\right)} n^{-1 /\left(d_T+4\right)}$. Further let Assumption \ref{assum:split_consistency} hold. Then $\hat{\mathrm{V}}_t-\mathrm{V}_t=o_p(1)$, $\hat{\mathrm{B}}_t-\mathrm{B}_t=o_p(1)$, and $\hat{h}_t / h_t^*-1=o_p(1)$.  
\end{customcor}

\begin{proof}
By Theorem \ref{thm:asymptotic_normality}, the asymptotic MSE is $h^4 \mathrm{~B}_{t, t^\prime}^2+\mathrm{V}_{t, t^\prime} /\left(n h^{d_T}\right)$. Solving the first-order condition yields the optimal bandwidth $h_t^*$.

Note that once the consistency of $\hat{\mathrm{V}}_{t, t^\prime}$ and $\hat{\mathrm{B}}_{t, t^\prime}$ is proven, the continuous mapping theorem implies $\hat{h}_t / h_t^*-1-o_p(1)$. Therefore, we show below the consistency of $\hat{\mathrm{V}}_{t, t^\prime}$ and $\hat{\mathrm{B}}_{t, t^\prime}$.

\paragraph{Consistency of $\hat{\mathrm{V}}_{t, t^\prime}$ :}

Let $\hat{\mathrm{V}}_{t, t^\prime}=L^{-1} \sum_{\ell=1}^L \hat{\mathrm{V}}_{t, t^\prime, \ell}$, where $\hat{\mathrm{V}}_{t, t^\prime, \ell} \equiv h^{d_T} n_{\ell}^{-1} \sum_{i \in I_{\ell}} \hat{\psi}_{i \ell}^2$. It suffices to show that $\hat{\mathrm{V}}_{t, t^\prime, \ell}$ is consistent for $\mathrm{V}_{t, t^\prime}$ as $n_{\ell} \rightarrow \infty$, for $\ell=1, \ldots, L$. Toward that end, we show that:

\begin{enumerate}
\item $h^{d_T} n_{\ell}^{-1} \sum_{i \in I_{\ell}} \psi_i^2-\mathrm{V}_{t, t^\prime}=o_p(1)$, where 
$\psi_i \equiv K_h(T_i-t)  \bar r_i \bar \lambda_i (Y_i-\widebar \cmo_i) + K_h(T_i-t')\bar \lambda_i (\widebar\cmo_i -\widebar\ccmo_i)  +  \widebar \ccmo_i -\mr_{t, t^\prime} $,
\item $h^{d_T} n_{\ell}^{-1} \sum_{i \in I_{\ell}} \mathbb{E}\left[\hat{\psi}_{i \ell}^2-\psi_i^2 \mid Z^{\ell}\right]=o_p(1)$, where $\hat{\psi}_{i \ell} \equiv K_h(T_i-t) \hat{r}_{i \ell} \hat{\lambda}_{i \ell} \cdot[Y_i-\widehat{\cmo}_{i \ell}] + K_h(T_i-t')\hat{\lambda}_{i \ell} \cdot[\widehat{\cmo}_{i \ell} -\widehat{\ccmo}_{i \ell}]  + \widehat{\ccmo}_{i \ell} -\hat{\psi}_i $
\item $h^{d_T} n_{\ell}^{-1} \sum_{i \in I_{\ell}} \Delta_{i \ell}=o_p(1)$, where $\Delta_{i \ell} \equiv$ $\hat{\psi}_{i \ell}^2-\psi_i^2-\mathbb{E}\left[\hat{\psi}_{i \ell}^2-\psi_i^2 \mid Z^{\ell}\right]$
\end{enumerate}

\paragraph{Showing $h^{d_T} n_{\ell}^{-1} \sum_{i \in I_{\ell}} \psi_i^2-\mathrm{V}_{t, t^\prime}=o_p(1)$.} Let $\ell \in \{1, \dots, L \}$, and let $i \in I_{\ell}$. As computed in the proof of Theorem \ref{thm:asymptotic_normality}, $h^{d_T} \mathbb{E}\left[\psi_i^2\right]=\mathrm{V}_{t, t^\prime}+o(1)$. We will now compute $h^{d_T} \mathbb{E}\left[\psi_i^4\right]$. Recall the decomposition of $\psi_i$ into $m_i, \phi_i, \varphi_i$ in Eq. \eqref{eq:ccmo_moment}, \eqref{eq:cmo_moment} and \eqref{eq:outcome_moment}, 
$$
\begin{aligned}
& \phi_i = K_h(T_i-t)  \bar r_i \bar \lambda_i (Y_i-\widebar \cmo_i) \\
& \varphi_i=K_h(T_i-t')\bar \lambda_i (\widebar\cmo_i -\widebar\ccmo_i) \\
& m_i = \widebar \ccmo_i -\mr_{t, t^\prime}
\end{aligned}
$$

Therefore, noticing that $\psi_i = \phi_i + \varphi_i + m_i$, we will consider different quantities that are in the expansion $\mathbb{E}\left[(\phi_i + \varphi_i + m_i)^4\right]$. By similar arguments as in \eqref{eq:kernel_difference_bounding} and Assumption \ref{assum:split_consistency}.3, 

\begin{align*}
\mathbb{E}\left[\phi_i^4\right]&=h^{-3 d_T} \mathbb{E}\left[\mathbb{E}\left[(Y-\widebar{\cmo}(t, X, M))^4 \mid T=t, X, M\right] \txmdensity(t \mid M, X) \bar{\lambda}(t^\prime, X)^4 \bar r(M, X)^4 \right]\left(\int_{-\infty}^{\infty} k(u)^4 d u\right)^{d_T} \\
&+ o\left(h^{-3 d_T}\right) \\
&=O\left(h^{-3 d_T}\right)    
\end{align*}

and
\begin{align*}
\mathbb{E}\left[\varphi_i^4\right]&=h^{-3 d_T} \mathbb{E}\left[\mathbb{E}\left[(\widebar{\cmo}(t, X, M)-\widebar{\ccmo}(t, t^\prime, X))^4 \mid T=t', X, M\right] \txmdensity(t^\prime \mid M, X) \bar{\lambda}(t^\prime, X)^4 \right]\left(\int_{-\infty}^{\infty} k(u)^4 d u\right)^{d_T}\\
&+ o\left(h^{-3 d_T}\right) \\
&=O\left(h^{-3 d_T}\right)
\end{align*}

Moreover, let us consider integers $p, q$ such that $p+q=4$. $\mathbb{E}\left(\phi_i^{p} \varphi_i^{q}\right)$ can be written as follows
\begin{align*}
\mathbb{E} \Bigg[ \mathbb{E} & \Bigg(K_{h}(T-t)^{p} K_{h}(T-t^{\prime})^{q} \big[Y - \mu_Y(t, M, X) \big]^p \lambda(t^\prime, X)^p r(t, t^\prime, M, X)^p  \mid X, M \Bigg) \\
& \times \big(\mu_Y(t, M, X) - \omega_Y(t, t^{\prime}, X) \big)^q \lambda(t^{\prime}, X)^q \Bigg]
\end{align*}

Take the inner expectation:
\begin{align*}
&=\mathbb{E}\left\{K_{h}(T-t)^{p} K_{h}\left(T-t^{\prime}\right)^{q}[[Y-\widebar \cmo(t, M, X)]\lambda(t^\prime, X) r(t, t^\prime, M, X)]^{p} \mid X, M\right\} \\
= &\int\left[\prod_{j=1}^{d_{T}} \frac{1}{h^{p+q}} k\left(\frac{s-t}{h}\right)^{p} k\left(\frac{s-t^{\prime}}{h}\right)^{q}\right][\cmo(s, M, X)-\widebar \cmo(t, M, X)] \txmdensity(s \mid X, M) d s \\
&= \frac{1}{h^{\left(p+q-1\right) d_T}} \int\left[\prod_{j=1}^{d_T} k\left(u_j\right)^{p} k\left(u_j+\frac{t_j-t_j^{\prime}}{h}\right)^{q}\right] [\cmo(u h+t, M, X)- \widebar \cmo(t, M, X)] \txmdensity(u h+t \mid X, M) d u \\
&= \frac{1}{h^{\left(p+q-1\right) d_T}} \int\left[\prod_{j=1}^{d_T} k\left(u_j\right)^{p} k\left(u_j+\frac{t_j-t_j^{\prime}}{h}\right)^{q}\right] \\
&\times\left[[\cmo(t, M, X)- \widebar \cmo(t, M, X)] + \sum_{j=1}^{d_{T}} u_{j} h \partial_{t_{j}} \cmo(t, M, X)+\frac{u_{j}^{2} h^{2}}{2} \partial_{t_{j}}^{2} \cmo(\bar t, M, X)\right] \\
& \times\left[\txmdensity(t \mid X, M)+\sum_{j=1}^{d_{T}} u_{j} h \partial_{t_{j}} \txmdensity(t \mid X, M)+\frac{u_{j}^{2} h^{2}}{2} \partial_{t_{j}}^{2} \txmdensity(\tilde t \mid X, M)\right] d u_{1} \cdots d u_{d_{T}} \\
& =\frac{1}{h^{\left(p+q-1\right) d_T}} [\cmo(t, M, X)- \widebar \cmo(t, M, X)]\txmdensity(t \mid X, M) \int\left[\prod_{j=1}^{d_T} k\left(u_j\right)^{p} k\left(u_j+\frac{t_j-t_j^{\prime}}{h}\right)^{q}\right] \\
& + O\left(\frac{1}{h^{\left(p+q-1\right) d_T}}\right)\\
& = O\left(\frac{1}{h^{\left(p+q-1\right) d_T}}\right) = O\left(\frac{1}{h^{3d_T}}\right)
\end{align*}

where $\bar{t}$ and $\tilde{t}$ are between $t$ and $t+u h$.

For the calculation of $\mathbb{E}\left(\phi_i^{p} m_i^{q}\right)$, since we assume the boundedness of $\ccmo\left(t, t^{\prime}, X\right)-\mr_{t, t^{\prime}}$, and taking the inner expectation of $\mathbb{E}\left[(Y-\widebar{\cmo}(t, X, M))^p \mid T=t, X, M\right] \txmdensity(t \mid M, X) \bar{\lambda}(t^\prime, X)^p \bar r(M, X)^p \left(\ccmo\left(t, t^{\prime}, X\right)-\mr_{t, t^{\prime}}\right)^q$ in a similar fashion as in $\mathbb{E}\left(\phi_i^4\right)$, we obtain that $\mathbb{E}\left(\phi_i^{p} m_i^{q}\right)=O\left(\frac{1}{h^{\left(c_1-1\right) d_T}}\right)$. With a reasoning similar to those two terms and with $\mathbb{E}\left(\varphi_i^{p} m_i^{q}\right)=O\left(\frac{1}{h^{\left(p-1\right) d_T}}\right)$.

Now considering integers $p, q, m$ such that $p+q+m=4$ similarly $\mathbb{E}\left(\phi_i^{p} \varphi_i^{q} m_i^{m}\right)=O\left(\frac{1}{h^{\left(p+q-1\right) d_T}}\right)$. 

Combining all the terms in the expansion $\mathbb{E}\left[(\phi_i + \varphi_i + m_i)^4\right]$, we obtain $\mathbb{E}\left(\psi_i^4\right)=O\left(h^{-3 d_T}\right)$. Then by Markov inequality, for any $\epsilon>0$,
$$
\begin{aligned}
& P\left(\left|h^{d_T} n^{-1} \sum_{i \in I_{\ell}} \psi_i^2-\mathrm{V}_{t,t'}\right|>\epsilon\right) \leq \frac{1}{\epsilon^2} \mathbb{E}\left\{\left[h^{d_T} n^{-1} \sum_{i \in I_{\ell}} \psi_i^2-\mathrm{V}_{t,t'}\right]^2\right\} \\
= & \frac{1}{\epsilon^2} \mathbb{E}\left\{\left[h^{d_T} n^{-1} \sum_{i \in I_{\ell}} \psi_i^2-h^{d_T} \mathbb{E}\left[\psi_i^2\right]+o_p(1)\right]^2\right\} \\
= & \frac{h^{2 d_T}}{n^2 \epsilon^2} \mathbb{E}\left\{\left[\sum_{i \in I_{\ell}} \psi_i^2-\mathbb{E}\left(\sum_{i \in I_{\ell}} \psi_i^2\right)\right]^2\right\}+o_p(1) \\
= & \frac{h^{2 d_T}}{n^2 \epsilon^2} \operatorname{var}\left(\sum_{i \in I_{\ell}} \psi_i^2\right)+o_p(1) \\
= & \frac{h^{2 d_T}}{n \epsilon^2} \operatorname{var}\left(\psi_i^2\right)+o_p(1) \\
= & O\left(\frac{1}{n h^{d_T}}\right)=o_p(1)
\end{aligned}
$$

where the equality in the last row comes from $\operatorname{var}\left(\psi_i^2\right)=O\left(\mathbb{E}\left(\psi_i^4\right)\right)=O\left(h^{-3 d_T}\right)$.
\\

\paragraph{Showing $h^{d_T} n_{\ell}^{-1} \sum_{i \in I_{\ell}} \mathbb{E}\left[\hat{\psi}_{i \ell}^2-\psi_i^2 \mid Z^{\ell}\right]=o_p(1)$.}
Let us now introduce:

\begin{align}
& \widehat \phi_{i \ell} = K_h(T_{i \ell}-t)  \hat r_{i \ell} \hat \lambda_{i \ell} (Y_i-\widehat \cmo_{i \ell}) \\
& \widehat \varphi_{i \ell}=K_h(T_{i \ell}-t')\hat \lambda_{i \ell} (\widehat\cmo_{i \ell} -\widehat\ccmo_{i \ell}) \\
& \widehat m_{i \ell} = \widehat \ccmo_{i \ell} - \widehat \psi_i
\end{align}

Then again $\hat{\psi}_{i \ell}=\widehat \phi_{i \ell}+\widehat \varphi_{i \ell}+\widehat m_{i \ell}$ and in the expansion of $\hat{\psi}_{i \ell}^2-\psi_i^2$, we first compute for $\widehat \phi_{i \ell}^2$
$$
h^{d_T} \mathbb{E}\left(\widehat \phi_{i \ell}^2 \mid Z^{\ell}\right) = \mathbb{E}\left[\mathbb{E}\left[\left(Y_i-\widehat{\cmo}_{i \ell}\right)^2 \mid T=t, X, M, Z^{\ell}\right] \txmdensity(t \mid X_i, M_i) \hat{r}_{i \ell} \hat{\lambda}_{i \ell}^2 \mid Z^{\ell}\right] R_k^{d_T}+o_p(1)
$$

By pairing $\widehat \phi_{i \ell}^2$ to $\phi_i^2$ we show that

$$
\mathbb{E}\left[\left.\mathbb{E}\left[\left.\left(Y_i-\widehat{\cmo}_{i \ell}\right)^2 \hat{r}_{i \ell} \hat{\lambda}_{i \ell}^2-\left(Y_i-\widebar{\cmo}_i \right)^2 \bar{r}_i \bar{\lambda}_i^2 \right\rvert\, T=t, X, Z^{\ell}\right] \txmdensity(t \mid X_i, M_i) \right\rvert\, Z^{\ell}\right]=o_p(1)
$$

Using the decomposition of the outcome residual in Eq. \eqref{eq:decomposition_outcome_residual}, the previous writes
\begin{align*}
&\mathbb{E}\left[\left\{ \left( \hat r_{i \ell} - \bar r_{i } \right) \left( \hat \lambda_{i \ell} - \bar \lambda_{i} \right) \left( Y_i - \widebar \cmo_{i} \right) -  \left( \hat r_{i \ell} - \bar r_{i } \right) \left( \hat \lambda_{i \ell} - \bar \lambda_{i } \right) \left( \widehat \cmo_{i \ell} - \widebar \cmo_{i } \right) \right. \right. \\
&  - \left( \hat r_{i \ell} - \bar r_{i } \right) \left( \widehat \cmo_{i \ell} - \widebar \cmo_{i } \right) \bar \lambda_{i } - \left( \hat \lambda_{i \ell} - \bar \lambda_{i } \right) \left( \widehat \cmo_{i \ell} - \widebar \cmo_{i } \right) \bar r_{i } \\
& \left. \left. \left. + \left( \hat r_{i \ell} - \bar r_{i } \right) \left( Y_i - \widebar \cmo_{i } \right)  \bar \lambda_{i } + \left( \hat \lambda_{i \ell} - \bar \lambda_{i } \right) \left( Y_i - \widebar \cmo_{i } \right) \bar r_{i } - \left( \widehat \cmo_{i \ell} - \widebar \cmo_{i} \right)    \bar r_{i } \bar \lambda_{i } \right\}^2 \right\rvert\, Z^{\ell} \right].    
\end{align*}

which is $o_p(1)$ by Assumption \ref{assum:regularity} and Assumption \ref{assum:split_consistency}.1.

Next, for $\widehat \varphi_{i \ell}$, we similarly write:

$$
h^{d_T} \mathbb{E}\left(\widehat \varphi_{i \ell}^2 \mid Z^{\ell}\right)=\mathbb{E}\left[\mathbb{E}\left[\left(\hat{\cmo}_{i \ell}-\hat{\ccmo}_{i \ell}\right)^2 \mid T=t, M, X, Z^{\ell}\right] \txmdensity(t^\prime \mid X_i, M_i) \hat{\lambda}_{\ell}(t, X)^2 \mid Z^{\ell}\right] R_k^{d_T}+o_p(1)
$$

By pairing $\widehat \varphi_{i \ell}^2$ to $\varphi_i^2$ we show that


$$
\mathbb{E}\left[\left.\mathbb{E}\left[\left. \left(\hat{\cmo}_{i \ell}-\hat{\ccmo}_{i \ell}\right)^2\hat{\lambda}_{i \ell}^2-\left(\widebar{\cmo}_i-\widebar{\ccmo}_i\right)^2\bar{\lambda}_i^2 \right\rvert\, T=t, X, Z^{\ell}\right] \txmdensity(t^\prime \mid X_i, M_i) \right\rvert\, Z^{\ell}\right]=o_p(1)
$$

This time using the decomposition of the conditional mean outcome residual in Eq. \eqref{eq:decomposition_cmo_residual}, the previous writes as 

\begin{align*}
&\mathbb{E}\left[\left\{ \left( \hat \lambda_{i \ell} - \bar \lambda_{i} \right) \left( \widehat \cmo_{i \ell} - \widebar \cmo_{i} \right)  - \left( \hat \lambda_{i \ell} - \bar \lambda_{i \ell} \right) \left( \widehat \ccmo_{i \ell} - \widebar \ccmo_{i} \right) + \left( \hat \lambda_{i \ell} - \bar \lambda_{i} \right)  \widebar \cmo_{i} \right. \right. \\
& + \left. \left. \left. \left( \widehat \cmo_{i \ell} - \widebar \cmo_{i} \right)  \widebar \lambda_{i}  - \left( \widehat \ccmo_{i \ell} - \widebar \ccmo_{i} \right)  \bar \lambda_{i} - \left( \hat \lambda_{i \ell} - \bar \lambda_{i} \right)  \widebar \ccmo_{i}. \right\}^2 \right\rvert\, Z^{\ell} \right].    
\end{align*}

which is $o_p(1)$ by Assumption \ref{assum:regularity} and Assumption \ref{assum:split_consistency}.1.

Next, if we pair $\widehat m_{i \ell}^2$ to $m_i^2$ through
$$
h^{d_T} \mathbb{E}\left\{\left[\widehat{\ccmo}{i \ell}- \widehat \psi_i\right]^2 \mid Z_{I_{\ell}}^c\right\}=o_p(1)
$$

This holds because we assume the nuisance estimators are bounded, and following a similar calculation as Eq. \eqref{eq:variance_ccmo} it can be seen that $h^{d_T} \mathbb{E}\left[\widehat \psi_i \mid Z^{\ell}\right]=o_p(1)$. Combined with Jensen's inequality, this can be used to obtain the desired result. Eventually, $h^{d_T} \mathbb{E}\left(\widehat \phi_{i \ell} \widehat \varphi_{i \ell} -  \phi_i \varphi_i \mid Z^{\ell}\right) = o_p(1)$ following a computation similar to \eqref{eq:variance_outcome_ccmo} and by Assumption \ref{assum:regularity} and Assumption \ref{assum:split_consistency}.1. Similar reasoning applies to $\mathbb{E}\left(\widehat \phi_{i \ell} \widehat m_{i \ell} -  \phi_i  \hat \psi_i \mid Z^{\ell}\right)$, $\mathbb{E}\left(\widehat \varphi_{i \ell} \widehat m_{i \ell} -  \varphi_i  \hat \psi_i \mid Z^{\ell}\right)$. 
\\

\paragraph{Showing $h^{d_T} n_{\ell}^{-1} \sum_{i \in I_{\ell}} \Delta_{i \ell}=o_p(1)$.}

We write $\Delta_{i \ell} \equiv$ $\hat{\psi}_{i \ell}^2-\psi_i^2-\mathbb{E}\left[\hat{\psi}_{i \ell}^2-\psi_i^2 \mid Z^{\ell}\right]$ and will use the same decompositions using $\widehat \phi_{i \ell}, \widehat \varphi_{i \ell}, \widehat m_{i \ell}$.

Writing $\Delta_{i \ell, \phi} \equiv$ $\hat{\phi}_{i \ell}^2-\phi_i^2-\mathbb{E}\left[\hat{\phi}_{i \ell}^2-\phi_i^2 \mid Z^{\ell}\right]$, with the decomposition in Eq. \eqref{eq:decomposition_outcome_residual} and Assumption \ref{assum:split_consistency}.2-3., the similar arguments as for $\mathbb{E}\left[\phi_i^4\right]$ above yields 

\hspace*{-1cm}

\begin{equation}
\mathbb{E}\left[\Delta_{i \ell, \phi}^2 \mid Z^{\ell}\right]=O_p\left(\mathbb{E}\left[\left.K_h(T_i-t)^4
\left[\left(Y_i-\widehat{\cmo}_{i \ell}\right)^2 \hat{r}_{i \ell} \hat{\lambda}_{i \ell}^2-\left(Y_i-\widebar{\cmo}_i \right)^2 \bar{r}_i \bar{\lambda}_i^2 \right]^2
\right\rvert\, Z^{\ell}\right]\right)=O_p\left(h^{-3 d_T}\right).   
\end{equation}

We can apply a similar reasoning for $\Delta_{i \ell, \varphi} \equiv$ $\hat{\varphi}_{i \ell}^2-\varphi_i^2-\mathbb{E}\left[\hat{\varphi}_{i \ell}^2-\varphi_i^2 \mid Z^{\ell}\right]$, and other cross terms.

Then $\operatorname{var}\left(h^{d_T} n_{\ell}^{-1} \sum_{i \in I_{\ell}} \Delta_{i \ell} \mid Z^{\ell}\right)=O_p\left(h^{2 d_T} n_{\ell}^{-1} h^{-3 d_T}\right)=O_p\left(n_{\ell}^{-1} h^{-d_T}\right)=o_p(1)$. The result follows by the conditional Markov's inequality.

\paragraph{Consistency of $\hat{\mathrm{B}}_{t, t^\prime}$ :}

Theorem \ref{thm:asymptotic_normality} provides $\mathbb{E}\left[\hat{\mr}_{t, t^\prime \epsilon b_n}\right]=b_n^2 \epsilon^2 \mathrm{~B}_{t, t^\prime}+o\left(b_n^2\right)$ and therefore $\mathbb{E}\left[\hat{\mathrm{B}}_{t, t^\prime}\right]=\mathrm{B}_{t, t^\prime}+o(1)$. Compute

$$
\begin{aligned}
\mathbb{E} & {\left[\hat{\mr}_{t, t^\prime, b_n} \hat{\mr}_{t, t^\prime, \epsilon b_n}\right] } \\
= & O\left(\frac{1}{n}\mathbb { E } \left[\int_{\mathcal{T}} K_{b_n}(s-t) K_{\epsilon b_n}(s-t) \mathbb{E}\left[\left(Y-\hat{\cmo}_{\ell}(t, M, X)\right)^2 \mid T=s, X, M, Z^{\ell}\right] \right. \right. \\
& \times \left. \left. \txmdensity(s \mid M, X) d s \hat{r}(M,X)^2 \hat{\lambda}_{\ell}(t^\prime, X)^2 \mid Z^{\ell}\right]\right) \\
 & + O\left(\frac{1}{n}\mathbb { E } \left[\int_{\mathcal{T}} K_{b_n}(s-t^\prime) K_{\epsilon b_n}(s-t^\prime) \mathbb{E}\left[\left(\hat{\cmo}_{\ell}(t, M, X)-\hat{\ccmo}_{\ell}(t, t^\prime, X)\right)^2 \mid T=s, X, Z^{\ell}\right] \right. \right. \\
& \times \left. \left. \txmdensity(s \mid M, X) d s \hat{\lambda}_{\ell}(t^\prime, X)^2 \mid Z^{\ell}\right]\right) \\
= & O\left(n ^ { - 1 } b_n ^ { - d _ { T } } \epsilon ^ { - d _ { T } } \int _ { \mathcal { R } ^ { d _ { T } } } \Pi _ { j = 1 } ^ { d _ { T } } k ( u _ { j } ) k ( \frac { u _ { j } } { \epsilon } ) d u \mathbb { E } \left[\mathbb{E}\left[\left(Y-\hat{\gamma}_{\ell}(t, X)\right)^2 \mid T=t, X, Z^{\ell}\right] \right. \right. \\
& \times \left. \left. \left. \txmdensity(t \mid X, M)\hat{r}_{\ell}(M, X)^2\hat{\lambda}_{\ell}(t^\prime, X)^2 \right\rvert\, Z^{\ell}\right]\right) \\
+ & O\left(n ^ { - 1 } b_n ^ { - d _ { T } } \epsilon ^ { - d _ { T } } \int _ { \mathcal { R } ^ { d _ { T } } } \Pi _ { j = 1 } ^ { d _ { T } } k ( u _ { j } ) k ( \frac { u _ { j } } { \epsilon } ) d u \mathbb { E } \left[\mathbb{E}\left[\left(\hat{\cmo}_{\ell}(t, M, X)-\hat{\ccmo}_{\ell}(t, t^\prime, X)\right)^2 \mid T=t^\prime, X, M, Z^{\ell}\right] \right. \right. \\
& \times \left. \left. \left. \txmdensity(t^\prime \mid X, M)\hat{\lambda}_{\ell}(t^\prime, X)^2 \right\rvert\, Z^{\ell}\right]\right) \\
= & O\left(\left(n b_n^{d_T}\right)^{-1}\right)
\end{aligned}
$$

with the same arguments as in Lemma \ref{lemma:kernel_smoothing}. So $\operatorname{cov}\left(\hat{\mr}_{t, t^\prime, b_n}, \hat{\mr}_{t, t^\prime, \epsilon b_n}\right)=O\left(\operatorname{var}\left(\hat{\mr}_{t, t^\prime, \epsilon b_n}\right)\right)=O\left(\left(n b_n^{d_T}\right)^{-1}\right)$. It follows that 

\begin{equation*}
\operatorname{var}\left(\hat{\mathrm{B}}_{t, t^\prime}\right)=b_n^{-4}\left(1-\epsilon^2\right)^{-2}\left(\operatorname{var}\left(\hat{\mr}_{t, t^\prime, b_n}\right)+\operatorname{var}\left(\hat{\mr}_{t, \epsilon b_n}\right)-2 \operatorname{cov}\left(\hat{\mr}_{t, b_n}, \hat{\mr}_{t, t^\prime \epsilon b_n}\right)\right)=O\left(b_n^{-4}\left(n b_n^{d_T}\right)^{-1}\right).    
\end{equation*} 

By the Markov's inequality and Assumption \ref{assum:split_consistency}.4., 

\begin{equation*}
P\left(\left|\hat{\mathrm{B}}_{t, t^\prime}-\mathbb{E}\left[\hat{\mathrm{B}}_{t, t^\prime}\right]\right|>\epsilon\right) \leq \operatorname{var}\left(\hat{\mathrm{B}}_{t, t^\prime}\right) / \epsilon^2= O\left(\left(n b_n^{d_T+4}\right)^{-1}\right)=o(1). 
\end{equation*}
 So $\hat{\mathrm{B}}_{t, t^\prime}-\mathrm{B}_{t, t^\prime}=o_p(1).$   

\end{proof}

\section{NON PARAMETRIC LEARNING OF NUISANCE PARAMETERS}

\label{appendix:nuisances}




We present the non parametric method to estimate the conditional and cross conditional mean outcomes following \citep{singh2023sequential}. Using kernel mean embeddings, the mediated response curve $\mr_{t,t^\prime}$ defined in Eq. \eqref{eq:mediated_response} is identified as a sequential integral of the form $\int \cmo(t',m,x) dQ$ for the distribution $Q=\mdensity(m|t',x) f(x)$ in the Pearl formula in Eq. \eqref{eq:pearl_formula}. 
\\

To begin, we construct an RKHS $\cH$ for $\cmo$. We define a RKHS for the treatment $\cH_T$, mediator $\cH_M$ and covariates $\cH_X$, then assume that the conditional mean $\cmo$ is an element of the RKHS $\cH$ with the kernel $k \left((t,m,x),(t',m',x')\right) = k_T(t,t') k_M(m,m') k_X(x,x')$. Formally, this choice of kernel corresponds to the tensor product : $\cH = \cH_T \otimes  \cH_M \otimes \cH_X  $. As such, 
\begin{equation}
    \cmo(t,m,x) = \langle \mu_Y, \phi_T(t) \otimes \phi_M(m) \otimes \phi_X(x) \rangle,
\label{eq:kernel_mean_embedding}
\end{equation}

where $\phi_T, \phi_M$ and $\phi_X$ are the feature maps respectively associated with $k_T, k_M$ and $k_X$. 

Then, define:
\begin{align}
    \mu_{m}(t,x) &= \int \phi_M(m) d \mdensity(m|t,x) \\
    \mu_{m,x}(t) &= \int \lbrace \mu_m(t,x) \otimes \phi_X(x) \rbrace d f(x)
\end{align}

Consider also the cross conditional mean outcome $\ccmo(t',t,x)=\int \cmo(t',m,x) d\mdensity(m|t,x)$ we previously defined in Eq. \ref{eq:cross_conditional_mean_outcome}. We can then formulate it with the kernel mean embedding.

Then assuming regularity conditions on the RKHS, we have:

    \begin{align}
    \ccmo(t',t,x) &= \langle \cmo, \phi_T(t') \otimes \mu_{m}(t,x) \otimes \phi_X(x) \rangle_\cH \\
    \mr_{t,t^\prime} &= \langle \cmo, \phi_T(t') \otimes \mu_{m,x}(t) \rangle_\cH
\end{align}

Moreover, denote the kernel matrices by $K_{TT}, K_{MM}, K_{XX} \in \R^{n,n}$. Let $\odot$ be the element-wise product. Mediated response curves have closed form solutions:

\begin{multline}
    \hat \ccmo(t',t,x) = Y^\top \left(K_{TT} \odot K_{MM} \odot K_{XX} + n \lambda I \right)^{-1} \cdot \\
     \left[ K_{Tt'} \odot \lbrace K_{MM} (K_{TT} \odot K_{XX} + n \lambda_1 I)^{-1} \right] \\\left( (K_{Tt} \odot K_{Xx}) \rbrace \odot K_{Xx} \right)
    \label{eq:omega}
\end{multline}

\begin{equation}
   \hat \mr_{t,t^\prime} = \frac{1}{n} \sum_{i=1}^n \hat \ccmo(t',t,x_i)
\end{equation}

where $(\lambda, \lambda_1)$ are ridge regression penalty parameters. Moreover, the conditional mean outcome written in Eq. \ref{eq:kernel_mean_embedding} has the following closed form:
\begin{multline}
\hat \cmo(t,m,x) = Y^\top \left(K_{TT} \odot K_{MM} \odot K_{XX} + n \lambda I \right)^{-1} \\
\left( K_{Tt} \odot K_{Mm} \odot K_{Xx} \right)
\end{multline}

\section{EXPERIMENT DETAILS}

\label{appendix:experiment_details}

All the code to reproduce our experiment can be found at \url{https://github.com/houssamzenati/double-debiased-machine-learning-mediation-continuous-treatments}.

\subsection{Tuning of hyperparameters}

In this part we discuss the selection strategy for hyperparameters.
\label{appendix:hyperparameters}

\paragraph{Kernel smoothing bandwidth }For the \citet{scott2015multivariate} heuristic of bandwidth selection, the bandwidth $h$ is determined by multiplying the respective standard deviations of $T$ with $C n^{-1/5}$, where $C=1.06$. 

\paragraph{Kernel mean embedding ridge penalty} We experiment grid search and  generalized cross validation (GCV) as in \citep{singh2023sequential}, where for a variable $W \in \mathcal{W}$ we construct the matrices
$$
H_\lambda=I-K_{W W}\left(K_{W W}+n \lambda I\right)^{-1} \in \mathbb{R}^{n \times n}
$$
and set
$$
\lambda^*=\underset{\lambda \in \Lambda}{\operatorname{argmin}} \frac{1}{n}\left\|\left\{\operatorname{tr}\left(H_\lambda\right)\right\}^{-1} \cdot H_\lambda Y\right\|_2^2, \quad \Lambda \subset \mathbb{R}
$$

\paragraph{Kernel mean embedding bandwidth} The exponentiated quadratic kernel is widely used among machine learning practitioners:
$$
k\left(w, w^{\prime}\right)=\exp \left\{-\frac{1}{2} \frac{\left(w-w^{\prime}\right)_{\mathcal{W}}^2}{\iota^2}\right\}
$$

Importantly, it satisfies the required properties; it is continuous, bounded, and characteristic. Its hyperparameter is called the bandwidth $\iota$. A convenient heuristic as in \citep{singh2023sequential} is to set the lengthscale equal to the median interpoint distance of $\left(W_i\right)$, where the interpoint distance between the observations $i$ and $j$ is $\left\|W_i-W_j\right\|_{\mathcal{W}}$. When the input $W$ is multidimensional, we use the kernel obtained as the product of scalar kernels for each input dimension.

\subsection{Synthetic simulations}

In this part we provide all additional details and results for the experiment of \citet{hsu2020} and \citep{sani2024}.

\label{appendix:experiment_hsu}

\subsubsection{Generating process and causal experiment of \citet{hsu2020}}

The simulation is based on the following data generating process. 
$$
\begin{aligned}
X \sim \operatorname{Uniform}(-1.5,1.5), \\
U, V, W \sim \operatorname{Uniform}(-2,2)
\end{aligned}
$$
and then 
$$
\begin{aligned}
T&=0.3 X+ U \\
M & =0.3 T+0.3 X+ V \\
Y & =0.3 T+0.3 M+\alpha T M+0.3 X+\beta T^3+ W \\
\end{aligned}
$$
independently of each other.
Outcome $Y$ is a function of the observed variables $T, M, X$ and an unobserved term $W$. $\alpha$ gauges the interaction effect between $T$ and $M$. $\alpha=0$ satisfies the assumption of no interaction as discussed in \citep{robins2003}. In contrast, for $\alpha \neq 0$, direct and indirect effects are heterogeneous. The coefficient $\beta$ determines whether the direct effect of $T$ on $Y$ is linear $(\beta=0)$ or nonlinear, namely cubic $(\beta \neq 0)$.  
\\

In this simulation, the direct effect is given by $\dce(t, t^{\prime})=0.3\left(t-t^{\prime}\right)+0.3 \alpha t \left(t^\prime-t\right)+\beta\left(t^{\prime 3}-t^{3}\right)$ and the indirect effect by $\ice(t, t^{\prime})=$ $0.09\left(t^\prime-t\right)+0.3\alpha t^\prime\left(t^\prime-t\right)$. We will consider the setting of \citet{hsu2020} where $\alpha=0.25$ and $\beta=0.5$ as \citet{singh2023sequential} did. Since the outcome function is smooth and has regular properties, we use a single fold for training the DML estimator.

For the causal experiment, the definition of the direct and indirect effects \citet{hsu2020} sets $t^{\prime}=0$, for $t$, they consider a sequence of values defined by an equidistant grid between (and including) -1.5 and 1.5 with step size 0.1 (i.e. $t \in\{-1.5,-1.4, \ldots 1.4,1.5\}$. 

\subsubsection{Additional results on the syntethic setting of \citet{hsu2020}}
\label{appx:additional_results}

\paragraph{Comparison of DML to other methods}

Table \ref{tab:merged_samples} reports the averages of the absolute bias (bias), standard deviation (std), and root mean squared error (RMSE) for each effect, where we consider the average and standards deviations of averaging the error over all treatment comparisons $\left(t-t^{\prime}\right)$ in the grid defined above. The DML approach provides significantly better estimation for the direct and total effect and also provides satisfactory performance for the indirect effect. Not surprisingly, the performances increase with larger sample sizes.

\begin{table}[h]
    \centering
\caption{Average absolute bias (bias), standard deviation (std), and root mean squared error (RMSE) for each effect across different sample sizes.}
\begin{tabular}{llccccccccc}
\toprule
$n$ & Method & \multicolumn{3}{c}{direct} & \multicolumn{3}{c}{indirect} & \multicolumn{3}{c}{total} \\
 &  & bias & std & rmse & bias & std & rmse & bias & std & rmse \\
\midrule
500  & OLS  & 0.3065 & 0.0411 & 0.3499 & 0.0133 & 0.0099 & 0.0151 & 0.3067 & 0.0413 & 0.3501 \\
     & IPW  & 0.1785 & 0.0687 & 0.2167 & 0.0698 & 0.0000 & 0.0794 & 0.1610 & 0.0732 & 0.1922 \\
     & KME  & 0.3748 & 0.0372 & 0.4898 & 0.4174 & 0.0588 & 0.5443 & 0.0842 & 0.0336 & 0.1090 \\
     & DML  & 0.0992 & 0.0361 & 0.1240 & 0.0339 & 0.0163 & 0.0427 & 0.0989 & 0.0391 & 0.1235 \\
\midrule
1000 & OLS  & 0.3013 & 0.0332 & 0.3445 & 0.0079 & 0.0075 & 0.0090 & 0.3014 & 0.0340 & 0.3446 \\
     & IPW  & 0.1437 & 0.0443 & 0.1773 & 0.0698 & 0.0000 & 0.0794 & 0.1255 & 0.0450 & 0.1510 \\
     & KME  & 0.3726 & 0.0234 & 0.4860 & 0.4156 & 0.0440 & 0.5399 & 0.0681 & 0.0268 & 0.0903 \\
     & DML  & 0.0809 & 0.0320 & 0.0996 & 0.0233 & 0.0098 & 0.0299 & 0.0790 & 0.0323 & 0.0970 \\
\midrule
5000 & OLS  & 0.3022 & 0.0139 & 0.3461 & 0.0040 & 0.0029 & 0.0045 & 0.3027 & 0.0135 & 0.3466 \\
     & IPW  & 0.0899 & 0.0183 & 0.1116 & 0.0698 & 0.0000 & 0.0794 & 0.0607 & 0.0207 & 0.0737 \\
     & KME  & 0.3759 & 0.0107 & 0.4889 & 0.4134 & 0.0191 & 0.5364 & 0.0444 & 0.0123 & 0.0606 \\
     & DML  & 0.0463 & 0.0135 & 0.0576 & 0.0154 & 0.0054 & 0.0196 & 0.0417 & 0.0148 & 0.0523 \\
\bottomrule
\end{tabular}
\label{tab:merged_samples}
\end{table}

\paragraph{Strategy for bandwidth selection}

In this numerical experiment we compare two bandwidth selection procedures: we investigate the practical performance of the theoretical optimal bandwidth which minimizes the asymptotic mean squared error (AMSE) of Corollary \ref{thm:amse_bandwidth} and the \citet{scott2015multivariate} rule of thumb which satisfies the assumptions of our theorems and corollaries. Figure \ref{fig:bandwidth_strategy} illustrates the average bias of both approaches under the 100 simulation, and shows how the AMSE strategy provides less variance than the heuristic but slightly more bias.  

\begin{figure}[h]
    \centering
    \includegraphics[width=0.5\linewidth]{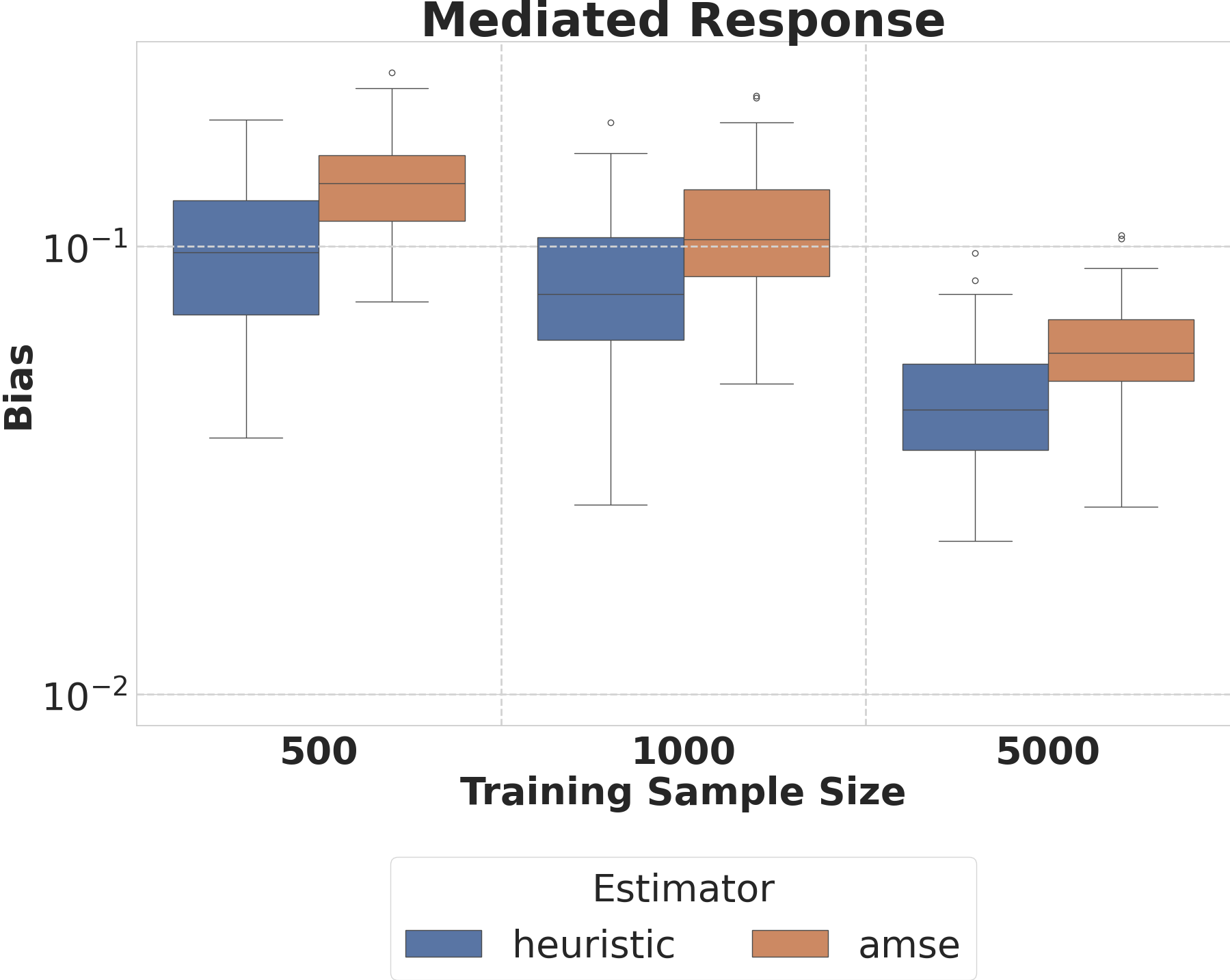}
    \caption{Bias of mediated response estimation with the DML estimator on simulations with different sample sizes and with two bandwidth selection strategies.}
    \label{fig:bandwidth_strategy}
\end{figure}

\paragraph{Validation of the asymptotic confidence interval}

We also perform a numerical experiment for the asymptotic confidence interval to test its empirical coverage. We run 100 simulations for the DML estimator where we report in Table \ref{table:coverage} the number of times the true mediated response is inside the asymptotic confidence interval. We take a $95\%$ level of confidence and therefore, the empirical coverage should ideally be close to that value. We can observe that this coverage has a strong value for a low sample size $n=500$ but gets closer to this target and increases along with larger sample sizes which are closer to asymptotic regimes. 

\begin{table}[h]
    \centering
\caption{Percentage of successful coverage of the asymptotic 95\% confidence interval on the synthetic setting of \citet{hsu2020} with regards to sample size. \label{table:coverage}}
\begin{tabular}{lccc}
\toprule
Sample size & 500 & 1000 & 5000  \\
\midrule
Coverage & 0.915 & 0.894 & 0.898 \\
\bottomrule
\end{tabular}
\end{table}

\paragraph{Benefits of cross-fitting} Regarding the finite sample performance of cross-fitting, we did not observe significant improvements in our numerical simulation. As expected, cross-fitting is less effective for simple problems involving Donsker class functions \citep{williamson2023general}. Practical studies, such as \citep{fan2021} and \citep{williamson2023general}, have also noted the limited benefits of cross-fitting in such cases, while its advantages become more evident in more complex nonparametric problems \citep{kennedy2022semiparametric}. We consider simple class of estimators in our simulation following \citep{hsu2020} and \citep{singh2023sequential} with a simple outcome function, which may explain the limited impact of cross-fitting. 

\paragraph{Parametric versus nonparametric nuisance estimation} We design an experiment similar in spirit to the misspecification experiments of \citep{kennedy2017} and \citep{doss2023nonparametricdoublyrobusttest}, where we aim at simulating the evaluating variants of the treatment densities and the conditional mean outcomes ($\mu$ and $\ccmo$). We either use simple parametric models (Gaussian parametrization of the treatments, linear model for the outcome), or leverage non-parametric methods (conditional kernel estimation and Gaussian kernel mean embedding).
We design four scenarios in which we compare our DML (with KME) against IPW and the G-computation (with KME)  in our numerical simulation \citep{hsu2020, singh2023sequential},  (i) treatment densities and conditional mean outcomes nonparametric ii) parametric well specified treatment densities and nonparametric conditional mean outcomes iii) non parametric treatment densities and parametric conditional mean outcome iv) parametric treatment densities and conditional mean outcomes.

\begin{table}[h]
    \centering
\caption{Performance summary for all nuisance parameter estimation scenarios}
\begin{tabular}{llcccccccc}
\toprule
Estimator & $n_\text{samples}$ & \multicolumn{2}{c}{Scenario i)} & \multicolumn{2}{c}{Scenario ii)} & \multicolumn{2}{c}{Scenario iii)} & \multicolumn{2}{c}{Scenario iv)} \\
 &  & Bias & RMSE & Bias & RMSE & Bias & RMSE & Bias & RMSE\\ 
\midrule
IPW & 500 & 0.1222 & 0.1472 & 0.1366 & 0.1655 & 0.1173 & 0.1424 & 0.1344 & 0.1631 \\
DML & 500 & 0.0824 & 0.0968 & 0.0822 & 0.0966 & 0.0800 & 0.0944 & 0.0786 & 0.0928 \\
KME & 500 & 0.1783 & 0.2095 & 0.1802 & 0.2123 & 0.1671 & 0.1933 & 0.1665 & 0.1930 \\ \hline
IPW & 1000 & 0.0945 & 0.1144 & 0.1194 & 0.1426 & 0.0946 & 0.1146 & 0.1180 & 0.1422 \\
DML & 1000 & 0.0624 & 0.0742 & 0.0620 & 0.0738 & 0.0631 & 0.0747 & 0.0633 & 0.0749 \\
KME & 1000 & 0.1731 & 0.2015 & 0.1739 & 0.2021 & 0.1637 & 0.1885 & 0.1629 & 0.1885 \\ \hline
IPW & 5000 & 0.0515 & 0.0638 & 0.0818 & 0.0982 & 0.0506 & 0.0627 & 0.0820 & 0.0983 \\
DML & 5000 & 0.0325 & 0.0390 & 0.0330 & 0.0393 & 0.0323 & 0.0391 & 0.0310 & 0.0369 \\
KME & 5000 & 0.1694 & 0.1932 & 0.1697 & 0.1936 & 0.1631 & 0.1862 & 0.1634 & 0.1865 \\
\bottomrule
\end{tabular}
\label{tab:nonparametric_table}
\end{table}

In Table \ref{tab:nonparametric_table}, we observe that for all the scenarios with all sample sizes DML consistently outperforms the two baselines, demonstrating the best performance even under all adverse conditions.

\subsection{Comparison to \citep{sani2024} in their simulation setting}
\label{appx:sani_setting}

Here we compare our DML formulation to the multiply robust estimator of \citet{sani2024} in the numerical experiment they did. We implemented their estimator and reproduced their numerical simulation, which is done with binary mediators. In Table \ref{table:sani}, we report the same metric they use, the average absolute bias to compare our DML to their variant, and also IPS, OLS, KME. Once again we see that our formulation outperforms theirs even in their setting; this illustrates the benefits of using an implicit integration of the conditional mean outcome $\cmo$ to estimate the cross conditional mean outcome $\ccmo$ and to use treatment propensities instead of the mediator density.

\begin{table}[h]
    \centering
    \caption{Average Absolute Bias for different estimators in the setting of \citep{sani2024}}
    \begin{tabular}{l c c}
        \hline
        Estimator & $n_{\text{samples}}$ & Average Absolute Bias \\
        \hline
        OLS & 2000 & 13.65 \\
        IPW & 2000 & 0.5298 \\
        KME & 2000 & 1.861 \\
        DML \citep{sani2024} & 2000 & 0.4099 \\
        DML (ours) & 2000 & 0.2507 \\
        \hline
        OLS & 5000 & 13.61 \\
        IPW & 5000 & 0.2883 \\
        KME & 5000 & 1.832 \\
        DML \citep{sani2024} & 5000 & 0.2997 \\
        DML (ours) & 5000 & 0.1445 \\
        \hline
        OLS & 8000 & 13.62 \\
        IPW & 8000 & 0.2481 \\
        KME & 8000 & 1.828 \\
        DML \citep{sani2024} & 8000 & 0.3111 \\
        DML (ours) & 8000 & 0.1280 \\
        \hline
    \end{tabular}
\label{table:sani}
\end{table}

\subsection{Application to cognitive function}

In this part we provide details on our application of mediation analysis on the UKBB project \citep{ukbb}.
\label{appendix:experiment_ukbb} We analyse the effect of a continuous measure, the glycated hemoglobin (HbA1c) on cognitive functions, and its possible mediation by the brain stucture.

\subsubsection{Causal estimand identification and data preparation}

The treatments in our study are well-defined, with a consistent definition among subjects and no interaction between subjects, as they are independent participants in the UKBB study, so the Stable Unit Treatment Value Assumption (SUTVA) is satisfied. Assumptions \ref{assum:indep_treatment} and \ref{assum:indep_mediator} focus on confounding variables between treatment, mediators, and outcomes, ensuring identifiability of the total, direct, and indirect causal effects. We accounted for a broad range of confounders, including behavioral, physiological, and societal factors, as well as technical artifacts such as the evaluation center, head positioning in imaging, and known brain diseases \citep{alfaro2021confound, newby2022understanding, schurz2021variability, topiwala2022alcohol}. The positivity assumption (Assumption \ref{assum:overlap}) was assessed by defining a grid of evaluation over values on which there was coverage in Figure \ref{fig:histogram_ukbb}.

\begin{figure}[h]
    \centering
    \includegraphics[width=0.7\linewidth]{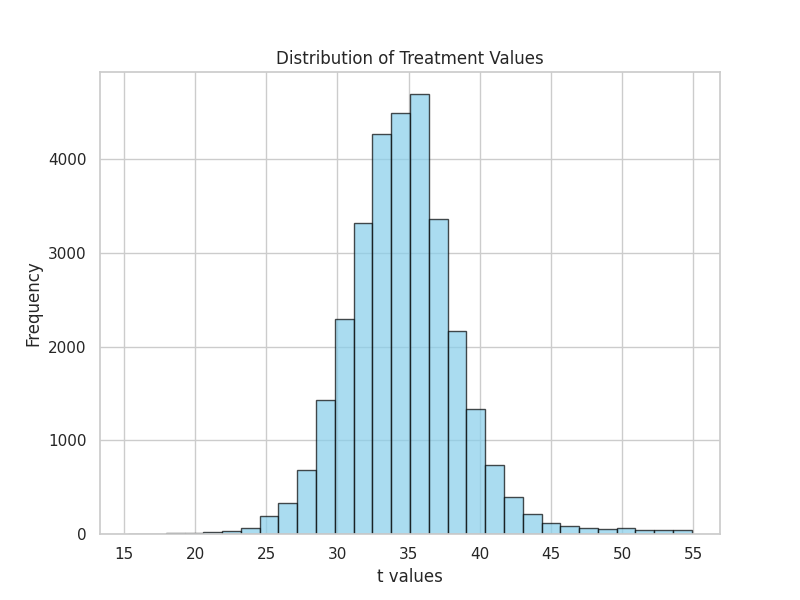}
    \caption{Empirical histogram of glycated hemoglobin levels (treatments) in the UKBB data.}
    \label{fig:histogram_ukbb}
\end{figure}

A key aspect of establishing identifiability is preparing the variables for estimation models. For the outcome, we aggregated cognitive test results (e.g., working memory, fluid intelligence, vocabulary) into a single "G-factor" measuring general cognitive ability, obtained via Principal Component Analysis (PCA) as in~\citep{fawns2020reliability}. This score was based on four cognitive tests administered at the imaging visit, with the G-factor correlating well with individual test results. For imaging data, we reduced the dimensionality of preprocessed imaging-derived phenotypes (IDPs) as in \citep{miller2016multimodal, alfaro2018image} using PCA after standard scaling, retaining six components from categories like dMRI skeleton, subcortical volumes, and regional gray matter volumes.

We also consolidated lifestyle variables to reduce dimensionality and address missing data due to the design of questionnaires. After transforming categorical variables via one-hot encoding and retaining those with more than 10,000 non-zero values (over the full dataset with 500,000 individuals), we further imputed missing data and reduced dimensionality using the R packages missMDA
485 and FactoMineR~\cite{josse2016missmda,le2008factominer}. The analysis was restricted to participants with brain MRI data, and we used two components for imputation.

After selecting and preparing the variables for causal mediation analysis, we limited the dataset to participants with brain imaging data and complete data for all relevant variables, resulting in a final sample of $n=30,595$ independent participants.

\subsubsection{Experiment details}

To estimate the total, direct and indirect causal effects for each treatment, mediated by brain structure, measured by brain MRI, we applied the estimators able to handle multidimensional mediators: the OLS estimator, the KME G-computation, the IPW and our double machine learning estimator as explained in Section \ref{sec:numerical}. We used cross-fitting with 2 folds. To assess uncertainty in the estimation, we applied the estimators to 100 bootstrapped samples for each exposure and compare the bootstrap procedure with the asymptotic confidence interval we derived for the mediated response. 

\subsubsection{Experiment results}

Previous studies \citep{newby2022understanding} show that suffering from diabetes has deleterious consequences and results in a lower G-factor score. Diabetes tends to have a negative total causal effect on cognitive functions, so we have explored the role of glycemic control, through the measure of glycated hemoglobin (HbA1c) at the initial visit. HbA1c is a proxy of the three-month average blood sugar level, so high levels of HbA1c can be found in individuals with diabetes with difficulties balancing their diet, physical activity, and medication to control their glucose level or in individuals with an undiagnosed issue. Exploring the role of HbA1c rather than the diabetic status allows us to directly assess the role of blood glucose in the health damages associated with this condition.

As illustrated in Figure~\ref{fig:effect_estimations_ukbb}~(left panel), the total effect is very low, indicating a lack of evidence of an effect of HbA1c at the initial visit on the cognitive functions a few years later (imaging visit). This lack of effect has been confimed with causal inference methods without mediation (Double Machine learning, as implemented in the Python package econML~\cite{econml}). 

Regarding the indirect effect estimates for diabetes and alcohol \citep{newby2022understanding, topiwala2022alcohol}, they are typically of weaker amplitude than the total effect. We present the total and indirect effect results for the estimators for exposure to glycated hemoglobin in Figure \ref{fig:effect_estimations_ukbb}. With DML, we found a slightly negative but not statistically significant effect of exposure to glycated hemoglobin on cognitive function.

\begin{figure}[h]
            \centering
        \includegraphics[width=0.49\linewidth]{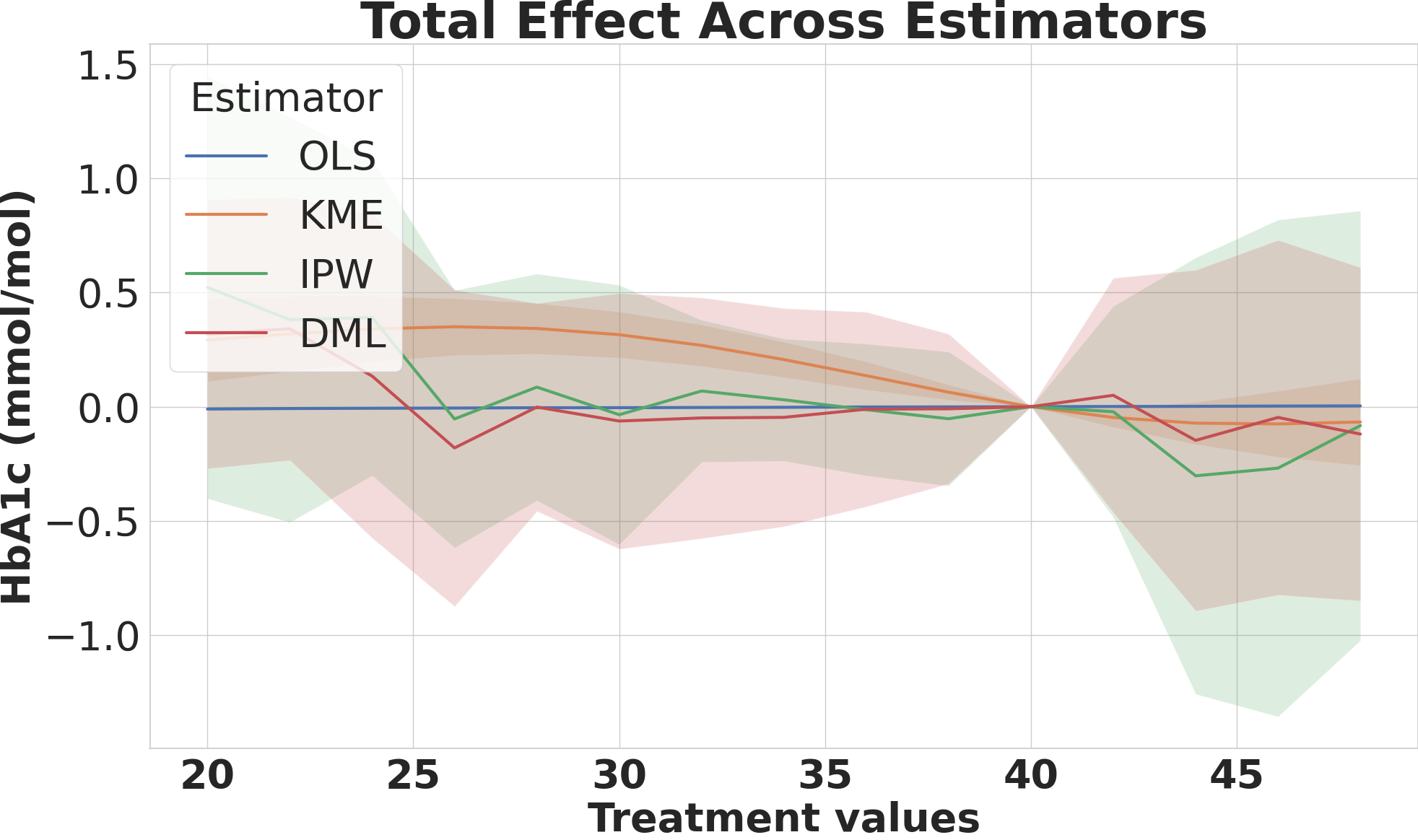}
    \hfill
\includegraphics[width=0.49\linewidth]{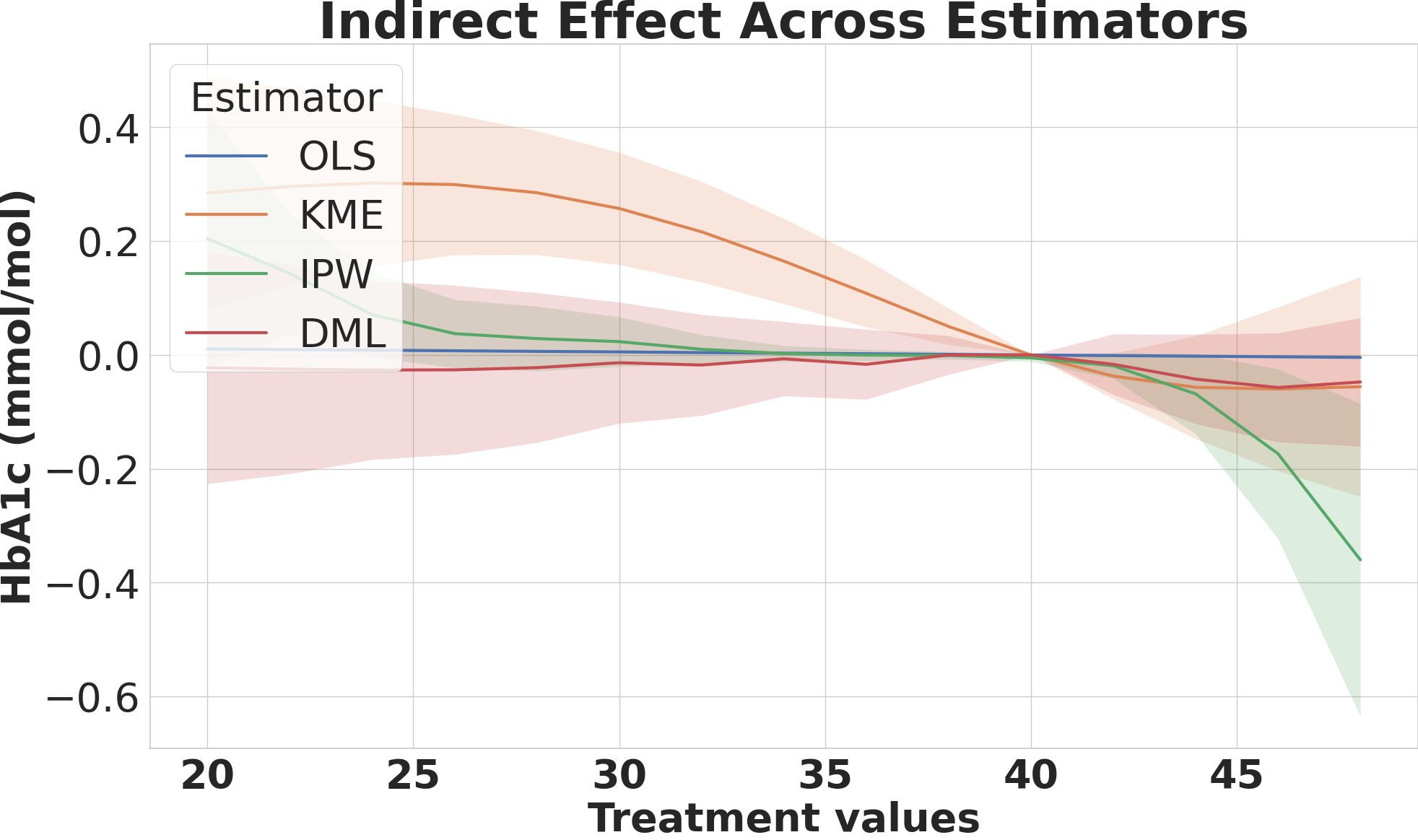}
    \caption{Effect estimation on the UKBB dataset for the total effect (left) and the indirect effect (right)}
    \label{fig:effect_estimations_ukbb}
\end{figure}

Overall, the results depended weakly on the choice of the estimation method. However, we notice an important difference in the variance between the different estimators, with a very small variance for the OLS and KME, and larger ones for the DML and IPW estimators. Our results also illustrate how IPW suffers more from overlap issues at the boundaries of the treatment space. Eventually, as a validation, we provide in Figure \ref{fig:bootstrap_ukbb} a comparison of our asymptotic confidence interval and the bootstrap procedure, showing a more regular estimation (and less computationally expensive).

\begin{figure}[h]
    \centering
    \includegraphics[width=0.8\linewidth]{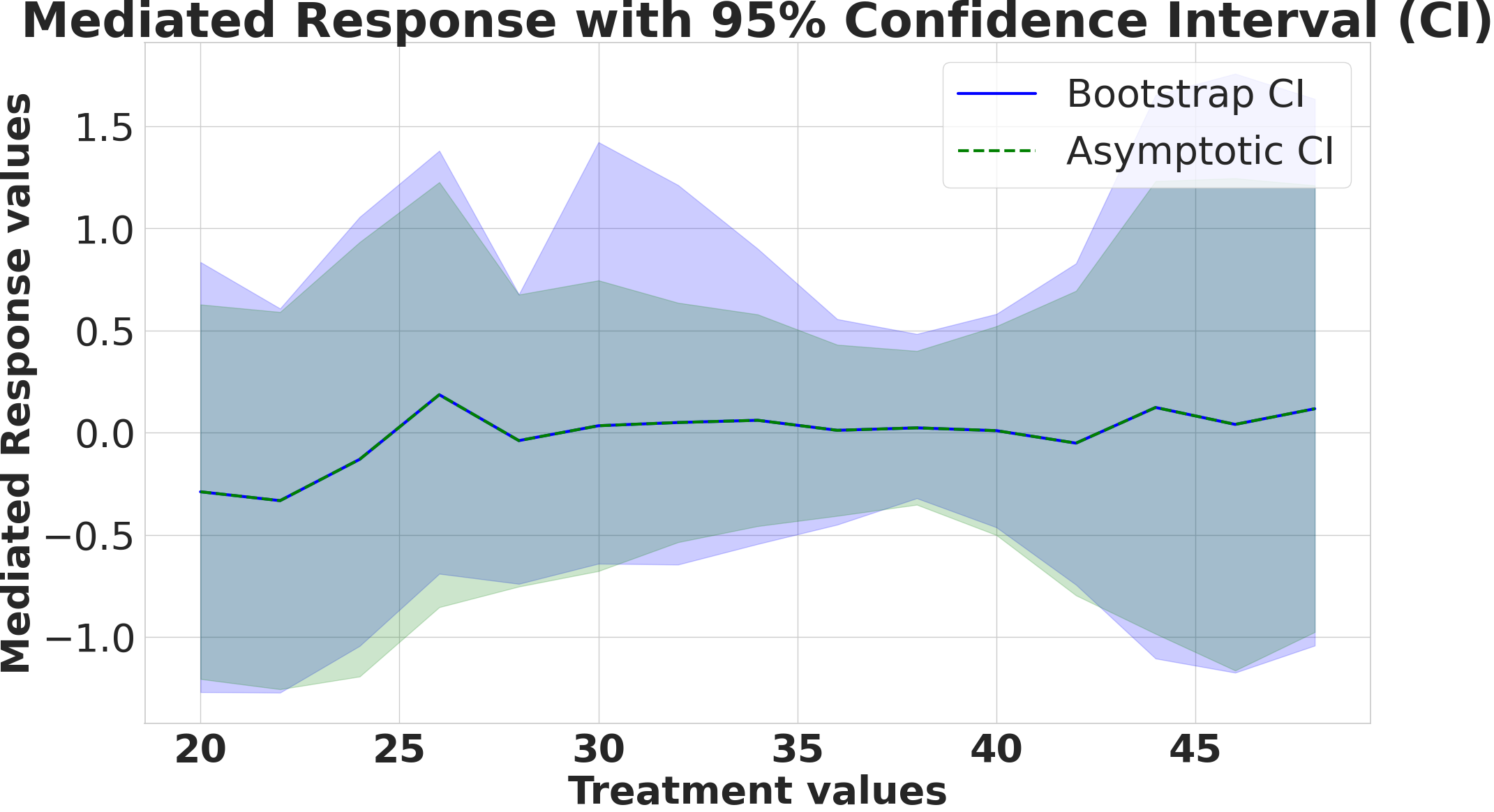}
    \caption{Comparison of the asymptotic confidence interval and the bootstrap procedure for uncertainty of the mediated response on the UKBB dataset.}
    \label{fig:bootstrap_ukbb}
\end{figure}

\subsection{Computation infrastructure}

\label{appendix:computation infrastructure}
We ran our experiments on a CPU clusters with the following characteristics.

\begin{itemize}[label=\textbullet]
    \item \textbf{Dell C6320:}
    \begin{itemize}
        \item 2 Intel(R) Xeon(R) CPU E5-2660 v2 @ 2.20GHz
        \item 256G of RAM
        \item Ethernet 1G
        \item Infiniband Mellanox QDR
    \end{itemize}
    
    \item \textbf{Dell R7525:}
    \begin{itemize}
        \item 2 AMD EPYC 7742 64-Core Processor
        \item 512G of RAM
        \item Ethernet 10G
    \end{itemize}
    
    \item \textbf{Dell R7525:}
    \begin{itemize}
        \item 2 AMD EPYC 7702 64-Core Processor
        \item 512G of RAM
        \item Ethernet 10G
    \end{itemize}
    
    \item \textbf{Dell R7525:}
    \begin{itemize}
        \item 2 AMD EPYC 7742 64-Core Processor
        \item 1024G of RAM
        \item Ethernet 10G
    \end{itemize}
    
    \item \textbf{Dell R640:}
    \begin{itemize}
        \item 2 Intel(R) Xeon(R) Gold 6226R CPU @ 2.90GHz
        \item 256G of RAM
        \item Ethernet 1G (pending upgrade to 10G)
    \end{itemize}
    
    \item \textbf{HPE Proliant DL385:}
    \begin{itemize}
        \item 2 AMD EPYC 7763 64-Core Processor
        \item 1024G of RAM
        \item Ethernet 10G
    \end{itemize}
\end{itemize}

\end{document}